%% file: SIIMS_Revised_Manuscript.tex
\documentclass[11pt,reqno]{amsart}
\usepackage{graphicx,enumerate,amssymb,bbold}
\usepackage[notrig]{physics}
\usepackage[font=footnotesize,labelfont=small]{subcaption}
\usepackage[ruled,vlined]{algorithm2e}
\usepackage[left=2.5cm,right=2.5cm,top=2.5cm,bottom=2.5cm,includeheadfoot]{geometry} 

\usepackage{tikz}
\tikzstyle{mybox} = [draw=blue!50, fill=blue!20, very thick,
    rectangle, rounded corners, inner sep=10pt, inner ysep=15pt]
\tikzstyle{fancytitle} =[fill=gray!20, rounded corners, text=black]
\usepackage{tabularx}
\usepackage{tabulary}
\usepackage{multirow}

\usepackage[colorlinks=true, pdfstartview=FitV, linkcolor=blue, 
            citecolor=blue, urlcolor=blue]{hyperref}

\numberwithin{equation}{section}
\numberwithin{figure}{section}
\input{TexInput/LatexDefinitions.tex}

\title[Image Labeling Based on Graphical Models Using Wasserstein Messages and Geometric Assignment]{Image Labeling Based on Graphical Models \\ Using Wasserstein Messages and Geometric Assignment}
\author[R.~H\"{u}hnerbein, F.~Savarino, F.~\r{A}str\"{o}m, C.~Schn\"{o}rr]{Ruben H\"{u}hnerbein, Fabrizio Savarino, Freddie \r{A}str\"{o}m, Christoph Schn\"{o}rr}
\address[R.~H\"{u}hnerbein]{Image and Pattern Analysis Group, Heidelberg University, Germany} 
\email{ruben.huehnerbein@iwr.uni-heidelberg.de}
\urladdr{\url{http://ipa.math.uni-heidelberg.de}}
\address[F.~Savarino]{Image and Pattern Analysis Group, Heidelberg University, Germany} 
\email{fabrizio.savarino@iwr.uni-heidelberg.de}
\urladdr{\url{http://ipa.math.uni-heidelberg.de}}
\address[F.~\r{A}str\"{o}m]{Heidelberg Collaboratory for Image Processing, Heidelberg University, Germany} 
\email{freddie.astroem@iwr.uni-heidelberg.de}
\urladdr{\url{https://hciweb.iwr.uni-heidelberg.de/user/fastroem}}
\address[C.~Schn\"{o}rr]{Image and Pattern Analysis Group, Heidelberg University, Germany} 
\email{schnoerr@math.uni-heidelberg.de}
\urladdr{\url{http://ipa.math.uni-heidelberg.de}}
\date{\today} 
\thanks{Support by the German Science Foundation, grant GRK 1653, is gratefully acknowledged.}



\begin{document}

\begin{abstract}
We introduce a novel approach to Maximum A Posteriori inference based on discrete graphical models. By utilizing local Wasserstein distances for coupling assignment measures across edges of the underlying graph, a given discrete objective function is smoothly approximated and restricted to the assignment manifold. A corresponding multiplicative update scheme combines in a single process (i) geometric integration of the resulting Riemannian gradient flow and (ii) rounding to integral solutions that represent valid labelings. Throughout this process, local marginalization constraints known from the established LP relaxation are satisfied, whereas the smooth geometric setting results in rapidly converging iterations that can be carried out in parallel for every edge. 
\end{abstract}

\maketitle
\tableofcontents

\section{Introduction}
\label{sec:Introduction}
\input{TexInput/Introduction}
\section{Image Labeling on the Assignment Manifold}
\label{sec:Labeling-Assignment-Manifold}
\input{TexInput/Assignment-Manifold}

\section{Application to Graphical Models}
\label{sec:Graphical-Models}
\input{TexInput/Graphical-Models}
\section{Implementation}
\label{sec:Numerical-Optimization}
\input{TexInput/Numerical-Optimization}

\section{Experiments}
\label{sec:Experiments}
\input{TexInput/Experiments}

\section{Conclusion}
\label{sec:Conclusion}
\input{TexInput/Conclusion}
%

\appendix
\section{Proofs}
\label{sec:Appendix}
\input{TexInput/Appendix}

\section*{Acknowledgements}
We thank Jan Kuske for sharing with us his framework for running series of experiments efficiently.

\bibliographystyle{amsalpha}
\bibliography{TexInput/assignmentFilter}
\end{document}

%% file: TexInput/LatexDefinitions.tex
\theoremstyle{plain}
\newtheorem{theorem}{Theorem}[section]
\newtheorem{lemma}[theorem]{Lemma}
\newtheorem{proposition}[theorem]{Proposition}
\newtheorem{corollary}[theorem]{Corollary}

\theoremstyle{definition}

\newtheorem{remark}{Remark}[section]

\newcommand{\bitem}{\begin{itemize}}
\newcommand{\eitem}{\end{itemize}}
\newcommand{\mc}[1]{\mathcal{#1}}
\newcommand{\mb}[1]{\mathbb{#1}}

\newcommand{\N}{\mathbb{N}}
\newcommand{\R}{\mathbb{R}}

\newcommand{\B}{\mathbb{B}}

\newcommand{\bpm}{\begin{pmatrix}}
\newcommand{\epm}{\end{pmatrix}}
\newcommand{\bvm}{\begin{vmatrix}}
\newcommand{\evm}{\end{vmatrix}}
\newcommand{\bsm}{\left(\begin{smallmatrix}}
\newcommand{\esm}{\end{smallmatrix}\right)}
\newcommand{\T}{\top}

\newcommand{\ol}[1]{\overline{#1}}
\newcommand{\la}{\langle}
\newcommand{\ra}{\rangle}

\newcommand{\veps}{\varepsilon}

\newcommand{\gdw}{\Leftrightarrow}

\newcommand{\eins}{\mathbb{1}}

\DeclareMathOperator{\Diag}{Diag}

\DeclareMathOperator{\rint}{rint}

\DeclareMathOperator{\argmin}{arg min}
\DeclareMathOperator{\argmax}{arg max}

\DeclareMathOperator{\mean}{mean}

\DeclareMathOperator{\logexp}{logexp}
\DeclareMathOperator{\vecmax}{vecmax}


%% file: TexInput/Introduction.tex
\subsection{Overview and Motivation}
\label{sec:Overview}\label{sec:Overview}
Let $\Omega \subset \R^{2}$ be a domain where image data are observed, and let $\mc{G}=(\mc{V},\mc{E}),\; |\mc{V}|=m$, denote a grid graph embedded into $\Omega$. Each vertex $i \in \mc{V}$ indexes the location of a pixel, to which a random variable 
\begin{equation}\label{eq:def-X-labels}
x_{i} \in \mc{X}=\{\ell_{1},\dotsc,\ell_{n}\}
\end{equation}
is assigned which takes values in a finite set $\mc{X}$ of \textit{labels}. 
The \textit{image labeling problem} is the task to assign to each $x_{i}$ a label such that the discrete \textit{objective function}
\begin{equation}\label{eq:def-E}
\min_{x \in \mathcal{X}^{m}} E(x),\qquad
E(x) = \sum_{i \in \mathcal{V}} E_{i}(x_{i}) + \sum_{ij \in \mathcal{E}} E_{ij}(x_{i},x_{j})
\end{equation}
is minimized. This function comprises for each pixel $i \in \mc{V}$ local energy terms $E_{i}(x_{i})$ that evaluate local label predictions for each possible value of $x_{i} \in \mc{X}$. In addition, $E(x)$ comprises for each edge $ij \in \mc{E}$ local distance functions $E_{ij}(x_{i},x_{j})$ that evaluate the joint assignment of labels to $x_{i}$ and $x_{j}$. If the local energy functions $E_{ij}(x_{i},x_{j})=d(x_{i},x_{j})$ are defined by a metric $d \colon \mc{X} \times \mc{X} \to \R$, then \eqref{eq:def-E} is called the \textit{metric labeling problem} \cite{Kleinberg2002}. In general, the presence of these latter terms makes image labeling a combinatorially hard task.  
Function $E(x)$ has the common format of variational problems for image analysis comprising a data term and a regularizer. From a Bayesian perspective, therefore, minimizing $E$ corresponds to \textit{Maximum A-Posteriori} inference with respect to the probability distribution $p(x) = \frac{1}{Z} \exp(-E(x))$. We refer to \cite{Kappes2015} for a recent survey on the image labeling problem and on algorithms for solving either approximately or exactly problem \eqref{eq:def-E}.

A major class of algorithms for approximately solving \eqref{eq:def-E} is based on the \textit{linear} (programming) \textit{relaxation} \cite{WernerPAMI07} (see Section \ref{sec:LP-relaxation} for details)
\begin{equation}\label{eq:LP-relaxation}
\min_{\mu \in \mc{L}_{\mc{G}}} \la \theta, \mu \ra.
\end{equation}
Solving the linear program (LP) \eqref{eq:LP-relaxation} returns a globally optimal \textit{relaxed indicator vector} $\mu$ whose components take values in $[0,1]$. If $\mu$ is a binary vector, then it corresponds to a solution of problem \eqref{eq:def-E}. In realistic applications, this is not the case, however, and the relaxed solution $\mu$ has to be rounded to an integral solution in a post-processing step. 

In this paper, we present an alternative inference algorithm that deviates from the traditional two-step process: convex relaxation and rounding. It is based on the recently proposed geometric approach \cite{Astroem2017} to image labeling. The basic idea underlying this approach is to restrict indicator vector fields to the relative interior of the probability simplex, equipped with the Fisher-Rao metric, and to regularize label assignments by iteratively computing Riemannian means (see Section \ref{sec:Labeling-Assignment-Manifold} for details). This results in a highly parallel, multiplicative update scheme, that rapidly converges to an integral solution. Because this model of label assignment does not interfere with data representation, the approach applies to any data given in a metric space. The recent paper \cite{Bergmann2017a} reports a convergence analysis and the application of our scheme to a range of challenging labeling problems of manifold-valued data.

Adopting this starting point, the objectives of the present paper are:
\begin{itemize}
\item Show how the approach \cite{Astroem2017} can be used to efficiently compute high-quality (low-energy) solution for an arbitrary given instance of the labeling problem \eqref{eq:def-E}.
\item Devise a novel labeling algorithm that tightly integrates both relaxation and rounding to an integral solution in a single process.
\item Stick to the smooth geometric model suggested by \cite{Astroem2017} so as to overcome the inherent non-smoothness of convex polyhedral relaxations and the slow convergence of corresponding first-order iterative methods of convex programming.
\end{itemize}
Regarding the last point, a key ingredient of our approach is a \textit{smooth} approximation 
\begin{equation}\label{eq:def-J-smooth}
E_{\tau}(\mu_{\mc{V}}) = \la \theta_{\mc{V}}, \mu_{\mc{V}} \ra
+ \sum_{ij \in \mc{E}} d_{\theta_{ij},\tau}(\mu_{i},\mu_{j}),\qquad \tau > 0
\end{equation}
of problem \eqref{eq:LP-relaxation}, where $d_{\theta_{ij},\tau}$ denotes the local \textit{smoothed} Wasserstein distance 
between the discrete label assignment measures $\mu_{i}, \mu_{j}$ coupled along the edge $ij$ of the underlying graph. Besides achieving the degree of smoothness required for our geometric setting, this approximation also properly takes into account the regularization parameters that are specified in terms of the local energy terms $E_{ij}$ of the labeling problem \eqref{eq:def-E}. Our approach restricts the function $E_{\tau}$ to the so-called assignment manifold and iteratively determines a labeling by tightly combining geometric optimization with rounding to an integral solution in a smooth fashion.

\subsection{Related Work}
\label{sec:Related-Work}

Problem sizes of linear program (LP) \eqref{eq:LP-relaxation} are large in typical applications of image labeling, which rules out the use of standard LP codes. In particular, the theoretically and practically most efficient interior point methods based on self-concordant barrier functions
\cite{nesterov1987,Renegar1995} are infeasible due to the dense linear algebra steps required to determine search and update directions. 

Therefore, the need for dedicated solvers for the LP relaxation \eqref{eq:LP-relaxation} has stimulated a lot of research. A prominent example constitute subclasses of objective functions \eqref{eq:def-E} as studied in \cite{EnergyGraphCuts-PAMI04}, in particular binary submodular functions, that enable to reformulate the labeling problem as maximum-flow problem in an associated network and the application of discrete combinatorial solvers \cite{GraphCuts-01,Boykov2004}. 

Since the structure of such algorithms inherently limits fine-grained parallel implementations, however, \textit{belief propagation} and variants \cite{Yedidia-GenBP-05} have been popular among practitioners. These fixed point schemes in terms of dual variables iteratively enforce the so-called local polytope constraints that define the feasible set of the LP relaxation \eqref{eq:LP-relaxation}. They can be efficiently implemented  using `message passing' and exploit the structure of the underlying graph. Although convergence is not guaranteed on cyclic graphs, the performance in practice may be good \cite{TRBP-CPLEX-06}. The theoretical deficiencies of basic belief propagation in turn stimulated research on \textit{convergent} message passing schemes, either using heuristic damping or utilizing in a more principled way \textit{convexity}. Prominent examples of the latter case are 
\cite{WaiJaaWil-TRMAP-05, Hazan2010}. We refer to  \cite{Kappes2015} for many more references and a comprehensive experimental evaluation of a broad range of algorithms for image labeling.

The feasible set of the relaxation \eqref{eq:LP-relaxation} is a superset of the original feasible set of \eqref{eq:def-E}. Therefore, globally optimal solutions to \eqref{eq:LP-relaxation} generally do \textit{not} constitute valid labelings but comprise \textit{non}-integral components $\mu_{i}(x_{i}) \in (0,1),\, x_{i} \in \mc{X},\, i \in \mc{V}$. Randomized rounding schemes for converting a relaxed solution vector $\ol{\mu}$ to a valid labeling $x \in \mc{X}^{m}$, along with suboptimality bounds, were studied in \cite{Kleinberg2002, Chekuri:2005aa}. The problem to infer components $x^{\ast}_{i}$ of the unknown globally optimal \textit{combinatorial} labeling that minimizes \eqref{eq:def-E}, through partial optimality and persistency, was studied in \cite{Swoboda2016}. We refer to \cite{WernerPAMI07} for the history and more information about the LP relaxation of labeling problems, and to \cite{WainrightJordan08} for connections to discrete probabilistic graphical models from the variational viewpoint.

The approach \cite{Ravikumar:2010aa} applies the mirror descent scheme \cite{Nemirovsky:1983aa} to the LP \eqref{eq:LP-relaxation}. This amounts to sequential proximal minimization \cite{Rockafellar:1976aa}, yet using a Bregman distance as proximity measure instead of the squared Euclidean distance \cite{Censor1992}. A key technical aspect concerns the proper choice of entropy functions related to the underlying graphical model, that qualify as convex functions of Legendre type (cf.~\cite{Bauschke:1997aa}). The authors of \cite{Ravikumar:2010aa} observed a fast convergence rate. However, the scheme  does not scale up to the typically large problem sizes used in image analysis, especially when graphical models with higher edge connectivity are considered, due to the memory requirements when working entirely in the primal domain.

\textit{Optimal transport} and the \textit{Wasserstein distance} have become a major tool of signal modeling and analysis \cite{Kolouri2017}. 
In connection with the metric labeling problem, using the Wasserstein distance (aka.~optimal transport costs, earthmover metrics) was proposed before by \cite{Archer:2004aa} and \cite{Chekuri:2005aa}. These works study bounds on the integrality gap of an `earthmover LP' and performance guarantees of rounding procedures applied as post-processing. While the earthmover LP corresponds to our approach \eqref{eq:def-J-smooth} \textit{without} smoothing, authors do not specify how to solve such LPs efficiently, especially when the LP relates to a large-scale graphical models as in image analysis. Moreover, the bounds derived by \cite{Archer:2004aa} become weak with increasing numbers of variables, which are fairly large in typical problems of image analysis. In contrast, the focus of the present paper is on a \textit{smooth geometric} problem reformulation that scales well with both the problem size and the number of labels, and performs rounding \textit{simultaneously}. If and how theoretical guarantees regarding the integrality gap and rounding carry over to our setting, is an interesting open research problem of future research.

Regarding the finite-dimensional formulation of optimal discrete transport in terms of linear programs, the design of efficient algorithms for large-scale problems requires sophisticated techniques \cite{Schmitzer2016}. The problems of discrete optimal transport studied in this paper, in connection with the local Wasserstein distances of \eqref{eq:def-J-smooth}, have a small or moderate size ($n^{2}$: number of labels squared). We apply the standard device of enhancing convexity through entropic regularization, which increases smoothness in the dual domain. We refer to \cite{Schneider1990} and \cite[Ch.~9]{Brualdi2006} for basic related work and the connection to matrix scaling algorithms and the history. When entropic regularization is very weak and for large problem sizes, the related fixed point iteration suffers from numerical instability, and dedicated methods for handling them have been proposed \cite{Schmitzer16}. Smoothing of the Wasserstein distance and  Sinkhorn's algorithm has become popular in machine learning due to \cite{Cuturi2013}. The authors of \cite{Peyre2015,Cuturi2016} comprehensively investigated barycenters and interpolation based on the Wasserstein distance. Our approach to image labeling, in conjunction with the geometric approach of \cite{Astroem2017}, is novel and elaborates \cite{Astroem2017a}.

Finally, since our approach is defined on a graph and works with data on a graph, our work may be assigned to the broad class of nonlocal methods for image analysis on graphs, from a more general viewpoint. Recent major related work includes \cite{Bertozzi:2012aa} on the connection between the Ginzburg-Landau functional for binary regularized segmentation and spectral clustering, and \cite{Bergmann:2017aa} on generalizing PDE-like models on graphs to manifold-valued data. We refer to the bibliography in these works and to the seminal papers \cite{Ambrosio:1989aa} on regularized variational segmentation using $\Gamma$-convergence and to \cite{Gilboa:2008aa,Elmoataz:2008aa} on nonlocal variational image processing on graphs, that initiated these fast evolving lines of research. The focus on the present paper however is on discrete graphical models and the corresponding labeling problem, in terms of any discrete objective function of the form \eqref{eq:def-E}.

\subsection{Contribution and Organization}

We collect basic notation, background material and details of the LP relaxation \eqref{eq:LP-relaxation} in Section \ref{sec:Preliminaries}. 
Section \ref{sec:Labeling-Assignment-Manifold} summarizes the basic concepts of the geometric labeling approach of \cite{Astroem2017}, in particular the so-called assignment manifold, and the general framework of \cite{SavHuen17} for numerically integrating Riemannian gradient flows of functionals defined on the assignment manifold. This section provides the basis for the two subsequent sections that contain our main contribution.

Section \ref{sec:Gradients} studies the approximation \eqref{eq:def-J-smooth} and provides explicit expressions for the Riemannian gradient of the restriction of $E_{\tau}$ to the assignment manifold. A key property of this set-up concerns the local polytope constraints that define the feasible set $\mc{L}_{\mc{G}}$ of the LP relaxation \eqref{eq:LP-relaxation}: by construction, they are \textit{always} satisfied throughout the resulting iterative process of label assignment. Thus, our formulation is \textit{both more tightly constrained and smooth}, in contrast to the established convex programming approaches based on \eqref{eq:LP-relaxation}.

Section \ref{sec:Graphical-Models} details the combination of all ingredients into a \textit{single}, smooth, geometric approach that performs simultaneously minimization of the objective function \eqref{eq:def-J-smooth} and rounding to an integral solution (label assignment). This tight integration is a second major property that distinguishes our approach from related work. Section \ref{sec:Graphical-Models} also explains the notion `Wasserstein messages' in the title of this paper due to the dual variables that are numerically utilized to evaluate gradients of local Wasserstein distances, akin to how dual (multiplier) variables in basic belief propagation schemes are used to enforce local marginalization constraints. Unlike the latter computations they have the structure of message passing on a dataflow architecture, however, message passing induced by our approach is fully parallel along all edges of the underlying graph and hence  resembles the structure of numerical solvers for PDEs.

The remaining two sections are devoted to numerical evaluations of our approach. To keep this paper at a reasonable length, we merely consider the most elementary iterative update scheme, based on the geometric integration of the Riemannian gradient flow with the (geometric) explicit Euler scheme. The potential of the framework outlined by \cite{SavHuen17} for more sophisticated numerical schemes will be explored elsewhere along with establishing bounds for parameter values that provably ensure stability of numerical integration of the underlying gradient flow. Furthermore, working out any realistic application is beyond the scope of this paper. Rather, the experimental results demonstrate major properties of our approach.

Section \ref{sec:Numerical-Optimization} provides all details of our implementation that are required to reproduce our computational results. Section \ref{sec:Experiments} reports and discusses the results of four types of experiments:
\begin{enumerate}
\item
The interplay between two parameters $\tau$ and $\alpha$ that control smoothness of the approximation \eqref{eq:def-J-smooth} and rounding, respectively, is studied. In order to miminize efficiently \eqref{eq:def-E}, the Riemannian flow with respect to the smooth approximation \eqref{eq:def-J-smooth} must reveal proper descent directions. This imposes an upper bound on the smoothing parameter $\tau$. Naturally, the effect of rounding has to be stronger to make the iterative process converge to an integral solution. A corresponding choice of $\alpha$ controls the compromise between quality of integral labelings in terms of the energy \eqref{eq:def-J-smooth} and speed of convergence. Fortunately, the upper bound on $\tau$ is large enough to achieve attractive convergence rates.
\item
We comprehensively explore numerically the entire model space of the minimal binary graphical model on the \textit{cyclic} triangle graph $\mc{K}^{3}$, whose relaxation in terms of the so-called \textit{local} polytope already constitutes a superset of the \textit{marginal} polytope as admissible set for valid integral labelings. In this way, we explore the performance of our approach in view of the LP relaxation and established inference based on convex programming, and with respect to the (generally intractable) feasible set of integral solutions. Corresponding phase diagrams display and support quantitatively the trade-off between accuracy of optimization and rate of convergence through the choice of the single parameter $\alpha$.
\item 
A labeling problem of the usual size was conducted to confirm and demonstrate that the finding of the preceding points for `all' models on $\mc{K}^{3}$ also hold in a typical application.  A comparison to sequential tree-reweighted message passing (TRWS) \cite{Kolmogorov-TRMP06} which defines the state of the art, and to loopy belief propagation (BP) based on the OpenGM package \cite{opengm-library}, shows that our approach is on par with these methods regarding the energy level $E(x)$ of the resulting labeling $x$.
\item
A final experiment based on the graphical model with a pronounced non-uniform (non-Potts) prior demonstrates that our approach is able to perform inference for any given graphical model.
\end{enumerate}
We conclude in Section \ref{sec:Conclusion} and relegate some proofs to an Appendix in order not to  interrupt too much the overall line of reasoning.

\section{Preliminaries}\label{sec:Preliminaries}
We introduce basic notation in Section \ref{sec:Notation} and the common linear programming (LP) relaxation of the labeling problem in Section \ref{sec:LP-relaxation}. In order to clearly distinguish between the LP relaxation and our geometric approach to the labeling problem based on \cite{Astroem2017} (see Section \ref{sec:Assignment-Manifold}), we keep the standard notation in the literature for the former approach and the notation from \cite{Astroem2017} for the latter one. Remark \ref{rem:W-vs-mu} below identifies variables of both approaches that play a similar role.

\subsection{Basic Notation}
\label{sec:Notation} 
For an \textit{undirected} graph $\mc{G}=(\mc{V},\mc{E})$, the adjacency relation $i \sim j$ means that vertices $i$ and $j$ are connected by an undirected edge $ij \in \mc{E}$, where the latter denotes the \textit{unordered} pair $\{i,j\}=ij=ji$. The neighbors of vertex $i$ form the set 
\begin{equation}\label{eq:def-Ni}
\mc{N}(i)=\{j \in \mc{V} \colon i \sim j\}
\end{equation}
of all vertices adjacent to $i$, and its cardinality $d(i)=|\mc{N}(i)|$ is the degree of $i$. $\mc{G}$ is turned into a \textit{directed} graph by assigning an \textit{orientation} to every edge $ij$, which then form \textit{ordered} pairs denoted by $(i,j)=ij \neq ji=(j,i)$. We only consider graphs \textit{without multiple} edges between any pair of nodes $i,j \in \mc{V}$.

We use the abbreviation $[n]=\{1,2,\dotsc,n\}$ for $n \in \N$. 
$\ol{\R} = \R\cup\{+\infty\}$ denotes the extended real line. All vectors are regarded as column vectors, and $x^{\T}$ denotes transposition of a vector $x$. We ignore transposition however when vectors are explicitly specified by their components; e.g.~we write $x=(y,z)$ instead of the more cumbersome $x=(y^{\T},z^{\T})^{\T}$. We set $\eins_{n} = (1,1,\dotsc,1) \in \N^{n}$ and write $\eins$ if $n$ is clear from the context. $\la x, y \ra = \sum_{i \in [n]} x_{i} y_{i}$ denotes the Euclidean inner product. Given a matrix 
\begin{equation}
A = \bpm A_{1} \\ \vdots \\ A_{m} \epm
= \bpm A^{1} \dotsc A^{n}\epm \in \R^{m \times n},
\end{equation}
we denote the row vectors by $A_{i},\; i \in [m]$ and the column vectors by $A^{j},\; j \in [n]$. The canonical matrix inner product is $\la A, B \ra = \tr(A^{\T} B)$, where $\tr$ denotes the trace of a matrix, i.e.~$\tr(A^{\T} B)=\sum_{i \in [m]} \la A_{i}, B_{i} \ra = \sum_{j \in [n]} \la A^{j}, B^{j} \ra = \sum_{i \in [m], j \in [n]} A_{ij} B_{ij}$. Superscripts in brackets, e.g.~$A_{i}^{(k)}$, index iterative steps.

The set of nonnegative vectors $x \in \R^{n}$ is denoted by $\R_{+}^{n}$ and the set of strictly positive vectors by $\R_{++}^{n}$. 
The probability simplex $\Delta_{n}=\{p \in \R_{+}^{n} \colon \la\eins_{n},p\ra=1\}$ contains all discrete distributions on $[n]$. A doubly stochastic matrix $\mu_{ij} \in \R_{+}^{n \times n}$, also called \textit{coupling measure} in this paper in connection with discrete optimal transport, has the property: $\mu_{ij} \eins_{n} \in \Delta_{n}$ and $\mu_{ij}^{\T} \eins_{n} \in \Delta_{n}$. We denote these two \textit{marginal distributions} of $\mu_{ij}$ by $\mu_{i}$ and $\mu_{j}$, respectively, and the linear mapping for extracting them by 
\begin{subequations}
\begin{align}\label{eq:def-A}
\mc{A} \colon \R^{n \times n} \to \R^{2 n},\qquad
\mu_{ij} &\mapsto \mc{A} \mu_{ij} = \bpm \mu_{ij} \eins_{n} \\ \mu_{ij}^{\T} \eins_{n} \epm
= \bpm \mu_{i} \\ \mu_{j} \epm.
\intertext{Its transpose is given by}
\label{eq:def-AT}
\mc{A}^{\T} \colon \R^{2 n} \to \R^{n \times n},\qquad
(\nu_{i},\nu_{j}) &\mapsto 
\mc{A}^{\T} \bpm \nu_{i} \\ \nu_{j} \epm 
= \nu_{i} \eins_{n}^{\T} + \eins_{n} \nu_{j}^{\T}.
\end{align}
\end{subequations}
The kernel (nullspace) of a linear mapping $\mc{A}$ is denoted by $\mc{N}(\mc{A})$ and its range by $\mc{R}(\mc{A})$.  

The functions $\exp, \log$ apply \textit{componentwise} to strictly positive vectors $x \in \R_{++}^{n}$, e.g.~$e^{x}=(e^{x_{1}},\dotsc,e^{x_{n}})$, and similarly for strictly positive matrices. Likewise, if $x, y \in \R_{++}^{n}$, then we simply write 
\begin{equation}
x \cdot y = (x_{1} y_{1}, \dotsc, x_{n} y_{n}),\qquad
\frac{x}{y} = \Big(\frac{x_{1}}{y_{1}},\dotsc,\frac{x_{n}}{y_{n}}\Big)
\end{equation}
for the \textit{componentwise} multiplication and division.

We define $\mc{F}_{0}$ to be the class of proper, lower-semicontinuous and convex functions defined on $\R^{n}$. For any function $f \in \mc{F}_{0}$, $\partial f(x)$ denotes its subdifferential at $x$, and the conjugate function $f^{\ast}\in \mc{F}_{0}$ of $f$ is given by the Legendre-Fenchel transform (cf.~\cite[Section 11.A]{Rockafellar2009})
\begin{equation}\label{eq:def-f-conjugate}
f^{\ast}(y) = \sup_{x \in \R^{n}}\{\la y, x \ra - f(x)\}.
\end{equation}
For a given closed convex set $C$, its indicator function is denoted by
\begin{equation}
\delta_{C}(x) = \begin{cases}
0, &\text{if}\; x \in C, \\
+\infty, &\text{otherwise},
\end{cases}
\end{equation}
and 
\begin{equation}\label{eq:def-PC}
P_{C} \colon \R^{n} \to C,\qquad
P_{C}(x) = \argmin_{y \in C} \|x-y\|
\end{equation}
denotes the orthogonal projection onto $C$. The shorthand ``s.t.'' means: ``subject to'' in connection to the specification of constraints. 

The \textit{log-exponential} function $\logexp_{\veps} \in \mc{F}_{0}$ is defined as 
\begin{subequations}\label{eq:def-logexp}
\begin{align}
\logexp_{\veps}(x) &= \veps \log\bigg(\sum_{i \in [n]} e^{\frac{x_{i}}{\veps}}\bigg).
\intertext{
It uniformly approximates the function $\vecmax \in \mc{F}_{0}$ \cite[Ex.~1.30]{Rockafellar2009}, i.e.}
\lim_{\veps\searrow 0} \logexp_{\veps}(x) 
&= \vecmax(x)=\max\{x_{i}\}_{i \in [n]}.
\end{align}
\end{subequations}
We will use the following basic result from convex analysis (cf., e.g.~\cite[Ch.~11]{Rockafellar2009}), where $\partial f(x)$ denotes the subdifferential of a function $f \in \mc{F}_{0}$ at $x$.
\begin{theorem}[inversion rule for subgradients]\label{tmh:inversion-subgradients}
Let $f \in \mc{F}_{0}$. Then
\begin{equation}
\hat p \in \partial f(\hat x)
\quad\gdw\quad
\hat x \in \partial f^{\ast}(\hat p)
\quad\gdw\quad
f(\hat x) + f^{\ast}(\hat p) = \la \hat p, \hat x \ra
\end{equation}
\end{theorem}
We will also apply the following classical theorem of Danskin and its extension by Rockafellar.
\begin{theorem}[\cite{Danskin1966,DanskinCorrolar-91}]\label{thm:Danskin}
Let $f(z)=\max_{w \in W} g(z,w)$, where $W$ is compact and the function $g(\cdot,w)$ is differentiable and $\nabla_{z}g(z,w)$ is continuously depending 
on $(z,w)$. If in addition $g(z,w)$ is convex in $z$, and if $\overline{z}$ is a point such that $\arg\max_{w \in W} g(\overline{z},w) = \{\overline{w}\}$, 
then $f$ is differentiable at $\overline{z}$ with 
\begin{equation}\label{eq:Danskin}
\nabla f(\overline{z}) = \nabla_{z}g(\overline{z},\overline{w}).
\end{equation}
\end{theorem}
\subsection{The Local Polytope Relaxation of the Labeling Problem}
\label{sec:LP-relaxation}

We sketch in this section the transition from the discrete energy minimization problem \eqref{eq:def-E} to the LP relaxation \eqref{eq:LP-relaxation} and thereby introduce additional notation needed in subsequent sections.

The first step concerns the definition of \textit{local model parameter vectors} and \textit{matrices}
\begin{equation}\label{eq:def-theta-i-theta-ij}
\theta_{i} := \big(\theta_{i}(\ell_{k})\big)_{k\in[n]} \in \mathbb{R}^{n}, \qquad
\theta_{ij} := \big(\theta_{ij}(\ell_{k},\ell_{r})\big)_{k,r \in [n]} \in \mathbb{R}^{n \times n},\qquad
\text{with}\quad \ell_k, \ell_r \in \mc X,
\end{equation}
which merely encode the values of the discrete objective function \eqref{eq:def-E}: $\theta_{i}(\ell_{k}) = E_{i}(\ell_{k})$, $\theta_{ij}(\ell_{k},\ell_{r}) = E_{ij}(\ell_{k},\ell_{r})$.
These local terms are commonly called \textit{unary} and \textit{pairwise terms} in the literature. Recall from the discussion of \eqref{eq:def-E} that the unary terms represent the data and the pairwise terms specify a regularizer. All these local terms are indexed by the vertices $i \in \mc{V}$ and edges $ij \in \mc{E}$ of the underlying graph $\mc{G}=(\mc{V},\mc{E})$ and assembled into the vectors
\begin{equation}\label{eq:def-theta-V-theta-E}
\theta := (\theta_{\mc{V}},\theta_{\mc{E}}),\qquad \text{where}\qquad
\theta_{\mc{V}} := (\theta_{i})_{i\in\mc V},\qquad \text{and}\qquad
\theta_{\mc{E}} := (\theta_{ij})_{ij \in \mc{E}},
\end{equation}
where we conveniently regard $\theta_{ij} \in \R^{n^{2}}$ either as local vector or as local matrix $\theta_{ij} \in \R^{n \times n}$, depending on the context. Next we define \textit{local indicator vectors}
\begin{equation}\label{eq:def-mu-i-mu-ij}
\mu_{i} := \big(\mu_{i}(\ell_{k})\big)_{k\in[n]} \in \{0,1\}^{n},\quad
\mu_{ij} := \big(\mu_{ij}(\ell_{k},\ell_{r})\big)_{k, r\in[n]} \in \{0,1\}^{n \times n},\quad
\text{with}\quad \ell_k, \ell_r \in \mc X,
\end{equation}
indexed in the same way as \eqref{eq:def-theta-i-theta-ij} and assembled into the vectors
\begin{equation}\label{eq:def-mu-V-mu-E}
\mu := (\mu_{\mc{V}},\mu_{\mc{E}}),\qquad \text{where}\qquad
\mu_{\mc{V}} := (\mu_{i})_{i \in \mc{V}},\qquad \text{and}\qquad
\mu_{\mc{E}} := (\mu_{ij})_{ij \in \mc{E}}.
\end{equation}
The combinatorial optimization problem \eqref{eq:def-E} now reads $\min_{\mu} \la \theta, \mu \ra$. The corresponding linear programming relaxation consists in replacing the discrete feasible set of \eqref{eq:def-mu-i-mu-ij} by the convex polyhedral sets
\begin{subequations}\label{eq:def-mu-Pi}
\begin{align}
\mu_{i} &\in \Delta_{n},\quad 
\mu_{ij} \in \Pi(\mu_{i},\mu_{j}),\qquad 
i \in \mc{V}, \; ij \in \mc{E}, 
\label{eq:def-mui-muj-convex} \\ \label{eq:def-Pi-mui-muj}
\Pi(\mu_{i},\mu_{j}) &= \big\{\mu_{ij} \in \R_{+}^{n \times n} \colon \mu_{ij} \eins = \mu_{i},\; \mu_{ij}^{\T} \eins = \mu_{j},\; \mu_{i}, \mu_{j} \in \Delta_{n}\big\}.
\end{align}
\end{subequations}
As a result, the linear programming relaxation \eqref{eq:LP-relaxation} of \eqref{eq:def-E} reads more explicitly
\begin{equation}\label{eq:LP-relaxation-b}
\min_{\mu \in \mc{L}_{\mc{G}}} \langle \theta, \mu \rangle 
= \min_{\mu \in \mc{L}_{\mc{G}}} \la \theta_{\mc{V}}, \mu_{\mc{V}} \ra + \la \theta_{\mc{E}}, \mu_{\mc{E}} \ra, 
\end{equation}
where the so-called \textit{local polytope} $\mc{L}_{\mc{G}}$ is the set of all vectors $\mu$ of the form \eqref{eq:def-mu-V-mu-E} with components ranging over the sets specified by  \eqref{eq:def-mu-Pi}. The adjective ``local'' refers to the local marginalization constraints \eqref{eq:def-Pi-mui-muj}.

%% file: TexInput/Assignment-Manifold.tex

This section sets the stage for our approach to solving approximately the labeling problem \eqref{eq:def-E}. 
We first introduce in Section \ref{sec:Assignment-Manifold} in terms of the assignment manifold the setting for the smooth approach to image labeling \cite{Astroem2017}, to be sketched in Section \ref{sec:labeling-by-assignment}. Section \ref{sec:Geometric Integration} summarizes the general framework of \cite{SavHuen17} for numerically integrating Riemannian gradient flows of functionals defined on the assignment manifold.

\subsection{The Assignment Manifold} \label{sec:Assignment-Manifold}

The relative interior of the probability simplex $\mc S := \rint(\Delta_n)$, given by
$\mc S = \{ p\in \mb{R}^n_{++} \colon \la \eins, p\ra = 1\}$, is a $n-1$ dimensional smooth manifold with constant tangent space
\begin{equation}\label{eq:def-tangent-space-T}
  T_p\mc S = \{ v \in \mb{R}^n\colon \la \eins, v\ra = 0\} =: T \subset \mb{R}^n\;,\quad \text{for}\quad p\in \mc S.
\end{equation}
Due to $\la \eins, v\ra=0$ for all $v\in T$, we have the orthogonal decomposition $\mb{R}^n = T \oplus \mb R \eins$. The orthogonal projection onto $T$ is given by
\begin{equation}\label{eq:projection_matrix_Rn_to_T}
  P_T \colon \mb{R}^n \to T\;,\quad x\mapsto P_T(x) = x - \frac{1}{n} \la \eins, x \ra \eins = \left ( I - \frac{1}{n} \eins\eins^{\T} \right ) x,
\end{equation}
where $I$ denotes the $(n \times n)$ identity matrix.
The manifold $\mc S$ becomes a Riemannian manifold  by endowing it with the Fisher-Rao metric. At a point $p \in \mc S$, this metric is given by
\begin{equation}
  \la \cdot, \cdot \ra_p \colon T_p \mc S \times T_p \mc S \to \mb R\;,\quad
  (u, v) \mapsto \la u, v \ra_p = \Big\la \frac{u}{\sqrt{p}}, \frac{v}{\sqrt{p}} \Big\ra. 
\end{equation}
In this setting, there is an important map, called the \textit{lifting map} (cf.~\cite[Def.~4]{Astroem2017}), defined as
\begin{equation}\label{eq:def-tilde-L}
  \tilde{L} \colon \mb{R}^{n} \to \mc S,\qquad x \mapsto \tilde{L}_p(x) := \frac{p\cdot e^x}{\la p, e^x\ra}.
\end{equation}
By restricting $\tilde{L}$ onto the tangent space, we obtain a diffeomorphism
\begin{equation}\label{eq:def-L}
L := \tilde{L}|_T \colon T \to \mc{S},\qquad
\tilde{L} = L \circ P_T.
\end{equation}
This restricted lifting map $L$ is also a local first order approximation to the exponential map of the Riemannian manifold $\mc S$ 
(cf~\cite[Prop.~3]{Astroem2017}), with the inverse mapping given by
\begin{equation}\label{eq:def-L-inverse}
  L_p^{-1} \colon \mc S \to T\;, \quad q \mapsto L_p^{-1}(q) := P_T\Big(\log\frac{q}{p}\Big).
\end{equation}\hfill

The \textit{assignment manifold} is defined as the product manifold $\mc W := \prod_{i\in[m]} \mc S$ and can be identified
with the space $\mc W = \{ W \in \mb{R}^{m\times n}_{++} \colon W \eins = \eins \}$ of row-stochastic matrices with full support. With the Riemannian
product metric, $\mc W$ also becomes a Riemannian manifold with constant tangent space
\begin{equation}
  T_W \mc W = \prod_{i\in[m]} T = \{ V \in \mb{R}^{m\times n}\colon V\eins = 0\} =: T^m \qquad \text{at}\quad W \in \mc W.
\end{equation}
The Fisher-Rao product metric reads
\begin{equation}
  \la U, V \ra_W = \sum_{i\in[m]} \Big\la \frac{U_i}{\sqrt{W_i}}, \frac{V_i}{\sqrt{W_i}}\Big\ra\qquad \text{at}\quad W \in \mc W,\quad U, V \in T^m.
\end{equation}
The orthogonal decomposition of $T$ induces the orthogonal decomposition 
\begin{equation}
\mb{R}^{m\times n} = T^m \oplus \{ \lambda \eins_n^\top \in \mb{R}^{m\times n}\colon \lambda \in \mb{R}^m \}
\end{equation} 
together with the orthogonal projection
\begin{equation}\label{eq:projection_matrix_Rn_to_Tm}
  P_{T^m} \colon \mb{R}^{m\times n} \to T^m,\qquad X \mapsto P_{T^m}(X) =  X \left ( I - \frac{1}{n} \eins\eins^{\T} \right ).
\end{equation}
Thus, the projection of a matrix $X$ onto $T^m$ is just the 
projection \eqref{eq:projection_matrix_Rn_to_T} applied to every row of $X$. The lifting map, the restricted lifting map and its inverse are naturally extended to
\begin{equation}\label{eq:def-L-W}
  \tilde{L}_W \colon \mb{R}^{m\times n} \to \mc W,\qquad  L_W \colon T^m \to \mc W\qquad \text{and} \qquad L_W^{-1}\colon \mc W \to T^m
\end{equation}
for every $W \in \mc W$, by applying $\tilde{L}\colon \mb{R}^n \to \mc S$, $L\colon T \to \mc S$ and $L^{-1}\colon \mc S \to T$ from \eqref{eq:def-tilde-L}, \eqref{eq:def-L}, \eqref{eq:def-L-inverse} to every row,
\begin{equation}
  \big(\tilde{L}_W(X)\big)_i := \tilde{L}_{W_i}(X_i),\qquad  \big(L_W(V)\big)_i := L_{W_i}(V_i)\qquad 
  \text{and} \qquad \big(L_W^{-1}(Q)\big)_i := L_{W_i}^{-1}(Q_i),
\end{equation}
for $i\in [m]$, $X \in \mb{R}^{m\times n}$, $V\in T^m$ and $Q \in \mc W$.

\subsection{Image Labeling on $\mc W$}
\label{sec:labeling-by-assignment}
In \cite{Astroem2017} the following approach was proposed. Let $\mc G = (\mc V, \mc E)$ be a graph with vertex set $\mc V = [m]$.
Suppose a function is given on this graph with values in some feature space $\mc F$,
\begin{equation}
  f\colon \mc V = [m] \to \mc F,\qquad i\mapsto f_i.
\end{equation}
Furthermore, let the set $\mc{X}=\{\ell_{1},\dotsc,\ell_{n}\}$ from \eqref{eq:def-X-labels} denote a set of prototypes or labels (possibly $\mc X \subset \mc F$) 
and assume a distance function is specified,
\begin{equation}
  d\colon \mc F \times \mc X \to \mb R,
\end{equation}
measuring how well a feature is represented by a certain prototype. We are interested in the assignment of the prototypes to the data in terms
of an \textit{assignment matrix} $W \in \mc W \subset \mb{R}^{m \times n}$. The elements of $W$ can be interpreted as the \textit{posterior probability}
\begin{equation}
  W_{i,j} = \Pr(\ell_j| f_i), \quad i\in [m],\quad j\in[n],
\end{equation}
that $\ell_j$ generated the observation $f_i$. The assignment task of determining an optimal assignment $W^\ast$ can thus be interpreted as finding
an `explanation' of the data in terms of the prototypes $\mc X$.
\begin{remark}[$W$ vs.~$\mu$]\label{rem:W-vs-mu}
Each row vector $W_{i},\, i \in [m]$ plays the role of a corresponding vector $\mu_{i}$ of the basic LP relaxation as defined by \eqref{eq:def-mu-i-mu-ij}, with relaxed domain due to \eqref{eq:def-mu-Pi}. Unlike $\mu_{i}$, however, vectors $W_{i} \in \R_{++}^{n}$ always have full support and live on the manifold $\mc{S}$.
\end{remark}

The objective function for measuring the quality of an assignment involves three matrices defined next. First, all distance information between observed feature vectors and prototypes (labels) 
are gathered by the \textit{distance matrix}
\begin{equation}\label{eq:def-D-matrix}
  D \in \mb{R}^{m\times n},\qquad 
  D_{i,j} = d(f_i, \ell_j)
\end{equation}
and then lifted onto the assignment manifold at $W \in \mc W$. By using \eqref{eq:def-L-W} we obtain the \textit{likelihood matrix}
\begin{equation}\label{eq:-def-L-matrix}
  L = \tilde{L}_W\Big(-\frac{1}{\rho} D\Big) = L_W\Big(-\frac{1}{\rho}P_{T^m}(D)\Big), \quad \rho > 0
\end{equation}
where each row $i$ of $L$ is given by $L_i = \tilde{L}_{W_i}(-\frac{1}{\rho} D_i)$ and $P_{T^m}$ is given by \eqref{eq:projection_matrix_Rn_to_Tm}.
Finally, the \textit{similarity matrix} 
\begin{equation}
S = S(W) \in \mc W
\end{equation}
is defined as a local geometric average of assignment vectors at neighboring nodes, i.e. the $i$-th 
row $S_i$ is defined to be the Riemannian mean 
(cf.~\cite[Def.~2]{Astroem2017})
\begin{equation}\label{eq:def-S-original}
  S_i = \mean_{\mc S} \{ L_j \}_{j\in\ol{\mc N}(i)}
\end{equation}
of the lifted distances $L_{j}$ in the neighborhood $\ol{\mc N}(i) = \mc N(i) \cup \{i\}$.

\vspace{0.2cm}
The correlation between $W$ and the local averages defining $S(W)$, as measured by the basic matrix inner product, is used as the objective function 
\begin{equation}
  \sup_{W \in\mc W} J(W),\qquad J(W) := \la W, S(W) \ra
\end{equation}
to be maximized. 
The optimization strategy is to follow the Riemannian gradient ascent flow on $\mc W$ (see Section \ref{sec:Geometric Integration} for the formal definition of the Riemannian gradient)
\begin{equation}
  \dot{W}(t) = \nabla_{\mc W} J(W(t)),\qquad W(0) = \frac{1}{n} \eins_m \eins_n^\top =: C.
\end{equation}
The initialization $W_i(0) = \frac{1}{n}\eins_n^\top$ with the barycenter of $\mc S$ constitutes an \textit{uninformative} uniform assignment 
which is not biased towards any prototype.

To obtain an efficient numerical algorithm, the Riemannian mean is approximated using the geometric mean 
\begin{equation}\label{eq:def-SW-original}
  S_i(W) = \frac{\mean_{g}\{ L_j\}_{j \in \ol{\mc N}(i)}}{\big\la \eins, \mean_{g}\{ L_j\}_{j \in \ol{\mc N}(i)}\big\ra}\;,\quad
  \mean_{g}\{ L_j\}_{j \in \ol{\mc N}(i)} = \Big(\prod_{j \in \ol{\mc N}(i)} L_j\Big)^{\frac{1}{|\ol{\mc N}(i)|}}\;.
\end{equation}
Based on the simplifying, plausible assumption that the mean only changes slowly and by using the explicit Euler-method directly on $\mc W$ with a certain adaptive
step-size (cf.~\cite[Sect.~3.3]{Astroem2017}), the following multiplicative update scheme is obtained
\begin{equation}\label{eq:W-update-original}
  W_i^{(k+1)} = \frac{W_i^{(k)} \cdot S_i(W^{(k)})}{\la W_i^{(k)}, S_i(W^{(k)})\ra}\;, \quad W_i^{(0)} = \frac{1}{n}\eins_n^\top, \quad i\in [m].
\end{equation}

\subsection{Geometric Integration of Gradient Flows}
\label{sec:Geometric Integration}
In this section we collect the basic ingredients needed in the remainder of this paper, of a general framework due to \cite{SavHuen17} for integrating a Riemannian gradient flow of an arbitrary function $J \colon \mc W \to \mb R$ defined on the assignment manifold.  

We first recall the definition of the Riemannian gradient. Let $M$ be a Riemannian manifold with an inner product $g^M_x$ on each tangent space $T_xM$ varying smoothly with $x \in M$ and $f\colon M \to \mb{R}$ a smooth function. 
Using the identification $T_r\mb{R} = \mb R$ for $r\in \mb R$, the Riemannian gradient $\nabla_Mf(x) \in T_x M$ of $f$ at $x\in M$ can be defined 
as the unique element of $T_xM$ satisfying
\begin{equation}\label{eq:def-Rgrad}
  g^M_x(\nabla_M f(x), v) = Df(x)[v], \quad \forall v \in T_xM,
\end{equation}
where $Df(x) : T_xM \rightarrow T_{f(x)}\mathbb{R} = \mathbb{R}$ is the differential of $f$.

\vspace{0.2cm}
Suppose $J \colon \mc W \to \mb R$ is a general smooth objective function modeling an assignment problem and we are interested in minimizing $J$ by following
the Riemannian gradient descent flow 
\begin{equation}\label{eq:minimizing_R_gradient_descent_flow_W}
  \dot{W}(t) = -\nabla_{\mc W} J(W(t))\;, \quad W(0) = C \in \mc W\;,
\end{equation}
with the barycenter $C = \frac{1}{n} \eins_m \eins_n^\top$. Instead of directly minimizing $J$ on $\mc W$, the basic idea of \cite{SavHuen17} is to pull
the optimization problem back onto the tangent space $T^m = T_C \mathcal W$ by setting 
\begin{equation}
\ol J := J \circ L_C,
\end{equation}
using the diffeomorphism $L_C \colon T^{m} \to \mc{W}$ given by \eqref{eq:def-L-W}.
Furthermore, the pullback of the Fisher-Rao metric under $L_C$ is used to equip $T^m$ with a Riemannian metric and to turn $L_C$ into
an isometry. In this setting, the Riemannian gradient of $\ol J \colon T^m \to \mb R$ at $V \in T^m$ is given by \cite[Sec.~3]{SavHuen17}
\begin{equation}
  \nabla_{T^m} \ol J(V) = \nabla J\big(L_C(V)\big) \in T^m\;,
\end{equation}
where $\nabla J$ denotes the standard Euclidean gradient of $J \colon \mc W \to \mb R$.
Based on this construction, solving the gradient flow \eqref{eq:minimizing_R_gradient_descent_flow_W} is equivalent to 
\begin{equation}
W(t) = L_C(V(t)),
\end{equation}
where $V(t) \in T^m$ solves
\begin{equation}
  \dot{V}(t) = -\nabla_{T^m} \ol J(V(t)) = -\nabla J \big( W(t) \big)\;,\quad V(0) = 0\;.
\end{equation}
Choosing the explicit Euler method for solving this gradient flow problem on the vector space $T^m$, results in the numerical update scheme for every row $i \in [m]$
\begin{equation}
  V_i^{(k+1)} = V_i^{(k)} - h \nabla J \big( L_C(V_i^{(k)})\big),\quad V_i^{(0)} = 0,
\end{equation}
with step-size $h \in \mb R$. Lifting this update scheme to the assignment manifold $\mc W$ yields a multiplicative update rule 
\begin{equation}\label{eq:geometric_minimizing_multiplicative_update_on_W}
  W_i^{(k+1)} = \frac{W_i^{(k)}\cdot e^{- h \nabla J( W_i^{(k)})}}{\big\la W_i^{(k)}, e^{- h \nabla J( W_i^{(k)})}\big\ra}\;, \quad W_i^{(0)} = \frac{1}{n} \eins_n, \quad i \in [m].
\end{equation}

\section{Energy, Gradients and Wasserstein Messages}
\label{sec:Gradients}

In this section we study the smooth objective function \eqref{eq:def-J-smooth} \textit{restricted} to the assignment manifold, in order to prepare the application of the approach of Section \ref{sec:Labeling-Assignment-Manifold} to graphical models in Section \ref{sec:Graphical-Models}. 

After detailing the rationale behind \eqref{eq:def-J-smooth} in Section \ref{sec:smooth-LP-approximation}, we compute the Euclidean gradient of the objective function in Section \ref{sec:nabla-Jtau} on which the Riemannian gradient will be based. This gradient involves the gradients of local Wasserstein distances that are considered in Section \ref{sec:Objective-Gradient}. From the viewpoint of belief propagation, these gradients can be considered as `Wasserstein messages', as discussed in Section \ref{sec:Graphical-Models}. 

\subsection{Smooth Approximation of the LP Relaxation}\label{sec:smooth-LP-approximation}
The starting point \eqref{eq:def-D-matrix} for applying the labeling approach of Section \ref{sec:labeling-by-assignment} to a given problem is a definition of suitable distances. Regarding problem \eqref{eq:def-E} and the corresponding model parameter vector $\theta$ defined by \eqref{eq:def-theta-V-theta-E}, this is straightforward to do for the \textit{unary} terms $\theta_{i}$ that typically measure a  local distance to observed data. But this is less obvious for the \textit{pairwise} terms $\theta_{ij}$ that do not have a direct counterpart in the geometric labeling approach.

The following Lemma explains why the 
local Wasserstein distances 
\begin{equation}\label{def-WD}
d_{\theta_{ij}}(\mu_{i},\mu_{j}) := \min_{\mu_{ij} \in \Pi(\mu_{i},\mu_{j})} \la \theta_{ij},\mu_{ij} \ra\;,
\end{equation}
defined for every edge $ij\in\mc E$ with $\Pi(\mu_{i},\mu_{j})$ due to \eqref{eq:def-Pi-mui-muj}, are natural candidates for taking into account pairwise model parameters $\theta_{ij}$. 

\begin{lemma}\label{lem:LP-reformulation}
 The local polytope relaxation \eqref{eq:LP-relaxation-b} 
is equivalent to the problem 
 \begin{equation}\label{eq:LP-reformulation}
   \min_{\mu_{\mathcal{V}} \in \Delta_{n}^{m} } \Big(\sum_{i \in \mathcal{V}}\langle\theta_{i},\mu_{i}\rangle + \sum_{ij \in \mathcal{E}} d_{\theta_{ij}}(\mu_{i},\mu_{j})\Big)
\end{equation}
involving the local Wasserstein distances \eqref{def-WD}.
\end{lemma}
\begin{proof}
The claim follows from reformulating the LP-relaxation based on the local polytope constraints \eqref{eq:def-mu-Pi} as follows.
    \begin{align*}
      \min_{\mu \in \mc{L}_{\mc{G}}} \langle \theta, \mu \rangle 
      &= \min_{\mu \in \mc{L}_{\mc{G}}}\la \theta_{\mc{V}}, \mu_{\mc{V}} \ra + \la \theta_{\mc{E}}, \mu_{\mc{E}} \ra \\
      &= \min_{\mu_{\mathcal{V}} \in \Delta_{n}^{m} }\Big(\langle\theta_{\mathcal{V}},\mu_{\mathcal{V}}\rangle + \min_{\mu_{\mathcal{E}}}\sum_{ij\in\mathcal{E}} \big (\langle\theta_{ij},\mu_{ij}\rangle + \delta_{\Pi(\mu_{i},\mu_{j})}(\mu_{ij}) \big ) \Big )\\
      &= \min_{\mu_{\mathcal{V}} \in \Delta_{n}^{m} } 
      \Big(
      \sum_{i \in \mathcal{V}} \langle\theta_{i},\mu_{i}\rangle
      + \sum_{ij \in \mathcal{E}} \min_{\mu_{ij} \in \Pi(\mu_{i},\mu_{j})} \langle\theta_{ij},\mu_{ij}\rangle \Big)
      \\
      &= \min_{\mu_{\mathcal{V}} \in \Delta_{n}^{m} } \Big(\sum_{i \in \mathcal{V}}\langle\theta_{i},\mu_{i}\rangle + \sum_{ij \in \mathcal{E}} d_{\theta_{ij}}(\mu_{i},\mu_{j})\Big).
    \end{align*}
\end{proof}
In order to conform to our smooth geometric setting, we regularize the convex but non-smooth (piecewise-linear (cf.~\cite[Def.~2.47]{Rockafellar2009})) local Wasserstein distances \eqref{def-WD} with a general convex 
\textit{smoothing function} $F_{\tau}$,
\begin{equation}\label{eq:def-WD-smoothed}
d_{\theta_{ij},\tau}(\mu_{i},\mu_{j})
= \min_{\mu_{ij} \in \Pi(\mu_{i},\mu_{j})} \big\{
\la\theta_{ij},\mu_{ij} \ra + F_{\tau}(\mu_{ij})\big\},\quad
ij \in \mc{E},\quad F_{\tau} \in \mc{F}_{0},\quad \tau>0,
\end{equation}
with smoothing parameter $\tau$. 
\begin{remark}[role of the smoothing]\label{rem:role-of-smoothing}
The influence of the smoothing parameter $\tau$ will be examined in detail in the remainder of this paper. We wish to point out from the beginning, however, that the ability of our smooth geometric approach to compute \textit{integral} labeling assignments does \textit{not} necessarily imply values of $\tau \approx 0$ close to zero, because the rounding mechanism to integral assignments is a \textit{different one}, as will be shown in Section \ref{sec:Graphical-Models}. As a consequence, larger feasible values of $\tau$ weaken the nonlinear relation \eqref{eq:def-WD-smoothed} and considerably speed up the convergence of numerical algorithm for iterative label assignment.
\end{remark}
\begin{remark}[local polytope constraints]\label{rem:local-polytope-constraints}
Using the regularized local Wasserstein distances \eqref{eq:def-WD-smoothed} implies by their definition that the local marginalization constraints \eqref{eq:def-mu-Pi} are \textit{always} satisfied. This is in sharp contrast to alternative labeling schemes, like loopy belief propagation, were these constraints are gradually enforced during the iteration and are guaranteed to hold only \textit{after} convergence of the entire iteration process. 

This elucidates two key properties that distinguish the manifold setting of our labeling approach from established work: 
\begin{enumerate}[(i)]
\item inherent smoothness and 
\item anytime validity of the local polytope constraints.
\end{enumerate}
\end{remark}
Based on Lemma \ref{lem:LP-reformulation} and the regularized local Wasserstein distances \eqref{eq:def-WD-smoothed}, we study in this paper the objective function \eqref{eq:def-J-smooth},
which is a \textit{smooth} approximation of the local polytope relaxation \eqref{eq:LP-relaxation-b} of the original labeling problem \eqref{eq:def-E}, with the local polytope constraints \eqref{eq:def-mu-Pi} \textit{built in}.

\vspace{0.2cm}
In order to get an intuition about suitable smoothing functions $F_\tau$, we inspect the smoothed local Wasserstein distance \eqref{eq:def-WD-smoothed} in more detail.
To this end, it will be convenient to simplify temporarily our notation in the remainder of this section by dropping indices as follows.
\begin{subequations}\label{eq:simplified_notation}
\begin{align}
\text{\textbf{Notation } for any edge $ij$}:\quad
M &= \mu_{ij} \in \R^{n \times n}, &
\Theta &= \theta_{ij} \in \R^{n \times n}, \\
\mu &= \bpm \mu_{1} \\ \mu_{2} \epm 
= \bpm M \eins_{n} \\ M^{\T} \eins_{n} \epm, &
\nu &= \bpm \nu_{1} \\ \nu_{2} \epm,
\end{align}
\end{subequations}
with the marginal vector $\mu$ playing the role of $\bsm \mu_{i} \\ \mu_{j} \esm$ in \eqref{eq:def-mu-Pi}. The local (non-smooth) Wasserstein distance \eqref{def-WD} then reads, for any edge $ij \in \mc{E}$,
\begin{equation}\label{eq:simplyNotationLocalWasserstein}
d_{\Theta}(\mu_{1},\mu_{2})
= \min_{M \in \Pi(\mu_{1},\mu_{2})} \la\Theta, M \ra \;.
\end{equation}
Using the linear map $\mathcal{A}$ defined by \eqref{eq:def-A}, we rewrite expression
 \eqref{eq:simplyNotationLocalWasserstein} as 
\begin{equation}\label{eq:W-distance}
d_{\Theta}(\mu_{1},\mu_{2}) =\min_{M \in \R^{n \times n}}\, 
\la\Theta,M\ra\quad\text{s.t.}\quad \mc{A} M = \bpm \mu_{1} \\ \mu_{2} \epm,\quad M \geq 0\;.
\end{equation}
The corresponding dual LP of \eqref{eq:W-distance} is given by
\begin{equation}\label{eq:dual-W-distance}
\max_{\nu \in \R^{2n}} \la\mu,\nu\ra \quad\text{s.t.}\quad
\mc{A}^{\T} \nu \leq \Theta\;.
\end{equation}
The \textit{smoothed} local Wasserstein distance \eqref{eq:def-WD-smoothed} is given by
\begin{equation}\label{eq:W-distance-smoothed}
\begin{aligned}
d_{\Theta, \tau}(\mu_{1},\mu_{2}) 
&:=\min_{M \in \R^{n \times n}}\, \la\Theta,M\ra + F_{\tau}(M)
\quad\text{s.t.}\quad \mc{A} M = \bpm \mu_{1} \\ \mu_{2} \epm,\quad M \geq 0,\\
&= \min_{M \in \R^{n \times n}}\, \la\Theta,M\ra + F_{\tau}(M)+ \delta_{\R_{+}^{n \times n}}(M) + \delta_{\{0\}}\big(\mc{A} M - \bsm \mu_{1} \\ \mu_{2} \esm\big),
\end{aligned}
\end{equation}
for $F_{\tau} \in \mc{F}_{0}$ and $\tau>0$, and the dual problem to \eqref{eq:W-distance-smoothed} reads 
\begin{equation}\label{eq:dual-smooth-local-W-distance}
\max_{\nu \in \R^{2n}} \la \mu,\nu \ra - G_{\tau}^{\ast}\big(\mc{A}^{\T} \nu - \Theta\big),
\end{equation}
with the conjugate function $G_{\tau}^{\ast}$ of 
\begin{equation}
G_{\tau}(M) = F_{\tau}(M) + \delta_{\R_{+}^{n \times n}}(M).
\end{equation}

Suitable candidates of functions $G_{\tau}$ for smoothing $d_{\Theta}$ suggest themselves by comparing the 
dual LP \eqref{eq:dual-W-distance} with the dual problem \eqref{eq:dual-smooth-local-W-distance} of the smoothed
 LP. Rewriting the constraints of \eqref{eq:dual-W-distance} in the form
\begin{equation}\label{eq:dual-constraints}
\delta_{\R_{-}^{n \times n}}(\mc{A}^{\T} \nu - \Theta)
\end{equation}
and comparing with \eqref{eq:dual-smooth-local-W-distance} shows that $G_{\tau}^{\ast}$ should be a smooth 
approximation of the indicator function $\delta_{\R_{-}^{n \times n}}$. We get back to this point in Section \ref{sec:Wasserstein-Numerics}.

\subsection{Energy Gradient $\nabla E_\tau$}\label{sec:nabla-Jtau}
The pairwise model parameters $\theta_{\mc{E}}$ may not be symmetric, $\theta_{ij} \neq \theta_{ij}^{\T},\; ij \in \mc{E}$, in general, which implies that the smoothed local Wasserstein distances are not symmetric either: \\ $d_{\theta_{ij}, \tau}(W_i, W_j) \neq d_{\theta_{ij}, \tau}(W_j, W_i)$. In order to compute the Euclidean gradient $\nabla E_{\tau}$ of the objective function \eqref{eq:def-J-smooth}, we therefore introduce an \textit{arbitrary fixed orientation} $(i,j)$ (ordered pair) of all edges $ij \in \mc{E}$, which means $ij \in \mc{E} \;\implies\; ji \not\in \mc{E}$.
As a consequence, \eqref{eq:def-J-smooth} reads
\begin{equation}\label{eq:J-smooth_rewritten}
  E_\tau(W) = \sum_{i\in V} \Big( \la \theta_i, W_i\ra + \sum_{j \colon (i, j) \in \mc E} d_{\theta_{ij}, \tau}(W_i, W_j)\Big)\;.
\end{equation}
The following proposition specifies the gradient $\nabla E_{\tau}$ in terms of an expression that involves local gradients of the smoothed Wasserstein distances $d_{\theta_{ij},\tau}$. These latter gradients are studied in Section \ref{sec:Objective-Gradient} (Theorem \ref{theo:gradientWasserstein}).
\begin{proposition}[objective function gradient]\label{prop:euclidean_gradient_smooth_energy_general}
  Suppose the edges $\mc E$ have an arbitrary fixed orientation. Then the Euclidean gradient of the objective function $E_\tau \colon \mathcal{W} \to \mb{R}$ due to \eqref{eq:def-J-smooth}, at $W \in \mc W$, is the matrix $\nabla E_\tau(W) \in T^m$ whose $i$-th row is given by 
\begin{equation}\label{eq:euclidean_gradient_smooth_energy_general}
    \nabla_i E_\tau(W) = P_{T}(\theta_i) + \sum_{j \colon (i,j) \in \mc E} \nabla_1 d_{\theta_{ij}, \tau}(W_i, W_j) + \sum_{j \colon (j,i) \in \mc E} \nabla_2 d_{\theta_{ji}, \tau}(W_j, W_i)\;,
  \end{equation}
where $\nabla_1 d_{\theta_{ij}, \tau}(W_i, W_j) \in T$ and $\nabla_2 d_{\theta_{ji}, \tau}(W_j, W_i) \in T$ are the Euclidean gradients of 
\begin{equation}
d_{\theta_{ij}, \tau}(\cdot, W_j)\colon \mc S 
\to \mb \R,  
\qquad
d_{\theta_{ij}, \tau}(W_j, \cdot)\colon \mc S 
\to \mb \R.
\end{equation}
\end{proposition}
\begin{proof}
Appendix \ref{sec:proof-euclidean_gradient_smooth_energy_general}.
\end{proof}

We now consider after a preparatory Lemma the specific case that all pairwise model parameters $\theta_{ij}=\theta_{ij}^{\T}$ are symmetric (Corollary \ref{cor:gradient-symmetric-case}). Recall definition \eqref{eq:def-Pi-mui-muj} of the set $\Pi(\cdot,\cdot)$ of coupling measures having its arguments as marginals and Remark \ref{rem:W-vs-mu} regarding notation.
\begin{lemma}\label{lem:symmetric_pairwise_terms}
  Suppose the convex smoothing function $F_{\tau}$ defining the regularized local Wasserstein distances \eqref{eq:def-WD-smoothed} satisfies $F_\tau(M) = F_\tau(M^\top)$ for all $M \in \Pi(W_i, W_j)$. Then 
\begin{equation}\label{eq:dtheta=dtheta-T}
d_{\theta_{ij}, \tau}(W_i, W_j) = d_{\theta_{ij}^\top, \tau}(W_j, W_i).
\end{equation}
\end{lemma}
\begin{proof}
Let $M_{\ast} \in \Pi(W_i, W_j)$ be a minimizer of \eqref{eq:W-distance-smoothed}. Then due to the assumption on $F_\tau$, we have 
  \begin{equation}
    d_{\theta_{ij}, \tau}(W_i, W_j) = \la \theta_{ij}, M_{\ast}\ra + F_\tau(M_{\ast}) = \la \theta_{ij}^\top, M_{\ast}^{\top} \ra + F_\tau(M_{\ast}^{\top})\;.
  \end{equation}
  Let $\tilde M \in \Pi(W_j, W_i)$ be arbitrary. Then ${\tilde M}^\top \in \Pi(W_i, W_j)$ and we have
  \begin{equation}
    \la \theta_{ij}^\top, \tilde M \ra + F_\tau(\tilde M) = \la \theta_{ij}, {\tilde M}^\top\ra + F_\tau({\tilde M}^\top) 
    \geq \la \theta_{ij}, M_{\ast}\ra + F_\tau(M_{\ast}) = \la \theta_{ij}^\top, M_{\ast}^\top \ra + F_\tau(M_{\ast}^\top)\;.
  \end{equation}
  This shows that $M_{\ast}^\top \in \Pi(W_j, W_i)$ is a minimizer of $d_{\theta_{ij}^\top, \tau}(W_j, W_i)$ and establishes equation \eqref{eq:dtheta=dtheta-T}.
\end{proof}
As a consequence of Lemma \ref{lem:symmetric_pairwise_terms}, if all pairwise model parameters $\theta_{ij}$ are symmetric, in addition to $F_\tau(M) = F_\tau(M^\top)$ for all $M \in [0, 1]^{n\times n}$, 
then there is no need to choose an edge orientation as was done in connection with \eqref{eq:J-smooth_rewritten}. Rather, using \eqref{eq:def-Ni}, we may rewrite \eqref{eq:J-smooth_rewritten} as
\begin{equation}\label{eq:J-smooth_rewritten_symmetric}
  E_\tau(W) = \sum_{i\in V} \Big( \la \theta_i, W_i\ra + \frac{1}{2}\sum_{j \in \mc{N}(i)} d_{\theta_{ij}, \tau}(W_i, W_j)\Big)
\end{equation}
and reformulate Proposition \ref{prop:euclidean_gradient_smooth_energy_general} accordingly.
\begin{corollary}[objective function gradient: symmetric case]\label{cor:gradient-symmetric-case}
  Suppose $F_\tau(T) = F_\tau(T^\top)$ for all $T \in [0, 1]^{n\times n}$ and $\theta_{ij}$ is symmetric for all $ij \in \mc E$. 
  Then the $i$-th row of the Euclidean gradient $\nabla E_\tau$ is given by
  \begin{equation}\label{eq:nabla-Ji-symmetric}
    \nabla_i E_\tau(W) = P_{T}(\theta_i) + \sum_{j \in \mc{N}(i)} \nabla_1 d_{\theta_{ij}, \tau}(W_i, W_j).
  \end{equation}
\end{corollary}
\begin{proof}
Applying the equation $\nabla_2 d_{\theta_{ji}, \tau}(W_j, W_i) = \nabla_1 d_{\theta_{ij}, \tau}(W_i, W_j)$ due to Lemma~\ref{lem:symmetric_pairwise_terms} to Eqn.~\eqref{eq:euclidean_gradient_smooth_energy_general}, we obtain
  \begin{subequations}
    \begin{align}
      \nabla_i E_\tau(W) &= P_{T}(\theta_i) + \sum_{j \colon (i,j) \in \mc E} \nabla_1 d_{\theta_{ij}, \tau}(W_i, W_j) + \sum_{j \colon (j,i) \in \mc E} \nabla_1 d_{\theta_{ij}, \tau}(W_i, W_j)\\
	&= P_{T}(\theta_i) + \sum_{j \in \mc{N}(i)} \nabla_1 d_{\theta_{ij}, \tau}(W_i, W_j),
    \end{align}
  \end{subequations}
which is \eqref{eq:nabla-Ji-symmetric}.
\end{proof}

\subsection{Local Wasserstein Distance Gradient}
\label{sec:Objective-Gradient}

In this section, we check differentiability of the distance functions $d_{\theta_{ij},\tau}(\mu_{i},\mu_{j}),\; ij \in \mc{E}$, given by \eqref{eq:def-WD-smoothed}, and 
specify an expression for the corresponding gradient.
To formulate the main result of this section, we again use the simplified notation \eqref{eq:simplified_notation}.
\begin{theorem}[{Wasserstein distance gradient}]\label{theo:gradientWasserstein}
  Consider $\mc{S} \subset \mathbb{R}^{n}$ as an Euclidean submanifold with tangent space $T$ defined by \eqref{eq:def-tangent-space-T}, and let
\begin{equation}\label{eq:def-g-dual}
  g(\mu,\nu) = \la\mu,\nu\ra - G_{\tau}^{\ast}(\mc{A}^{\T}\nu-\Theta)
\end{equation}
denote the dual objective function \eqref{eq:dual-dtau}. 
  Then the smoothed Wasserstein distance $d_{\Theta, \tau}\colon \mathcal S \times \mathcal S \to \mathbb R$ is differentiable, and the Euclidean gradient of $d_{\Theta, \tau}$ at $p = (p_{1}, p_{2}) \in \mathcal S \times \mathcal S$ is given by
  \begin{equation}\label{eq:nabla-dtau}
    \nabla d_{\Theta, \tau}(p) 
    = \nabla d_{\Theta, \tau}(p_1,p_2)
    = \ol{\nu}_{T} 
    := P_{T \times T}(\ol{\nu}) 
    = \bpm P_{T}(\ol{\nu}_{1}) \\ 
           P_{T}(\ol{\nu}_{2}) \epm,
\end{equation}
where 
\begin{equation}
\ol{\nu} = \bpm \ol{\nu}_{1} \\ \ol{\nu}_{2} \epm \in \underset{\nu\in\R^{2 n}}{\argmax} \; g(p,\nu).
\end{equation}
\end{theorem}
\noindent
The proof follows below after some preparatory Lemmas, that also clarify the structure of the dual solution set. In particular, this set restricted to $\mc{R}(\mc{A})$ is a singleton (Lemma 4.9).
\begin{lemma}\label{lem:dtau-primal-dual}
Let
\begin{equation}\label{eq:rewrite-smooth-local-W-distance-b}
G_{\tau}(M) = F_{\tau}(M) 
+ \delta_{\R_{+}^{n \times n}}(M)
\end{equation}
with the convex smoothing function $F_{\tau}$ of Eq.~\eqref{eq:def-WD-smoothed}, and assume the conjugate function $G_{\tau}^{\ast}$ is continuously differentiable. 
Then the dual problem of 
\begin{equation}\label{eq:primal-dtau}
\min_{M \in \Pi(\mu_{1},\mu_{2})} \big\{
\la\Theta,M \ra + F_{\tau}(M)\big\}
\end{equation}
is given by
\begin{equation}\label{eq:dual-dtau}
\max_{\nu_{1},\nu_{2}} \big\{
\la\mu,\nu\ra 
- G_{\tau}^{\ast}(\mc{A}^{\T}\nu - \Theta)\big\}.
\end{equation}
Furthermore, assuming that strong duality holds,  
the conditions for optimal primal $\ol{M}$ and dual $\ol{\nu} = (\ol{\nu}_{1},\ol{\nu}_{2})$ solutions are
\begin{subequations}\label{eq:dtau-opt-conditions}
\begin{gather}
\ol{M} = \nabla G_{\tau}^{\ast}\big(\mc{A}^{\T} \ol{\nu} - \Theta\big),
\qquad\qquad
\mc{A}^{\T} \ol{\nu} - \Theta 
\in \partial G_{\tau}(\ol{M}) 
\label{eq:dtau-opt-conditions-a}
\intertext{
together with the affine constraint
} \label{eq:dtau-opt-conditions-b}
\mc{A}\ol{M} = {\mu}.
\end{gather}
\end{subequations}
\end{lemma}
\begin{proof}
Taking into account \eqref{eq:def-Pi-mui-muj}, 
we write the right-hand side of \eqref{eq:W-distance-smoothed} in the form
\begin{equation}\label{eq:primal-dtau-rewrite}
\min_{M \in \R^{n \times n}} \la\Theta,M\ra + G_{\tau}(M)
\quad\text{s.t.}\quad
\mc{A} M = \mu, \quad M \ge 0.
\end{equation}
Let $\nu=(\nu_{1}, \nu_{2}) \in \R^{2 n}$ 
denote the dual variables corresponding to the affine constraint of \eqref{eq:primal-dtau-rewrite}. 
Then problem \eqref{eq:primal-dtau-rewrite} rewritten in Lagrangian form reads
\begin{subequations}
\begin{align}\label{eq:primal-dtau-rewrite-b}
&\min_{M \in \R^{n \times n}}\big\{
\la\Theta, M\ra + G_{\tau}(M)
+ \max_{\nu}\la \nu, \mu - \mc{A}M\ra\big\} \\
\gdw\qquad 
&\min_{M \in \R^{n \times n}}\big\{
\max_{\nu}\la\nu,\mu\ra + G_{\tau}(M)
-\big\la\mc{A}^{\T}\nu-\Theta, M\big\ra\big\}.
\end{align}
\end{subequations}
Since strong duality holds by assumption, interchanging $\min$ and $\max$ yields the dual problem \eqref{eq:dual-dtau}. Moreover, the optimal primal and dual objective function values are equal, which gives with \eqref{eq:primal-dtau-rewrite-b} and \eqref{eq:dual-dtau}
\begin{equation}
-\la \ol{M}, \mc{A}^{\T}\ol{\nu} -\Theta\ra + G_{\tau}(\ol{M}) 
+ G_{\tau}^{\ast}(\mc{A}^{\T}\ol{\nu}-\Theta) = 0.
\end{equation}
This implies \eqref{eq:dtau-opt-conditions-a} by the subgradient inversion rule \cite[Prop.~11.3]{Rockafellar2009}, whereas the primal constraint \eqref{eq:dtau-opt-conditions-b} is obvious.
\end{proof}
\begin{remark}[{smoothness of $G_{\tau}^{\ast}$}]\label{rem:smoothness-G*}
The \textit{smoothness} assumption with respect to $G_{\tau}^{\ast}$ enables to compute conveniently the gradient of the smoothed Wasserstein distance $d_{\Theta,\tau}$. It corresponds to a \textit{convexity} assumption on $G_{\tau}$. These aspects are further discussed in Section \ref{sec:Wasserstein-Numerics} as well.
\end{remark}
\begin{remark}[{strong duality}]\label{rem:strong-duality}
The condition of strong duality (cf.~\cite[Section I.5]{Boyd:2009aa}) made by Lemma \ref{lem:dtau-primal-dual} is crucial for what follows. This condition will be satisfied later on when working in a \textit{geometric} setting with local measures $M, \mu_{1}, \mu_{2}$ with \textit{full} support, as introduced in Section \ref{sec:Assignment-Manifold}.
\end{remark}

\begin{lemma}\label{lem:kernel_A_transposed}
Let the linear mapping $\mc{A}^{\T}$ be defined by \eqref{eq:def-AT}. Then
\begin{equation}\label{eq:N-AT}
\mc{N}(\mc{A}^{\T}) = \left\{\lambda \bpm \eins_{n} \\ -\eins_{n} \epm \in \R^{2 n} \colon \lambda \in \R \right\} \quad \text{and}\quad 
\mc{N}(\mc{A}^{\T})^\perp = \left\{ x \in \R^{2 n} \colon \Big\la x, \bpm \eins_{n} \\ -\eins_{n} \epm\Big\ra = 0\right\}.
\end{equation}
\end{lemma}
\begin{proof}
Let $z = \bsm x \\ y \esm \in \mathbb{R}^{2 n}$ with $0 = \mathcal{A}^\top z = x \eins_{n}^{\T} + \eins_{n} y^{\T}$. Applying $\mathcal{A}$, we get
\begin{equation}
0 = \mathcal{A}\mathcal{A}^\top z 
= \mathcal{A}(x \eins^\top) + \mc{A}(\eins y^\top) 
= \bpm n x + \la y, \eins_{n} \ra \eins_{n} \\
\la x, \eins_{n} \ra \eins_{n} + n y \epm
\quad\gdw\quad
z = \bpm x \\ y \epm 
= -\frac{1}{n} \bpm \la y, \eins_{n} \ra \eins_{n} \\
\la x, \eins_{n} \ra \eins_{n} \epm.
\end{equation}
This implies $\la x, \eins_{n} \ra = -\la y, \eins_{n} \ra$, and setting $\lambda = \frac{1}{n} \la x, \eins_{n} \ra \in \R$ shows that $z$ has the form \eqref{eq:N-AT}. 
Conversely, in view of the definition \eqref{eq:def-AT}, it is clear that any vector from the set \eqref{eq:N-AT} is in $\mc{N}(\mc{A}^{\T})$. The characterization of
$\mc N(\mc{A}^\T)^\perp$ directly follows from the definitions.
\end{proof}

The following Lemma characterizes the set of optimal dual solutions to problem \eqref{eq:dual-dtau}.
\begin{lemma}\label{lem:rewritten_argmax_of_g}
  Let the function $G_{\tau}^{\ast}$ of the dual objective function \eqref{eq:dual-dtau} resp.~\eqref{eq:def-g-dual} be continuously differentiable and strictly convex, and let $p \in \R_{++}^{2 n}$. Then the set of optimal dual solutions has the form
\begin{equation}\label{eq:argmax-g}
\underset{\nu \in \R^{2 n}}{\argmax}\;g(p,\nu)
= \begin{cases}
\{\ol{\nu}\}, &\text{if}\; \left\la p, \bsm \eins_{n} \\ -\eins_{n}\esm \right\ra \neq 0, \\
\ol{\nu} + \mc{N}(\mc{A}^{\T}), &\text{if}\; \left\la p, \bsm \eins_{n} \\ -\eins_{n}\esm \right\ra = 0.
\end{cases}
\end{equation}
\end{lemma}
\begin{proof}
Appendix \ref{sec:proof-lem-rewritten_argmax_of_g}.
\end{proof}
We next clarify the \textit{attainment} of optimal dual solutions due to Lemma \ref{lem:rewritten_argmax_of_g}.
\begin{lemma}\label{lem:rewritten_argmax_of_g_on_U}
Consider the orthogonal decomposition $\R^{2 n} = \mc{N}(\mc{A}^{\T}) \oplus \mc{R}(\mc{A})$ into linear subspaces and denote the 
corresponding components of a vector $\nu \in \R^{2 n}$ by $\nu = \nu_{\mc N} + \nu_{\mc R}$. Then, for $p \in \R_{++}^{2 n}$ satisfying 
$\la p, \bsm\eins_{n} \\ -\eins_{n} \esm \ra = 0$, we have
\begin{subequations}
\begin{align}
\underset{\nu_{\mc R} \in \mc{R}(\mc{A})}{\argmax}\; g(p,\nu_{\mc R}) &= \{\ol{\nu}_{\mc R}\}, \qquad\qquad
\ol{\nu}_{\mc R} = P_{\mc{R}(\mc{A})}(\ol{\nu})\quad\text{for any}\quad \ol{\nu} \in \underset{\nu \in \R^{2 n}}{\argmax}\; g(p,\nu), 
\label{eq:attainment-dual-nu-a} \\ 
\label{eq:attainment-dual-nu-b}
g(p,\ol{\nu}_{\mc R}) &= \max_{\nu_{\mc R} \in \mc{R}(\mc{A})} g(p,\nu_{\mc R}) = \max_{\nu \in \R^{2 n}} g(p,\nu),
\end{align}
\end{subequations}
that is a unique dual maximizer exists in the subspace $\mc{R}(\mc{A})$.
\end{lemma}
\begin{proof}
We first shown \eqref{eq:attainment-dual-nu-b}. Let $\ol{\nu}$ be an optimal dual solution. Since  $\left\la p, \bsm\eins_{n} \\ -\eins_{n} \esm \right\ra = 0$, Lemma \ref{lem:rewritten_argmax_of_g} yields $\argmax_{\nu\in \R^{2 n}} g(p,\nu) = \ol{\nu} + \mc{N}(\mc{A}^{\T}) = \ol{\nu}_{\mc N} + \ol{\nu}_{\mc R} + \mc{N}(\mc{A}^{\T})$. This shows $\ol{\nu}_{\mc R} \in \ol{\nu} + \mc{N}(\mc{A}^{\T})$, that is $\ol{\nu}_{\mc R} \in \mc{R}(\mc{A})$ is a maximizer, which implies \eqref{eq:attainment-dual-nu-b}.  

Let $\ol{\nu}_{\mc R}' \in \mc{R}(\mc{A})$ be another maximizer. As before, we have the representation $\ol{\nu}_{\mc R}' \in \ol{\nu}+\mc{N}(\mc{A}^{\T})$, that is $\ol{\nu}_{\mc R}' = \ol{\nu}_{\mc N}+\ol{\nu}_{\mc R} + \tilde{\nu}_{\mc N}$ for some $\tilde{\nu}_{\mc N} \in \mc{N}(\mc{A}^{\T})$, which implies $\ol{\nu}_{\mc R}'=\ol{\nu}_{\mc R}$, i.e.~uniqueness \eqref{eq:attainment-dual-nu-a} of the dual maximizer in $\mc{R}(\mc{A})$.
\end{proof}

\vspace{0.25cm}
We are now in a position to prove Theorem \ref{theo:gradientWasserstein}. 
\begin{proof}[Proof of Theorem \ref{theo:gradientWasserstein}]
We proceed by subsequently proving the following: First, we relate the orthogonal decomposition $\R^{2 n} = \mc{N}(\mc{A}^{\T}) \oplus \mc{R}(\mc{A})$ to the 
tangent space $T_p (\mc S \times \mc S) = T \times T \subset \R^{2 n}$ for any $p = (p_1, p_2) \in \mc S \times \mc S$. Second, the existence of a global 
isometric chart for the manifold $\mc{S} \times \mc{S}$ is shown in order to represent the smoothed Wasserstein distance $d_{\Theta,\tau}$ and the dual 
objective function $g(\mu,\nu)$ in a convenient way. Third, we apply Theorem \ref{thm:Danskin}.
\begin{enumerate}
\item
Consider the unique decomposition $\nu = \nu_{\mc N} + \nu_{\mc R} \in \mc{N}(\mc{A}^{\T}) \oplus \mc{R}(\mc{A})$ of any point $\nu \in \R^{2 n}$. 
Then we have
\begin{equation}\label{eq:Pi-TxT}
P_{T \times T}(\nu_{\mc R}) =  \nu_{T} = P_{T \times T}(\nu).
\end{equation} 
At first, we show $T \times T \subseteq \mc{R}(\mc{A})$. For this, take an arbitrary $v = \bsm v_{1} \\ v_{2} \esm \in T \times T$. Due to the 
definition of $T$, we have $\la \eins_{n}, v_{1} \ra = \la \eins_{n}, v_{2} \ra = 0$ and thus $\la v, \bsm \eins_{n} \\ -\eins_{n}\esm \ra = 0$, 
which according to Lemma \ref{lem:kernel_A_transposed} means $v \in \mc{N}(\mc{A}^\top)^{\perp} = \mc{R}(\mc{A})$. As a consequence of 
$T \times T \subseteq \mc{R}(\mc{A})$ we have $P_{T \times T} (\nu_{\mc N}) = 0$ and therefore Statement \eqref{eq:Pi-TxT} follows from 
\begin{equation}
  P_{T \times T}(\nu) - P_{T \times T}(\nu_{\mc R}) = P_{T \times T}(\nu - \nu_{\mc R}) = P_{T \times T} (\nu_{\mc N}) = 0.
\end{equation}
\item
There exists an open subset $U \subset \R^{2(n-1)}$ and an isometry $\phi \colon U \to \mc{S} \times \mc{S}$ such that $\phi^{-1}$ is a global isometric 
chart of the manifold $\mc{S} \times \mc{S}$. $\phi$ can be constructed as follows. Choose an orthonormal basis $\{v_1, \ldots, v_{2(n-1)}\}$ of the tangent 
space $T\times T$, set $b = \frac{1}{n} \bsm \eins_{n} \\ \eins_{n} \esm$ and define the isometry
\begin{equation}
\psi \colon \R^{2(n-1)} \to \big(T\times T\big) + b,\quad
x \mapsto \psi(x) := B x + b,\quad
B x = \sum_{i=1}^{2(n-1)} x_{i} v_{i}.
\end{equation}
Because $\mc{S}\times\mc{S}$ is an open subset of $\big(T\times T\big) + b$ and $\psi$ an isometry, we have that the set 
$U := \psi^{-1}(\mc{S}\times\mc{S}) \subset \R^{2(n-1)}$ is also open and 
\begin{equation}\label{eq:def-phi-SxS}
  \phi := \psi|_U\colon U \to \mc S \times \mc S
\end{equation}
the desired isometric mapping. Furthermore, since the basis $\{v_{i}\}_{i=1}^{2(n-1)}$ is orthonormal, the orthogonal projection reads 
\begin{equation}\label{eq:Pi-BBT}
P_{T \times T} = B B^{\T}.
\end{equation}

\item
Using $\phi$ given by \eqref{eq:def-phi-SxS}, we obtain the coordinate representations
\begin{equation}
\ol{d}_{\Theta, \tau} := d_{\Theta, \tau} \circ \phi,\qquad
\ol{g}(x, \nu) := g\big(\phi(x), \nu\big)
\end{equation}
of the smoothed Wasserstein distance $d_{\Theta,\tau}$ and the dual objective function $g(p,\nu)$. Since we assume strong duality, that is equality of the optimal values of \eqref{eq:primal-dtau} and \eqref{eq:dual-dtau}, we have $d_{\Theta,\tau}(p)=\max_{\nu \in \R^{2 n}} g(p,\nu)$. Setting $x_{p}=\phi^{-1}(p)$, this equation translates in view of Lemma \ref{lem:rewritten_argmax_of_g_on_U} to
\begin{equation}\label{eq:g-bar-g-relations}
\ol{g}(x_{p},\ol{\nu}_{\mc R})
= \max_{\nu_{\mc R} \in \mc{R}(\mc{A})} \ol{g}(x_{p},\nu_{\mc R})
= \ol{g}(x_{p},\ol{\nu}) 
= \max_{\nu \in \R^{2 n}} \ol{g}(x_{p},\nu)
= \ol{d}_{\Theta,\tau}(x_{p}),
\end{equation}
with unique maximizer $\ol{\nu}_{\mc R} = P_{\mc{R}(\mc{A})}(\ol{\nu})$. Let $\B_{\delta} \subset \mc{R}(\mc{A})$ be a compact neighborhood of $\ol{\nu}_{\mc R}$. Then \eqref{eq:g-bar-g-relations} remains valid after restricting $\mc{R}(\mc{A})$ to $\B_{\delta}$. Because $g$ given by \eqref{eq:def-g-dual} is linear in the first argument and the mapping $\phi$ is affine, the function $\ol{g}$ is convex in the first argument and differentiable, hence satisfies the assumptions of Theorem \ref{thm:Danskin}.

In order to compute the gradient $\nabla_{x}\ol{g}(x,\nu_{\mc R})$, it suffices to consider the first term $\la \phi(x), \nu_{\mc R} \ra$ of $\ol{g}$, which only depends on $x$. Using \eqref{eq:def-phi-SxS}, we have
\begin{equation}
\la\phi(x),\nu_{\mc R}\ra
= \la B x + b, \nu_{\mc R} \ra
= \la x, B^{\T} \nu_{\mc R} \ra + \la b, \nu_{\mc R} \ra.
\end{equation}
Thus, $\nabla_{x} \ol{g}(x,\nu_{\mc R}) = B^{\T} \nu_{\mc R}$ which continuously depends on $\nu_{\mc R}$. As a consequence, we may apply Theorem \ref{thm:Danskin} and obtain due to \eqref{eq:Danskin}
\begin{equation}
\nabla\ol{d}_{\Theta,\tau}(x_{p}) 
= \nabla_{x}\ol{g}(x_{p},\ol{\nu}_{\mc R})
= B^{\T}\ol{\nu}_{\mc R}.
\end{equation}
Using the differential $D\phi(x)=B$, we finally get
\begin{equation}
\nabla d_{\Theta,\tau}(p) = B \nabla\ol{d}_{\Theta,\tau}(x_{p}) = B B^{\T}\ol{\nu}_{\mc R}
\overset{\eqref{eq:Pi-TxT}}{=}
P_{T \times T}(\ol{\nu}_{\mc R})
\overset{\eqref{eq:Pi-TxT}}{=}
\ol{\nu}_{T},
\end{equation}
which proves \eqref{eq:nabla-dtau}.
\end{enumerate}
\end{proof}

%% file: TexInput/Graphical-Models.tex
This section explains how the labeling approach on the assignment manifold of Section \ref{sec:Labeling-Assignment-Manifold} can be applied to a graphical model, using the global and local gradients derived in Section \ref{sec:Gradients}. The graphical model is given in terms of an energy function $E(x)$ of the form \eqref{eq:def-E}. The basic idea, worked out in Section \ref{sec:min-Jtau-geometric}, for determining a labeling $x$ with low energy $E(x)$ is to combine minimization of the convex relaxation \eqref{eq:LP-relaxation} and non-convex rounding to an integral solution in a \textit{single smooth process}. 
This idea is realized by restricting the smooth approximation \eqref{eq:def-J-smooth} of the objective function to the assignment manifold from Section \ref{sec:Assignment-Manifold}, and by combining numerical integration of the corresponding Riemannian gradient flow from Section \ref{sec:Geometric Integration} with the assignment mechanism suggested by \cite{Astroem2017} from Section \ref{sec:labeling-by-assignment}. 

Section \ref{sec:Wasserstein-Messages} complements our preliminary observations stated as Remarks \ref{rem:role-of-smoothing} and \ref{rem:local-polytope-constraints}, in order to highlight the essential properties of this process as a novel way of `belief propagation' using dually computed gradients of local Wasserstein distances, that we call \textit{Wasserstein messages}.

\subsection{Smooth Integration of Minimizing and Rounding on the Assignment Manifold}
\label{sec:min-Jtau-geometric}

We recall how regularization is performed by the assignment approach of \cite{Astroem2017}: distance vectors \eqref{eq:def-D-matrix} representing the data term of classical variational approaches are lifted to the assignment manifold by \eqref{eq:-def-L-matrix} and geometrically averaged over  spatial neighborhoods -- see Eqns.~\eqref{eq:def-S-original} and \eqref{eq:def-SW-original}.

Given a graphical model in terms of an energy function \eqref{eq:def-E}, regularization is already \textit{defined} by the pairwise model parameters $E_{ij}(\ell_{k},\ell_{r})$ resp.~$\theta_{ij}(\ell_{k},\ell_{r})$, so that evaluating the gradient of the regularized objective function \eqref{eq:def-J-smooth} \textit{implies} averaging over spatial neighborhoods, as Eq.~\eqref{eq:euclidean_gradient_smooth_energy_general} clearly displays. Taking additionally into account the simplest (explicit Euler) update rule \eqref{eq:geometric_minimizing_multiplicative_update_on_W} for geometric integration of Riemannian gradient flows on the assignment manifold, a natural definition of the similarity matrix that consistently incorporates the graphical model into the geometric approach of \cite{Astroem2017}, is 
\begin{equation}\label{eq:euler_step_in_numerics_tangent}
  S_{i}(W^{(k)}) = \frac{W_i^{(k)} \cdot e^{-h \nabla_i E_\tau(W^{(k)})}}{\la W_i^{(k)}, e^{-h \nabla_i E_\tau(W^{(k)})}\ra},\quad i \in [m],\qquad h > 0,\qquad W^{(0)} = \frac{1}{n} \eins_{m} \eins_{n}^{\T},
\end{equation}
where $h$ is a stepsize parameter and the partial gradients $\nabla_i E_\tau(W^{(k)})$ are given by \eqref{eq:euclidean_gradient_smooth_energy_general}. The sequence $(W^{(k)})$ is initialized in an unbiased way at the  barycenter $W^{(0)} \in \mc{W}$. Adopting the fixed point iteration proposed by \cite{Astroem2017} leads to the update of the assignment matrix
\begin{equation}\label{eq:W-update}
  W_i^{(k+1)} = \frac{W_i^{(k)} \cdot S_i(W^{(k)})}{\la W_i^{(k)}, S_i(W^{(k)})\ra},\quad i \in [m].
\end{equation}
These two interleaved update steps represent two objectives: (i) minimize the function $E_{\tau}$ on the assignment manifold $\mc{W}$ (Section \ref{sec:Geometric Integration}) and (ii) converge to an integral solution, i.e.~a valid labeling. Plugging \eqref{eq:euler_step_in_numerics_tangent} into \eqref{eq:W-update} gives 
\begin{equation}
  W_i^{(k+1)} = \frac{(W_i^{(k)})^2 \cdot e^{-h \nabla_i E_\tau(W^{(k)})}}{\la (W_i^{(k)})^2, e^{-h \nabla_i E_\tau(W^{(k)})}\ra},
\end{equation}
which suggests to control more flexibly the latter rounding mechanism by a \textit{rounding parameter} $\alpha$ and the update rule
\begin{equation}\label{eq:W-update-alpha}
  W_i^{(k+1)} = \frac{(W_i^{(k)})^{1+\alpha} \cdot e^{-h \nabla_i E_\tau(W^{(k)})}}{\la (W_i^{(k)})^{1+\alpha}, e^{-h \nabla_i E_\tau(W^{(k)})}\ra},\quad \alpha \geq 0.
\end{equation}
The following proposition reveals the \textit{continuous} gradient flow that is approximated by the sequence \eqref{eq:W-update-alpha}.
\begin{proposition}\label{prop:W-update-alpha}
Let $E_{\tau}$ be given by \eqref{eq:def-J-smooth} and denote the entropy of the assignment matrix $W$ by 
\begin{equation}\label{eq-entropy-formula_W}
H(W) = -\la W, \log W \ra.
\end{equation}
Then the sequence of updates \eqref{eq:W-update-alpha} are geometric Euler-steps for numerically integrating the Riemannian gradient flow of the extended objective function
\begin{equation}\label{eq:def-f-tau-alpha}
f_{\tau,\alpha}(W) := E_{\tau}(W) + \alpha_{h} H(W),\qquad \alpha_{h} = \frac{\alpha}{h}.
\end{equation}
\end{proposition}
\begin{proof}
An Euler-step for minimizing $f_{\tau,\alpha}$ on the tangent space reads (with $\nabla_{i}=\nabla_{W_{i}}$)
\begin{equation}\label{eq:proof-W-update-alpha-Vk1}
  V_i^{(k+1)} = V_i^{(k)} - h \nabla_i f(W^{(k)}) = V_i^{(k)} - h \nabla_i E_\tau(W^{(k)}) - \alpha \nabla_{i} H(W^{(k)}),\qquad
  i \in [m],
\end{equation}
where the $i$-th row of $W^{(k)}$ is given by $W_i^{(k)} = L_c( V_i^{(k)}),\; c = \frac{1}{n} \eins_{n}$. In order to compute the gradient of the entropy, consider a smooth curve $\gamma \colon (-\varepsilon, \varepsilon) \to \mc W$ with $\gamma(0) = W$ and $\dot{\gamma}(0) = X$. Then 
  \begin{equation}
    \frac{d}{dt} H(\gamma(t))\big|_{t=0} = - \la X, \log(W) \ra - \la W, \frac{1}{W} \cdot X\ra = - \la X, \log(W) \ra - \la \eins \eins^\top, X\ra.
  \end{equation}
  Since $\la \log(W), X\ra = \la P_{T^{m}} \left (\log(W)\right ), X\ra$ and $\la \eins \eins^\top, X\ra = \la \eins, X\eins \ra = \la \eins, 0 \ra = 0$, we have
  \begin{equation}
    \la \nabla H(W), X\ra = \frac{d}{dt} H(\gamma(t))\big|_{t=0} = \la - P_{T^{m}}\left ( \log(W) \right ), X\ra.
  \end{equation}
Thus, using $P_T(\log(W_i)) = L_c^{-1}(W_i)$ from \eqref{eq:def-L-inverse}, we obtain 
\begin{equation}
\nabla_{i} H(W^{(k)}) = -P_{T} \left (\log(W_{i}^{(k)})\right ) = -L_c^{-1}\big(L_{c}(V_{i}^{(k)})\big) = -V_{i}^{(k)}.
\end{equation}
Substitution into \eqref{eq:proof-W-update-alpha-Vk1} gives
\begin{equation}
V_i^{(k+1)} = (1+\alpha) V_i^{(k)} - h \nabla_i E_\tau(W^{(k)})
\end{equation}
and in turn the update
\begin{subequations}
\begin{align}
W_{i}^{(k+1)} &= L_{c}(V_{i}^{(k+1)})
= \frac{e^{(1+\alpha) V_{i}^{(k)}} \cdot e^{-h\nabla_{i} E_{\tau}(W^{(k)})}}{\la \eins_{n}, e^{(1+\alpha) V_{i}^{(k)}} \cdot e^{-h\nabla_{i} E_{\tau}(W^{(k)})} \ra} \\
&= \frac{(e^{V_{i}^{(k)}})^{(1+\alpha)} \cdot e^{-h\nabla_{i} E_{\tau}(W^{(k)})}}{\la \eins_{n}, (e^{V_{i}^{(k)}})^{1+\alpha} \cdot e^{-h\nabla_{i} E_{\tau}(W^{(k)})} \ra} 
= \frac{(W_{i}^{(k)})^{(1+\alpha)} \cdot e^{-h\nabla_{i} E_{\tau}(W^{(k)})}}{\la \eins_{n}, (W_{i}^{(k)})^{1+\alpha} \cdot e^{-h\nabla_{i} E_{\tau}(W^{(k)})} \ra} \\
&= \frac{(W_{i}^{(k)})^{(1+\alpha)} \cdot e^{-h\nabla_{i} E_{\tau}(W^{(k)})}}{\la (W_{i}^{(k)})^{1+\alpha}, e^{-h\nabla_{i} E_{\tau}(W^{(k)})} \ra}
\end{align}
\end{subequations}
which is \eqref{eq:W-update-alpha}.
\end{proof}
\begin{remark}[continuous DC programming]\label{rem:dc-perspective}
Proposition \ref{prop:W-update-alpha} and \eqref{eq:def-f-tau-alpha} admit to interpret the update rule \eqref{eq:W-update-alpha} as a \textit{continuous difference of convex (DC) programming} strategy. Unlike the established DC approach \cite{PhamDinh1997,PhamDinh-LeThin-98}, however, which takes \textit{large steps} by solving to optimality a sequence of convex programs in connection with updating an affine upper bound of the concave part of the objective function, our update rule \eqref{eq:W-update-alpha} differs in two essential ways: \textit{geometric optimization} by numerically integrating the Riemannian gradient flow \textit{tightly interleaves with rounding} to an integral solution. The rounding effect is achieved by minimizing the entropy term of \eqref{eq:def-f-tau-alpha} which steadily sparsifies the assignment vectors comprising $W$.
\end{remark}
%

\subsection{Wasserstein Messages}
\label{sec:Wasserstein-Messages}

We get back to the informal discussion of  \textit{belief propagation} in Section \ref{sec:Related-Work} in order to highlight properties of our approach \eqref{eq:def-J-smooth} from this viewpoint. We first sketch belief propagation and the origin of corresponding \textit{messages}, and refer to \cite{Yedidia-GenBP-05,WainrightJordan08} for background and more details.

Starting point is the primal linear program (LP) \eqref{eq:LP-relaxation} written in the form
\begin{equation}
\min_{\mu \in \mc{L}_{\mc{G}}} \la \theta, \mu \ra 
= \min_{\mu} \la \theta, \mu \ra \quad\text{subject to}\quad
A \mu = b,\; \mu \geq 0,
\end{equation}
where the constraints represent the feasible set $\mc{L}_{\mc{G}}$ which is explicitly given by the local marginalization constraints \eqref{eq:def-mu-Pi}. The corresponding dual LP reads
\begin{equation}
\max_{\nu}\la b, \nu \ra = \max_{\nu}\la\eins, \nu_{\mc{V}} \ra,\quad A^{\T} \nu \leq \theta,
\end{equation}
with dual (multiplier) variables
\begin{equation}
\nu = (\nu_{\mc{V}},\nu_{\mc{E}}) = (\dotsc,\nu_{i},\dotsc,\nu_{ij}(x_{i}),\dotsc,\nu_{ij}(x_{j}),\dotsc),\quad
i \in \mc{V},\quad ij \in \mc{E}
\end{equation}
corresponding to the affine primal constraints. In order to obtain a condition that relates optimal vectors $\mu$ and $\nu$ without subdifferentials that are caused by the non-smoothness of these LPs, one considers the \textit{smoothed} primal convex problem
\begin{equation}\label{eq:smoothed-LP}
\min_{\mu \in \mc{L}_{\mc{G}}} \la \theta, \mu \ra - \veps H(\mu),\quad\veps>0,\qquad
H(\mu) = \sum_{ij \in \mc{E}} H(\mu_{ij}) - \sum_{i \in \mc{V}}\big(d(i)-1\big) H(\mu_{i})
\end{equation}
with smoothing parameter $\veps>0$, degree $d(i)$ of vertex $i$, and with the local entropy functions
\begin{equation}
H(\mu_{i})=-\sum_{x_{i} \in \mc{X}} \mu_{i}(x_{i})\log\mu_{i}(x_{i}),\qquad
H(\mu_{ij})=-\sum_{x_{i},x_{j} \in \mc{X}}\mu_{ij}(x_{i},x_{j})\log\mu_{ij}(x_{i},x_{j}).
\end{equation}
Setting temporarily $\veps=1$ and evaluating the optimality condition $\nabla_{\mu}L(\mu,\nu)=0$ based on the corresponding Lagrangian
\begin{equation}\label{eq:marginalization-Lagrangian}
L(\mu,\nu) = \la \theta, \mu \ra - H(\mu) + \la \nu, A \nu-b \ra,
\end{equation}
yields the relations connecting $\mu$ and $\nu$,
\begin{subequations}\label{eq:local-LP-opt-conditions}
\begin{align}
\mu_{i}(x_{i}) &= e^{\nu_{i}} e^{-\theta_{i}(x_{i})} \prod_{j \in \mc{N}(i)}e^{\nu_{ij}(x_{i})},\qquad x_{i} \in \mc{X},\; i \in \mc{V}, \\
\mu_{ij}(x_{i},x_{j}) &= e^{\nu_{i}+\nu_{j}} e^{-\theta_{ij}(x_{i},x_{j})-\theta_{i}(x_{i})-\theta_{j}(x_{j})} \prod_{k \in \mc{N}(i)\setminus\{j\}} e^{\nu_{ik}(x_{i})} \prod_{k \in \mc{N}(j)\setminus\{i\}} e^{\nu_{jk}(x_{j})},
\end{align}
\end{subequations}
$x_{i},x_{j} \in \mc{X},\; ij \in \mc{E}$, where the terms $e^{\nu_{i}}, e^{\nu_{i}+\nu_{j}}$ normalize the expressions on the right-hand side whereas the so-called \textit{messages} $e^{\nu_{ij}(x_{i})}$ enforce the local marginalization constraints $\mu_{ij} \in \Pi(\mu_{i},\mu_{j})$. Invoking these latter constraints enables to eliminate the left-hand side of \eqref{eq:local-LP-opt-conditions} to obtain after some algebra the fixed point equations
\begin{equation}\label{eq:BP}
e^{\nu_{ij}(x_{i})} = e^{\nu_{j}}\sum_{x_{j} \in \mc{X}} \Big(
e^{-\theta_{ij}(x_{i},x_{j})-\theta_{j}(x_{j})}
\prod_{k \in \mc{N}(j)\setminus\{i\}} e^{\nu_{ik}(x_{j})} \Big),\qquad ij \in \mc{E},\quad x_{i} \in \mc{X},
\end{equation}
solely in terms of the \textit{dual} variables, commonly called \textit{sum-product algorithm} or \textit{loopy belief propagation} by \textit{message passing}. Repeating this derivation, after weighting the entropy function $H(\mu)$ of \eqref{eq:marginalization-Lagrangian} by $\veps$ as in \eqref{eq:smoothed-LP}, and taking the limit $\lim_{\veps \searrow 0}$, yields relation \eqref{eq:BP} with the sum replaced by the $\max$ operation, as a consequence of taking the $\log$ of both sides and relation \eqref{eq:def-logexp}. This fixed point iteration is called \textit{max-product algorithm} in the literature.

\vspace{0.25cm}
\noindent
From this viewpoint, our alternative approach \eqref{eq:def-f-tau-alpha} emerges as follows, starting at the smoothed primal LP \eqref{eq:smoothed-LP} and following the idea of the proof from Lemma \ref{lem:LP-reformulation}.
\begin{subequations}
\begin{align}
&\min_{\mu \in \mc{L}_{\mc{G}}} \la \theta, \mu \ra - \veps H(\mu) \\
=\quad &\min_{\mu \in \mc{L}_{\mc{G}}} \la \theta, \mu \ra - \veps \Big(\sum_{ij \in \mc{E}} H(\mu_{ij}) - \sum_{i \in \mc{V}}\big(d(i)-1\big) H(\mu_{i})\Big)\\
=\quad &\min_{\mu \in \mc{L}_{\mc{G}}}\la \theta_{\mc{V}}, \mu_{\mc{V}} \ra + \la \theta_{\mc{E}}, \mu_{\mc{E}} \ra - \veps \sum_{ij \in \mc{E}} H(\mu_{ij})+ \veps \sum_{i \in \mc{V}}\big(d(i)-1\big) H(\mu_{i})
\\
=\quad &\min_{\mu_{\mathcal{V}} \in \Delta_{n}^{m}} E_{\veps}(\mu_{\mc{V}}) + \veps \sum_{i \in \mc{V}}\big(d(i)-1\big) H(\mu_{i}).
\end{align}
\end{subequations}
Formulation \eqref{eq:def-f-tau-alpha} results from replacing $\veps$ by a smoothing parameter $\tau$ which can be set to a value \textit{not very} close to $0$ (cf.~Remark \ref{rem:role-of-smoothing}), and we absorb the second nonnegative factor weighting the entropy term by a second parameter $\alpha$. As demonstrated in Section \ref{sec:Experiments}, this latter parameter enables to control precisely the trade-off between accuracy of labelings in terms of the given objective function $E_{\tau}$ of \eqref{eq:def-f-tau-alpha}, that approximates the original discrete objective function \eqref{eq:def-E}, and the speed of convergence to an integral (labeling) solution.

Regarding the resulting term $E_{\tau}$, a key additional step is to use the reformulation \eqref{eq:def-J-smooth}, because all edge-based variables are \textit{locally} `dualized away', as done \textit{globally} with \textit{all} variables when using established belief propagation (cf.~\eqref{eq:BP}). In this way, we can work in the primal domain and with graphs having higher connectivity, without suffering from the enormous memory requirements that would arise from merely smoothing the LP and solving \eqref{eq:smoothed-LP} in the primal domain. Furthermore, the `messages' defined by our approach have a clear interpretation in terms of the smoothed Wasserstein distance between local marginal measures.

\vspace{0.25cm}
We summarize this discussion by contrasting directly established belief propagation with our approach in terms of the following \textbf{key observations}. Regarding belief propagation, we have:
\begin{enumerate} 
\item 
\textbf{Local non-convexity.} The negative $-H(\mu)$ of the so-called \textit{Bethe entropy} function $H(\mu)$ is \textit{non-convex} in general for graphs $\mc{G}$ with cycles \cite[Section 4.1]{WainrightJordan08}, due to the negative sign of the second sum of \eqref{eq:smoothed-LP}. 
\item
\textbf{Local rounding at each step.} The max-product algorithm performs \textit{local rounding} at \textit{every} step of the iteration so as to obtain integral solutions, i.e.~a \textit{labeling} after convergence. This operation results as limit of a \textit{non-convex} function, due to (1).
\item
\textbf{Either nonsmoothness or strong nonlinearity.} The latter $\max$-operation is inherently nonsmooth. Preferring instead a smooth approximation with $0<\veps \ll 1$ necessitates to choose $\veps$ very small so as to ensure rounding. This, however, leads to \textit{strongly nonlinear} functions of the form \eqref{eq:def-logexp} that are difficult to handle numerically.
\item
\textbf{Invalid constraints.} Local marginalization constraints are only satisfied \textit{after} convergence of the iteration. Intuitively it is plausible that, by only \textit{gradually} enforcing constraints in this way, the iterative process becomes more susceptible to getting stuck in unfavourable stationary points, due to the non-convexity according to (1).
\end{enumerate}
Our \textit{geometric approach} removes each of these issues. \textit{Message passing} with respect to vertex $i \in \mc{V}$ is defined by evaluating the local Wasserstein gradients of \eqref{eq:euclidean_gradient_smooth_energy_general} for all edges incident to $i$. We therefore call these local gradients \textbf{Wasserstein messages} which are `passed along edges'. Similarly to \eqref{eq:BP}, each such message is given by \textit{dual} variables through \eqref{eq:nabla-dtau}, that solve the regularized \textit{local} dual LPs \eqref{eq:def-g-dual}. As a consequence, local marginalization constraints are \textit{always} satisfied, throughout the iterative process. \\

In addition, we make the following \textbf{observations} in correspondence to the points (1)-(4) above:
\begin{enumerate}
\item
\textbf{Local convexity.} Wasserstein messages of \eqref{eq:euclidean_gradient_smooth_energy_general} are defined by local \textit{convex} programs \eqref{eq:def-g-dual}. This contrasts with loopy belief propagation and holds true for any pairwise model parameters $\theta_{ij}$ of the prior of the graphical model and the corresponding coupling of $\mu_{i}$ and $\mu_{j}$. This removes spurious minima introduced through non-convex entropy approximations.
\item
\textbf{Smooth global rounding after convergence.} Rounding to integral solutions is \textit{gradually} enforced through the Riemannian flow induced by the extended objective function \eqref{eq:def-f-tau-alpha}. In particular, repeated `aggressive' local $\max$ operations of the max-product algorithm are replaced by a \textit{smooth} flow.
\item
\textbf{Smoothness and weak nonlinearity.} The role of the smoothing parameter $\tau$ of \eqref{eq:def-J-smooth} \textit{differs} from the role of the smoothing parameter $\veps$ of \eqref{eq:smoothed-LP}. While the latter has to be chosen quite close to $0$ so as to achieve rounding at all, $\tau$ merely mollifies the dual local problems \eqref{eq:def-g-dual} and hence should be chosen small, but may be considerably larger than $\veps$. In particular, this does not impair rounding due to (2), which happens due to the \textit{global} flow which is \textit{smoothly} driven by the Wasserstein messages. This \textit{decoupling} of smoothing and rounding enables to numerically compute labelings more efficiently. The results reported in Section \ref{sec:Experiments} demonstrate this fact.
\item
\textbf{Valid constraints.} By construction, computation of the Wasserstein messages enforces all local marginalization constraints \textit{throughout} the iteration. This is in sharp contrast to belief propagation where this generally holds after convergence only. Intuitively, it is plausible that our \textit{more tightly} constrained iterative process is less susceptible to getting stuck in poor local minima. The results reported in Section \ref{sec:experiments-K3} provide evidence of this conjecture.
\end{enumerate}

%% file: TexInput/Numerical-Optimization.tex
In this section we discuss several aspects of the implementation of our approach. The numerical update scheme used in our implementation is given by \eqref{eq:W-update-alpha},
\begin{equation}\label{eq:W-update-alpha-b}
  W_i^{(k+1)} = \frac{(W_i^{(k)})^{1+\alpha} \cdot e^{-h \nabla_i E_\tau(W^{(k)})}}{\la (W_i^{(k)})^{1+\alpha}, e^{-h \nabla_i E_\tau(W^{(k)})}\ra},
    \quad W_i^{(0)} = \frac{1}{n} \eins_n, \quad i\in \mc V
\end{equation}
where $\alpha \geq 0$ is the \textit{rounding} parameter, $h > 0$ the step-size and $\tau$ the \textit{smoothing} parameter
for the local Wasserstein distances.

Section \ref{sec:normalization} details a strategy for maintaining in a \textit{numerically stable} way strict positivity of all variables defined on the assignment manifold. Numerical aspects of computing local Wasserstein gradients are discussed in Section \ref{sec:Wasserstein-Numerics}, and the natural role of the entropy function is highlighted for assuming the role of the smoothing function $F_{\tau}$ in eq.~\eqref{eq:def-WD-smoothed}. Our criterion for convergence and terminating the iterative process \eqref{eq:W-update-alpha-b} of label assignment is specified in Section \ref{sec:termination}.

\subsection{Assignment Normalization}\label{sec:normalization}
The rounding mechanism addressed by Proposition \ref{prop:W-update-alpha} and Remark \ref{rem:dc-perspective} will be effective if $\alpha_{h}$ in \eqref{eq:def-f-tau-alpha} is chosen large enough  to compensate the influence of the function $F_{\tau}$ that regularizes the local Wasserstein distances \eqref{eq:def-WD-smoothed}. 

In this case, each vector $W_i$ approaches some 
vertex $e_i$ of the simplex and thus some entries of $W_i$ converge to zero. However, due to our optimization scheme every vector $W_i$ evolves on the interior of 
the simplex $\mc S$, that is all entries of $W_i$ have to be positive all the time -- see also Remark \ref{rem:strong-duality}. Since there is a limit for the precision of
representing small positive numbers on a computer, we avoid numerical problems by adopting the normalization strategy of \cite{Astroem2017}. After each iteration, we check all 
$W_i$ and whenever an entry drops below $\varepsilon = 10^{-10}$, we rectify $W_i$ by
\begin{equation}
  W_i \ \leftarrow \ \frac{1}{\la \eins, \tilde{W}_i\ra}\tilde{W}_i\;,\quad \tilde{W}_i = W_i - \min_{j = 1, \ldots, n} \{W_{i,j}\} + \varepsilon\;,\quad \varepsilon = 10^{-10}\;.
\end{equation}
Thus, the constant $\varepsilon$ plays the role of $0$ in our implementation. Our numerical experiments showed that this operation avoids numerical issues.

\subsection{Computing Wasserstein Gradients}
\label{sec:Wasserstein-Numerics}
A core subroutine of our approach concerns the computation of the local Wasserstein gradients as part of the overall gradient \eqref{eq:euclidean_gradient_smooth_energy_general}. We argue in this section why the  
\textit{negative entropy function} that we use in our implementation for smoothing the local Wasserstein distances, plays a distinguished role. To this end, we adopt again in this section the notation \eqref{eq:simplified_notation}.

Using this notation the
\textit{smooth} entropy regularized Wasserstein distance \eqref{eq:def-WD-smoothed} reads
\begin{equation}\label{eq:entropy-Wasserstein}
d_{\Theta, \tau}(\mu_{1},\mu_{2}) 
= \min_{M \in \R^{n \times n}}\, \la\Theta,M\ra - \tau H(M)
\quad\text{s.t.}\quad \mc{A} M = \bpm \mu_{1} \\ \mu_{2} \epm,\quad M \geq 0\;,
\end{equation}
with the entropy function 
\begin{equation}
H(M) = -\sum_{i,j} M_{i,j} \log M_{i,j}.
\end{equation}

As shown in Section \ref{sec:Objective-Gradient} and according to Theorem \ref{theo:gradientWasserstein}, the gradients of 
\eqref{eq:entropy-Wasserstein} are the maximizer of the corresponding dual problem. Using the notation \eqref{eq:simplified_notation}, the dual problem of 
\eqref{eq:entropy-Wasserstein} reads 
\begin{equation}\label{eq:dualObjExp}
\max_{\nu \in \R^{2n}} \la\mu, \nu\ra - \tau\sum_{k,l} \exp\Big [ \frac{1}{\tau}\Big(\mc{A}^{\T}\nu - \Theta \Big)_{k,l}\Big].
\end{equation}
In particular, in view of the general form \eqref{eq:dual-smooth-local-W-distance} of this dual problem, the indicator function \eqref{eq:dual-constraints} is smoothly approximated by the function $\tau \exp(\frac{1}{\tau} x)$. Figure  \ref{fig:log_barrier_vs_exp} compares this approximation with the classical
logarithmic barrier $- \log(-x)$ function for approximating the indicator function $\delta_{\R_{-}}$ of the nonpositive orthant. 
Log-barrier penalty functions are the method of choice for \textit{interior point methods} \cite{nesterov1987,Terlaky1996},
 which \textit{strictly} rule out violations of the constraints. While this is essential for many applications where constraints
 represent physical properties that cannot be violated, it is \textit{not} essential in the present case for calculating the
Wasserstein messages. Moreover, the bias towards interior points by log-barrier functions, as Figure \ref{fig:log_barrier_vs_exp} clearly shows, is detrimental in the present context and favours the formulation \eqref{eq:dualObjExp}.

\begin{figure}[h]
\centering
\includegraphics{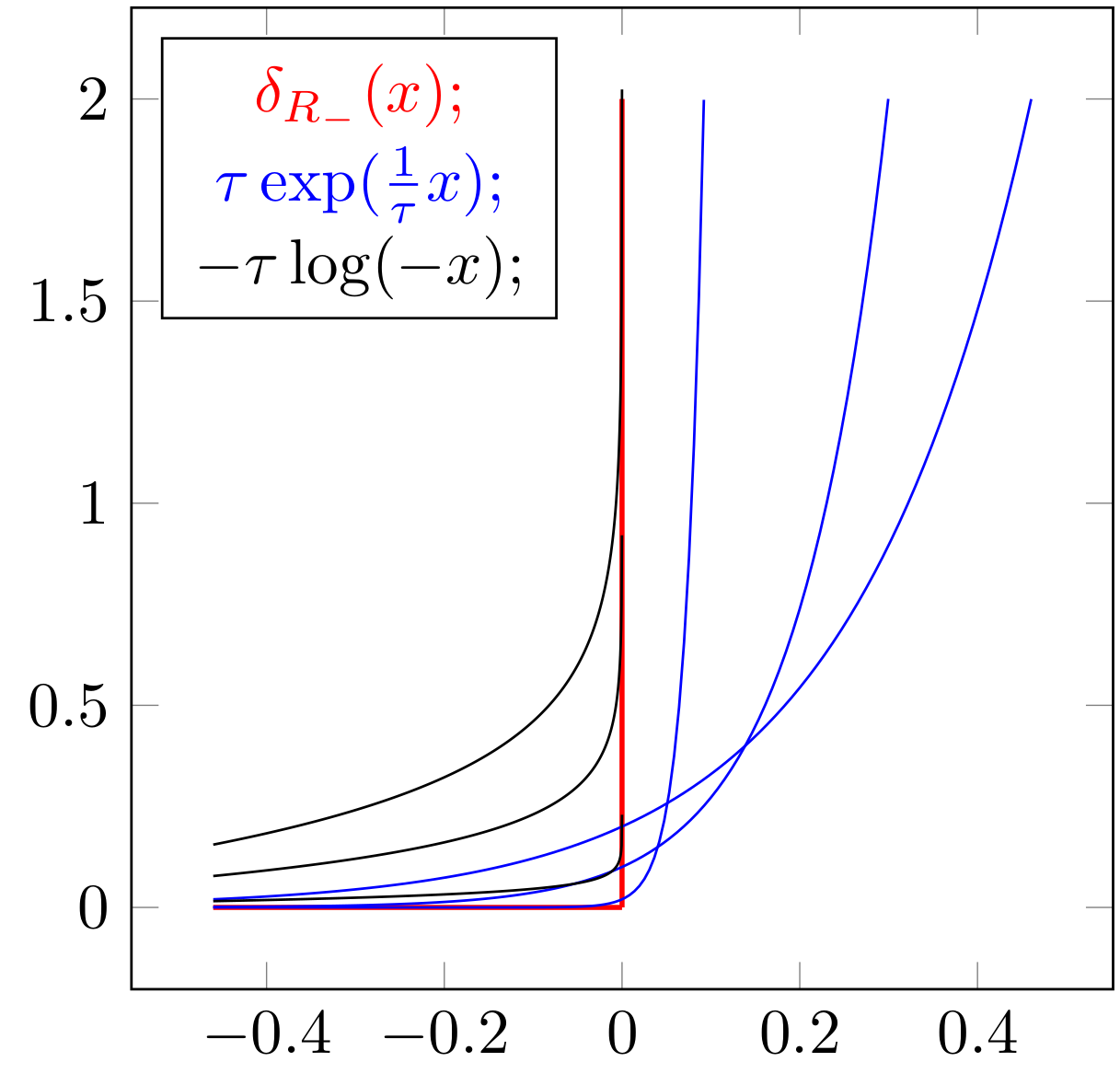}
\caption{Approximations of the indicator function $\delta_{\R_{-}}$ of the nonpositive orthant. The log-barrier function (black curves) strictly rules out violations of the constraints but induce a bias towards interior points. Our formulation (blue curves) is less biased and reasonable approximates the $\delta$-function (red curve) depending on the smoothing parameter 
$\tau$. Displayed are the approximations of $\delta_{\R_{-}}$ for $\tau = \frac{1}{5}, \frac{1}{10}, \frac{1}{50} $.}
\label{fig:log_barrier_vs_exp}
\end{figure}

\vspace{0.2cm}
We now make explicit how the local Wasserstein gradients \eqref{eq:nabla-dtau} are computed based on the formulation \eqref{eq:entropy-Wasserstein} and examine numerical aspects depending on the smoothing parameter $\tau$. It is well known that doubly stochastic matrices as solutions of convex programs like \eqref{eq:entropy-Wasserstein} can be computed by iterative matrix scaling \cite{sinkhorn1964,Schneider1990}, \cite[ch.~9]{Brualdi2006}. This has been made popular in the field of machine learning by \cite{Cuturi2013}.

The optimality condition \eqref{eq:dtau-opt-conditions} takes the form
\begin{equation}\label{eq:smoothed-W-dual-OC}
\ol{M}= \exp\Big[\frac{1}{\tau}\Big(\mc{A}^{\T} \ol{\nu}- \Theta \Big)\Big],
\end{equation} 
and rearranging yields the connection to matrix scaling:
\begin{equation}
\begin{aligned}
\ol{M} &= \exp\Big[\frac{1}{\tau}\Big(\mc{A}^{\T} \ol{\nu}- \Theta \Big)\Big] 
\stackrel{\eqref{eq:def-AT}}{=} \exp\Big[\frac{1}{\tau}\Big(\ol\nu_{1} \eins_{n}^{\T} + \eins_{n} \ol\nu_{2}^{\T} - \Theta  \Big)\Big]\\
&\stackrel{\phantom{\eqref{eq:def-AT}}}{=} \big (\exp(\tfrac{\ol\nu_1}{\tau}) \exp(\tfrac{\ol\nu_2}{\tau})^T\big ) \cdot \exp \big(-\tfrac{1}{\tau} \Theta \big ) 
\stackrel{\phantom{\eqref{eq:def-AT}}}{=} \Diag\big(\exp(\tfrac{\ol\nu_1}{\tau})\big)  \exp \big(-\tfrac{1}{\tau} \Theta \big )    \Diag\big(\exp(\tfrac{\ol\nu_2}{\tau})\big),
\end{aligned}
\end{equation}
where $\Diag(\cdot)$ denotes the diagonal matrix with the argument vector as entries. For given marginals 
$\mu = (\mu_1, \mu_2)$ due to \eqref{eq:entropy-Wasserstein} and with the shorthand $K = \exp\big(- \tfrac{1}{\tau}\Theta \big)$, the optimal dual variables 
$\ol{\nu}=(\ol{\nu}_1,\ol{\nu}_2)$ can be determined by the Sinkhorn's iterative algorithm \cite{sinkhorn1964}, up to a common multiplicative constant. Specifically, we have
\begin{lemma}[{\cite[Lemma 2]{Cuturi2013}}]
For $\tau > 0$, the solution $\ol M$ of \eqref{eq:entropy-Wasserstein} is unique and has the form $\ol M = \mathrm{diag}(v_1) K
\mathrm{diag}(v_2)$, where the two vectors $v_1, v_2 \in \mathbb R^n$ are uniquely defined up to a multiplicative 
factor.
\end{lemma}
\noindent
Accordingly, by setting 
\begin{equation}\label{eq:vi-nui}
v_{1} := \exp({\tfrac{\nu_{1}}{\tau}}),\qquad
v_{2} := \exp({\tfrac{\nu_{2}}{\tau}}),
\end{equation}
the corresponding fixed point iterations read
\begin{equation}\label{eq:vij-single-updates}
v_{1}^{(k+1)} = \frac{\mu_{1}}{K \Big(\frac{\mu_{2}}{K^\top v_{1}^{(k)}}\Big)},\qquad
v_{2}^{(k+1)} = \frac{\mu_2}{K^\top \Big(\frac{\mu_1}{K v_{2}^{(k)}}\Big)},
\end{equation}
which are iterated until the change between consecutive iterates is small enough. Denoting the iterates after convergence by $\ol v_{1}, \ol v_{2}$, resubstitution into \eqref{eq:vi-nui} determines the optimal dual variables
\begin{equation}\label{eq:optimal-dual-variables-sinkhorn}
\overline{\nu}_{1} = \tau \log \ol v_{1},\qquad
\overline{\nu}_{2} = \tau \log \ol v_{2}. 
\end{equation}
Due to Theorem \ref{theo:gradientWasserstein}, the local Wasserstein gradients then finally are given by
\begin{equation}\label{eq:edge-tangents}
    \nabla d_{\Theta, \tau}(\mu_1, \mu_2) = \bpm P_{T}(\ol{\nu}_{1}) \\ P_{T}(\ol{\nu}_{2}) \epm,
\end{equation}
where the projection $P_T$ due to \eqref{eq:projection_matrix_Rn_to_T} removes the common multiplicative constant resulting from  Sinkhorn's algorithm.
\begin{figure}
\centerline{
\hbox{\hspace{-3em}\includegraphics[width=1.15\textwidth]{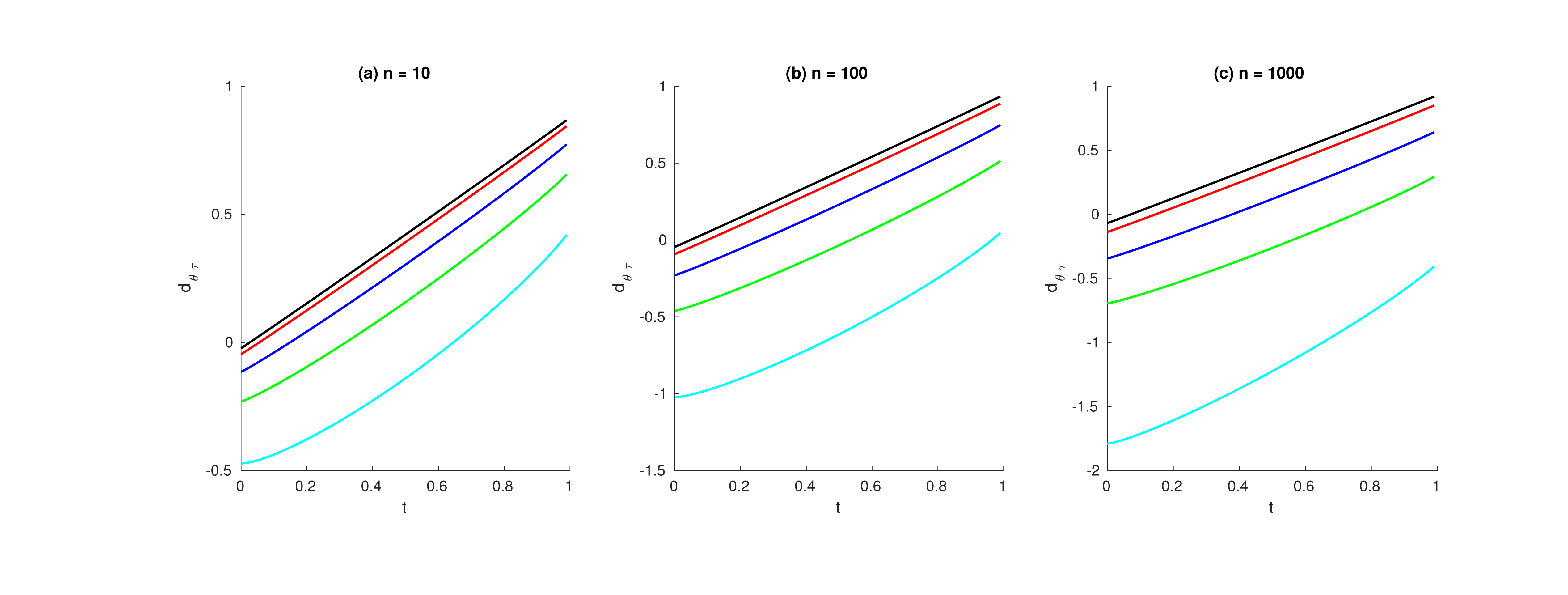}}
}
\caption{
The plots show the entropy-regularized Wasserstein distance $d_{\Theta, \tau}(c, \gamma(t))$ for varying parameter $\tau$ and 
increasing numbers $n$ of labels. Here, $\gamma(t) = t (e_1 - c) + c\in \Delta_n$, with $t \in [0, 1]$, is the line segment connecting the 
barycenter $c = \frac{1}{n} \eins$ to the vertex $e_1$ on the simplex $\Delta_n$. The cost matrix $\Theta$ is given by the 
Potts regularizer \eqref{eq:definition_potts_reg}. In all three plots the parameter $\tau$ has been chosen 
as $\tau = \frac{1}{5}$ (cyan), $\tau = \frac{1}{10}$ (green), $\tau = \frac{1}{20}$ (blue), $\tau = \frac{1}{50}$ (red) and $\tau = \frac{1}{100}$ (black).
Even though the values of the approximation of the distance itself differ considerably, the \textit{slope} of the distance, is already approximated quite well for larger values of $\tau$,    uniformly for small up to large numbers $n$ of labels. 
}
\label{fig:WassersteinLineApprox}
\end{figure}
\begin{figure}
\centerline{
\includegraphics[width=1\textwidth]{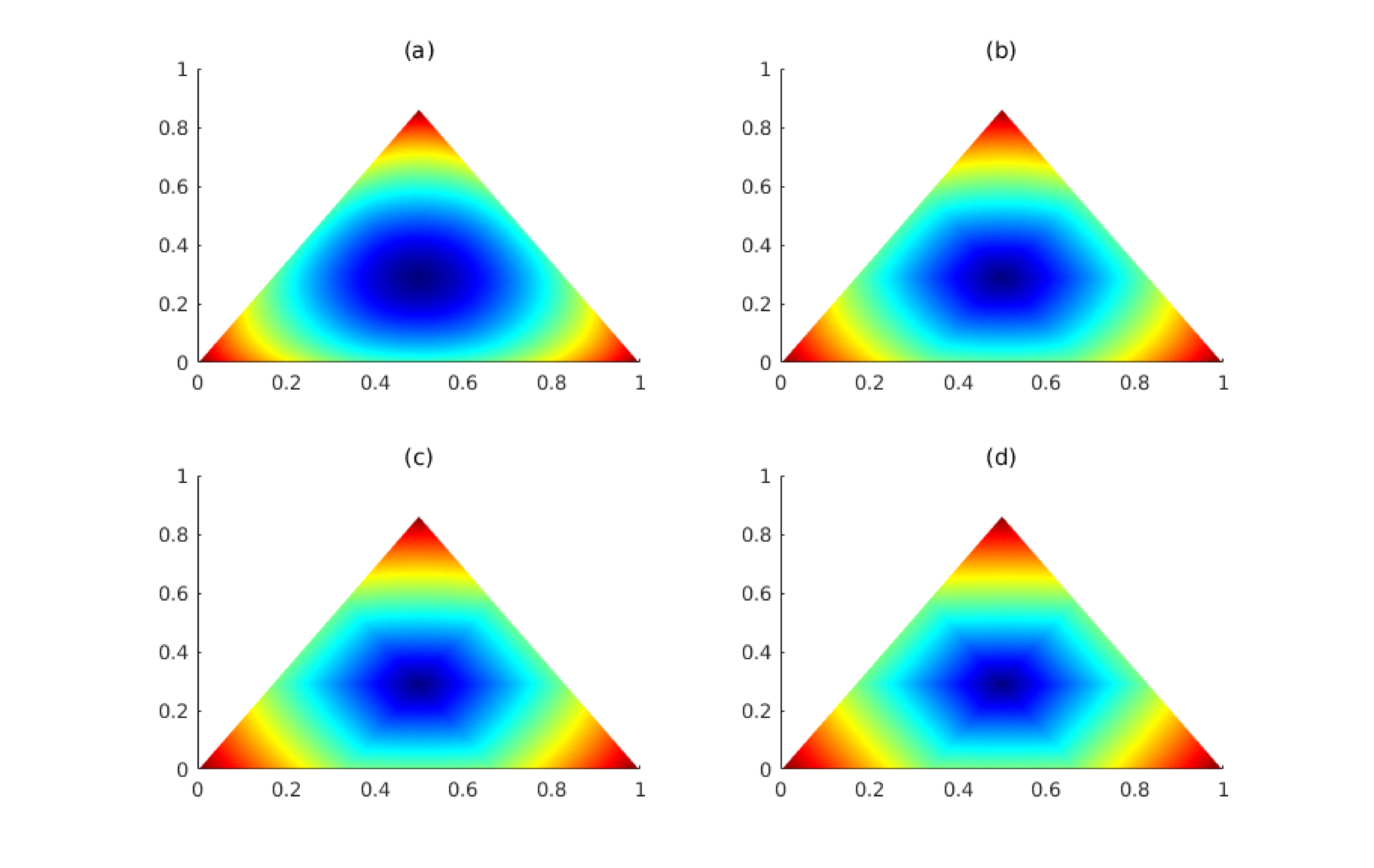}
}
\caption{
The plot shows the entropy-regularized Wasserstein distance with the Potts regularizer \eqref{eq:definition_potts_reg} from the barycenter to every point on $\Delta_3$ for different values of $\tau$: (a) $\tau = \tfrac{1}{5}$, (b) $\tau = \tfrac{1}{10}$, (c) $\tau = \tfrac{1}{20}$ and (d) $\tau = \tfrac{1}{50}$.
These plots confirm that even for relatively large values of $\tau$, e.g. $\tfrac{1}{10}$ and $\tfrac{1}{20}$, the gradient of 
the Wasserstein distance is sufficiently accurate approximated so as to obtain valid descent directions for distance minimization.
}
\label{fig:WassersteinApprox3Labels}
\end{figure}

\vspace{0.25cm}
While the linear convergence rate of Sinkhorn's algorithm is known theoretically \cite{Knight2008}, the numbers of iterations required in practice significantly depends on the smoothing parameter $\tau$.
In addition, for smaller values of $\tau$, an entry of the matrix $K = \exp\big(- \tfrac{1}{\tau}\Theta \big)$ might be too small to be represented on a computer, due to machine precision. As a consequence, the matrix $K$ might have entries which are numerically treated as zeros and Sinkhorn's algorithm does not necessarily converge to the true optimal solution.

Fortunately, our approach does allow larger values of $\tau$ because merely a sufficiently accurate approximation of the \textit{gradient} of the Wasserstein distance is required, rather than an approximation of the Wasserstein distance itself, to obtain valid \textit{descent} directions. Figures \ref{fig:WassersteinLineApprox} and \ref{fig:WassersteinApprox3Labels} demonstrate that this indeed holds for relatively large values of $\tau$, e.g. $\tau \in \{\tfrac{1}{5}, \tfrac{1}{10}, \tfrac{1}{15}\}$, no matter if the number of labels is $n=10$ or $n=1000$.

\subsection{Termination Criterion}\label{sec:termination}
In all experiments, the normalized averaged entropy
\begin{equation}\label{eq:normalized_avg_entropy}
  \frac{1}{m \log(n)} H(W) = - \frac{1}{m \log(n)}\sum_{i\in \mc V} \sum_{k = 1}^n W_{i,k} \log\big(W_{i,k}\big),\quad \text{for} \quad W \in \mc W,
\end{equation}
was used as a termination criterion, i.e. if the value drops below a certain threshold the algorithm is terminated. Due to this normalization,
the value does not depend on the number of labels and thus the threshold is comparable across different models with a varying number of pixels and labels.

For example, a threshold of $10^{-4}$ means in practice that, up to a small fraction of nodes $i\in \mc V$, all rows $W_i$ of the assignment matrix $W$
are very close to unit vectors and thus indicate an almost unique assignment of the prototypes or labels to the observed data.

%% file: TexInput/Experiments.tex
We demonstrate in this section main properties of our approach. The dependency of label assignment on the smoothing parameter $\tau$ and the rounding parameter $\alpha$ is illustrated in Section \ref{subsec:parameter_influence}. We comprehensively explored the space of binary graphical models defined on the minimal cyclic graph, the complete graph with three vertices $\mc{K}^{3}$, whose LP-relaxation is known to have a substantial part of nonbinary vertices. The results reported in Section \ref{sec:experiments-K3} exhibit a relationship between $\alpha$ and $\tau$ so that in fact a single effective parameter only controls the trade-off between accuracy of optimization and the computational costs. A competitive evaluation of our approach in Section \ref{sec:comparison-to-other-methods} together with two established and widely applied approaches, sequential tree-reweighted message passing (TRWS) \cite{Kolmogorov-TRMP06} and loopy belief propagation, reveals similar performance of our approach.
Finally, Section \ref{sec:non-potts-prior} demonstrates for a graphical model with pronounced \textit{non}-uniform pairwise model parameters (non-Potts prior) that our geometric approach accurately takes them into account. 

All experiments have been selected to illustrate properties of our approach, rather than to demonstrate and work out a particular application which will be the subject of follow-up work.

\subsection{Parameter Influence}
\label{subsec:parameter_influence}
We assessed the parameter influence of our geometric approach by applying it to a labeling problem. The task is to label a noisy RGB-image 
$f \colon \mc V \to [0, 1]^3$, depicted in Fig. \ref{fig:paramInfluence}, on the grid graph $\mc G = (\mc V, \mc E)$ with minimal neighborhood size $|\mc{N}(i)| = 3\times 3,\; i \in \mc{V}$. 
Prototypical colors $\mc P = \{ l_1, \ldots, l_8\} \subset [0, 1]^3$ (Fig. \ref{fig:paramInfluence}) were used as labels. The unary (or data term) is defined using the $\|\cdot\|_1$ distance
and a scaling factor $\rho>0$ by
\begin{equation}
  \theta_i = \frac{1}{\rho}\big(\|f(i) - l_1\|_1, \ldots, \|f(i) - l_8\|_1 \big), \quad i\in \mc V,
\end{equation}
and Potts regularization is used for defining the pairwise parameters of the model
\begin{equation}\label{eq:definition_potts_reg}
  \big( \theta_{ij} \big)_{k,r} = 1 - \delta_{k, r},\quad \text{where}\quad \delta_{k, r} = \begin{cases} 1 & \text{if } k = r,\\0 & \text{else},\end{cases}
    \quad \text{for}\quad ij \in \mc E.
\end{equation}
The feature scaling factor was set to $\rho = 0.3$, the step-size $h = 0.1$ was used for numerically integrating the Riemannian descent flow, and the threshold for the normalized average entropy termination 
criterion \eqref{eq:normalized_avg_entropy} was set to $10^{-4}$.

Fig.~\ref{fig:parameterInfluence_nAvgEntropy_t10},  top, displays the empirical convergence rate depending on the rounding parameter $\alpha$, for a fixed value of the smoothing parameter $\tau=0.1$ that ensures a sufficiently accurate approximation of the Wasserstein distance gradients and hence of the Riemannian descent flow. Fig.~\ref{fig:parameterInfluence_nAvgEntropy_t10},  bottom, shows the interplay between minimizing the smoothed energy $E_\tau$ \eqref{eq:def-J-smooth} and the rounding mechanism induced by the entropy $H$ \eqref{eq-entropy-formula_W} in $f_{\tau, \alpha}$ \eqref{eq:def-f-tau-alpha}. Less agressive rounding in terms of smaller values of $\alpha$ leads to a more accurate numerical integration of the flow using a larger number of iterations, and thus to higher quality label assignments with a lower energy of the objective function. This latter aspect is demonstrated quantitatively in Section \ref{sec:experiments-K3}. For too small values of the rounding parameter $\alpha$, the algorithm does naturally not converge to an integral solution.

Fig.~\ref{fig:paramInfluence} shows the influence of the rounding strength $\alpha$ and the smoothing parameter $\tau$ for the Wasserstein distance.
All images marked with an '$\ast$' in the lower right corner do not show an integral solution, which means that the normalized average 
entropy \eqref{eq:normalized_avg_entropy} of the assignment vectors $W_{i}$ did not drop below the threshold during the iteration and thus, even though the assignments show a clear tendency, they stayed far from integral solutions. As just explained for Fig.~\ref{fig:parameterInfluence_nAvgEntropy_t10}, this is not a deficiency of our approach but must happen if either no rounding is performed ($\alpha = 0$) or if the influence of  rounding is too small compared to the 
smoothing of the Wasserstein distance (e.g.~$\alpha = 0.1$ and $\tau = 0.5$). 
Increasing the strength of rounding (larger $\alpha$) leads to a faster decrease in entropy (cf. Fig.~\ref{fig:parameterInfluence_nAvgEntropy_t10} for the case of $\tau = 0.1$) and therefore to an earlier convergence of the process to a specific labeling. Thus, a more aggressive rounding scheme yields a less regularized result 
due to the rapid decision for a labeling at an early stage of the algorithm.
\begin{figure}[h]
\includegraphics[width=0.8\textwidth]{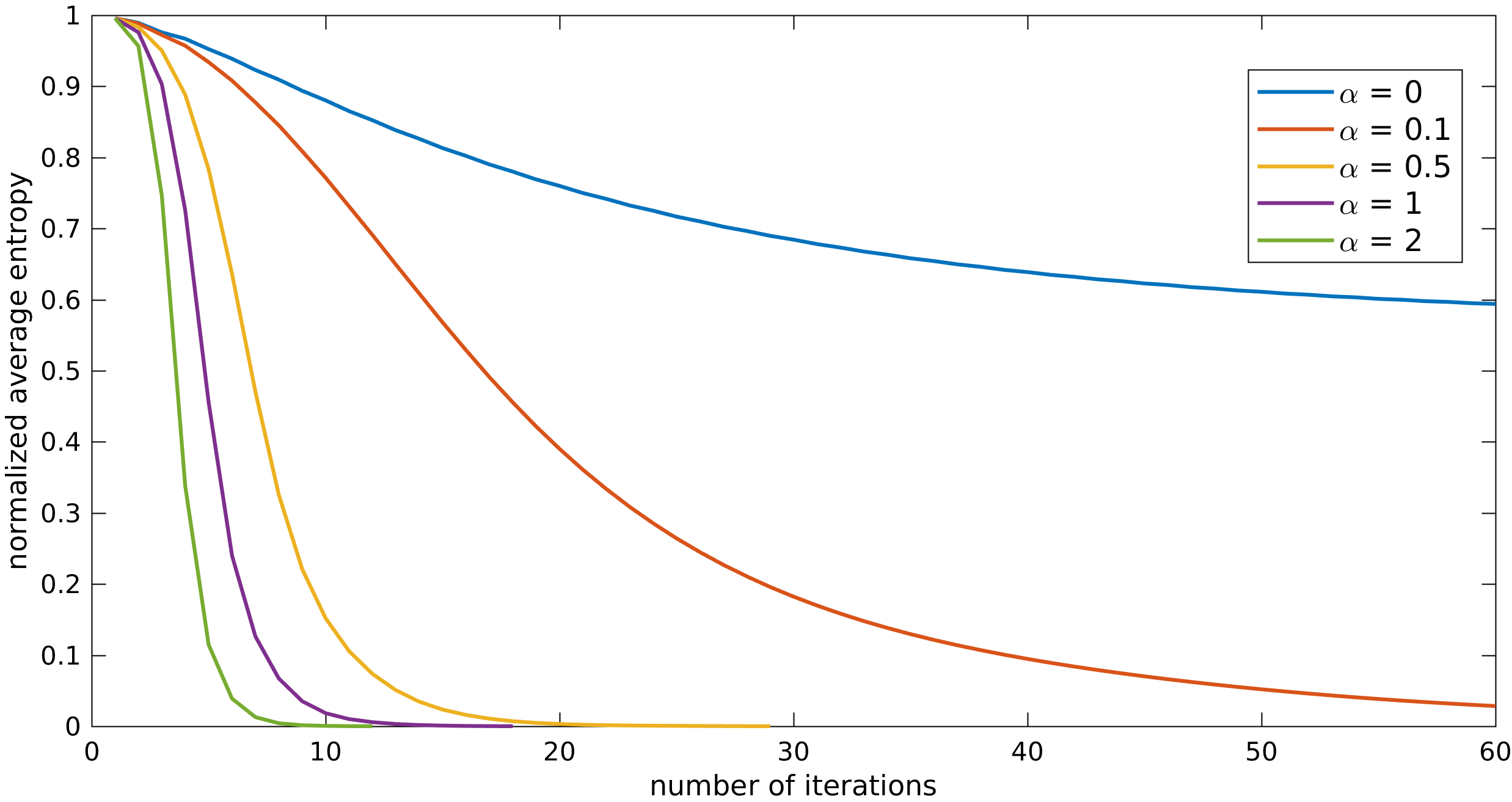} \\
\includegraphics[width=0.8\textwidth]{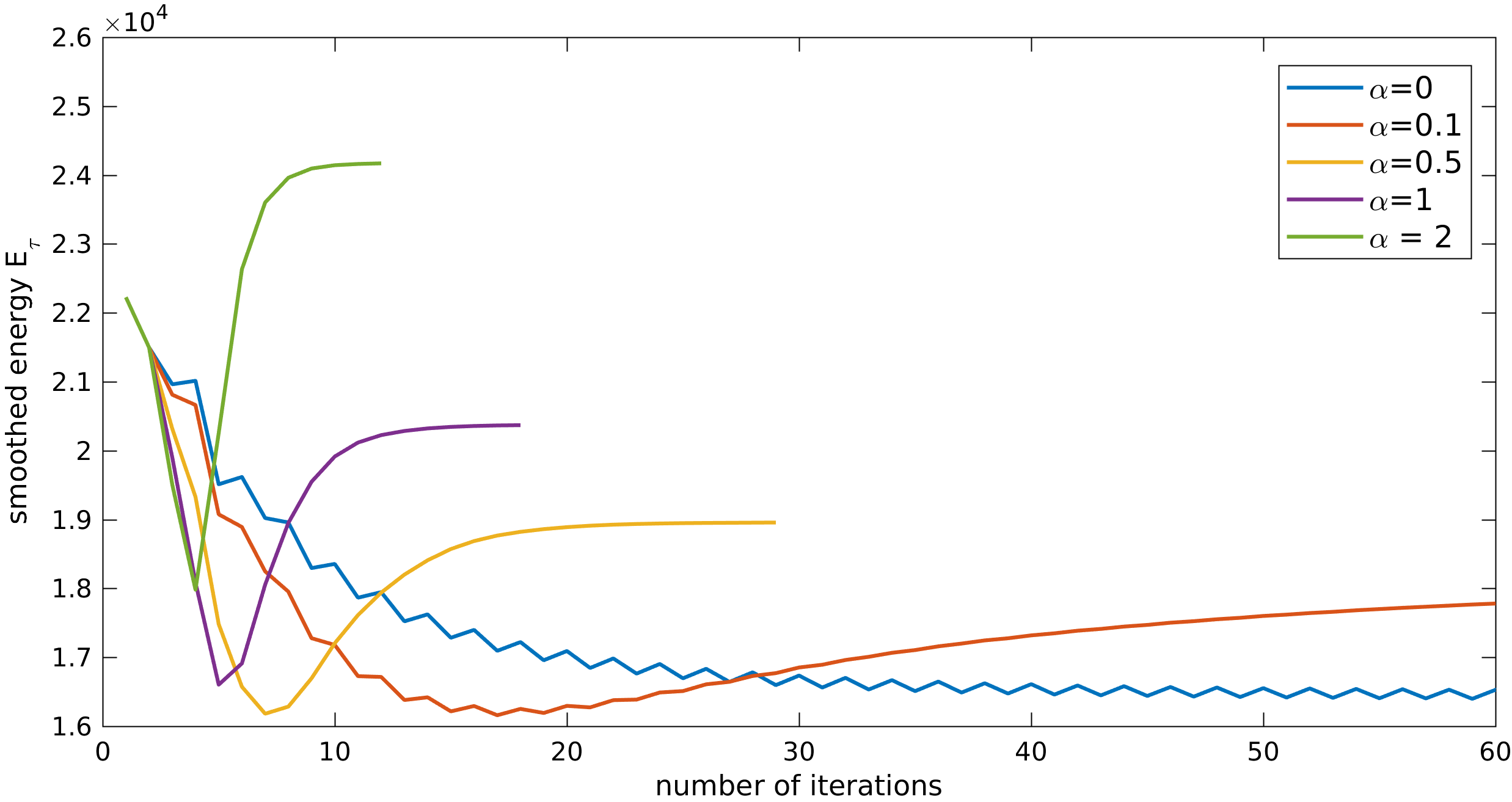}
\caption{
The normalized average entropy \eqref{eq:normalized_avg_entropy} (top) and the smoothed energy $E_\tau$ \eqref{eq:def-J-smooth} (bottom) are shown, for the smoothing parameter 
value $\tau = 0.1$, depending on the number of iterations. \textsc{Top:} With increasing values of the rounding parameter $\alpha$, the entropy drops more rapidly and hence converges faster 
to an integral labeling. \textsc{Bottom:} Two phases of the algorithm depending on the values for $\alpha$ are clearly visible. In the first phase, the smoothed 
energy $E_\tau$ is minimized up to the point where rounding takes over in the second phase. Accordingly, the sequence of energy values first drops down to lower values corresponding to the problem \textit{relaxation} and then adopts a higher energy level corresponding to an \textit{integral} solution. 
For smaller values of the rounding parameter $\alpha$, the algorithm spends more time on minimizing the smoothed energy. This generally results in lower energy values even \textit{after} rounding, i.e.~in higher quality labelings.
}
\label{fig:parameterInfluence_nAvgEntropy_t10}
\end{figure}
\newcommand{\imgsize}{1\linewidth}
\begin{figure}[h]
  \centering
  \newcommand{\xcoord}{1.9}
  \newcommand{\ycoord}{2.05}
  \small
  \begin{minipage}{0.18 \linewidth}
    \includegraphics[width=\imgsize]{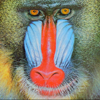} \\
    \begin{tabularx}{1\linewidth}{X}
      {Original data}
    \end{tabularx} \\
    \includegraphics[width=\imgsize]{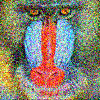} \\
    \begin{tabularx}{1\linewidth}{X}
      {Noisy data}
    \end{tabularx} \\
        \includegraphics[width=\imgsize]{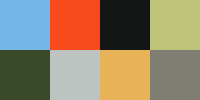} \\
    \begin{tabularx}{1\linewidth}{X}
      {Prototypes}
    \end{tabularx} \\
  \end{minipage}
  \begin{minipage}{0.6\linewidth}
    \begin{tabularx}{1\linewidth}{c X XX}
      & $\tau = {0.5}$ & ${0.1}$ & ${0.05}$ \\
      \vspace{2mm}
      \rotatebox{90}{{$0$}} \hspace{-2mm} &
      \includegraphics[width=\imgsize]{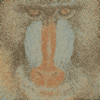}$\ast$ &
      \includegraphics[width=\imgsize]{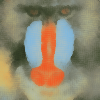}$\ast$ &
      \includegraphics[width=\imgsize]{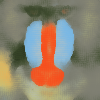}$\ast$ \\
      \vspace{2mm}
      \rotatebox{90}{{$0.1$}} \hspace{-2mm} &
      \includegraphics[width=\imgsize]{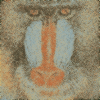}$\ast$ &
      \includegraphics[width=\imgsize]{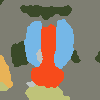} &
      \includegraphics[width=\imgsize]{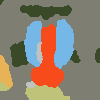} \\
      \vspace{2mm}
      \rotatebox{90}{$0.5$} \hspace{-2mm} &
      \includegraphics[width=\imgsize]{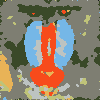} &
      \includegraphics[width=\imgsize]{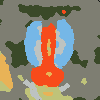} &
      \includegraphics[width=\imgsize]{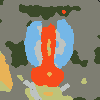} \\
      \vspace{2mm}
      \rotatebox{90}{$1$} \hspace{-2mm} &
      \includegraphics[width=\imgsize]{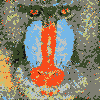} &
      \includegraphics[width=\imgsize]{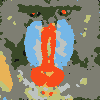} &
      \includegraphics[width=\imgsize]{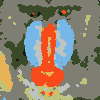} \\
      \vspace{2mm}
      \rotatebox{90}{$\alpha = 2$} \hspace{-2mm} &
      \includegraphics[width=\imgsize]{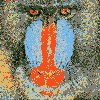} &
      \includegraphics[width=\imgsize]{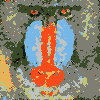} &
      \includegraphics[width=\imgsize]{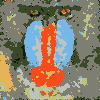} \\
      \end{tabularx}
  \end{minipage}
  \caption{Influence of the rounding parameter $\alpha$ and the smoothing parameter $\tau$ on the assignment of $8$ prototypical labels to noisy input data. All images marked with an '$\ast$' do not show integral solutions due to smoothing too strongly the Wasserstein distance in terms of $\tau$ relative to $\alpha$, which overcompensates the effect of rounding. Likewise, smoothing too strongly the Wasserstein distance (left column, $\tau=0.5$) yields poor approximations of the objective function gradient and to erroneous label assignments. The remaining parameter regime, i.e.~smoothing below a  reasonably large upper bound $\tau=0.1$, leads to fast numerical convergence, and the label assignment can be precisely controlled by the rounding parameter $\alpha$.
}
\label{fig:paramInfluence}
\end{figure}
\begin{figure}[h]
\centerline{
\includegraphics[width=0.9\textwidth]{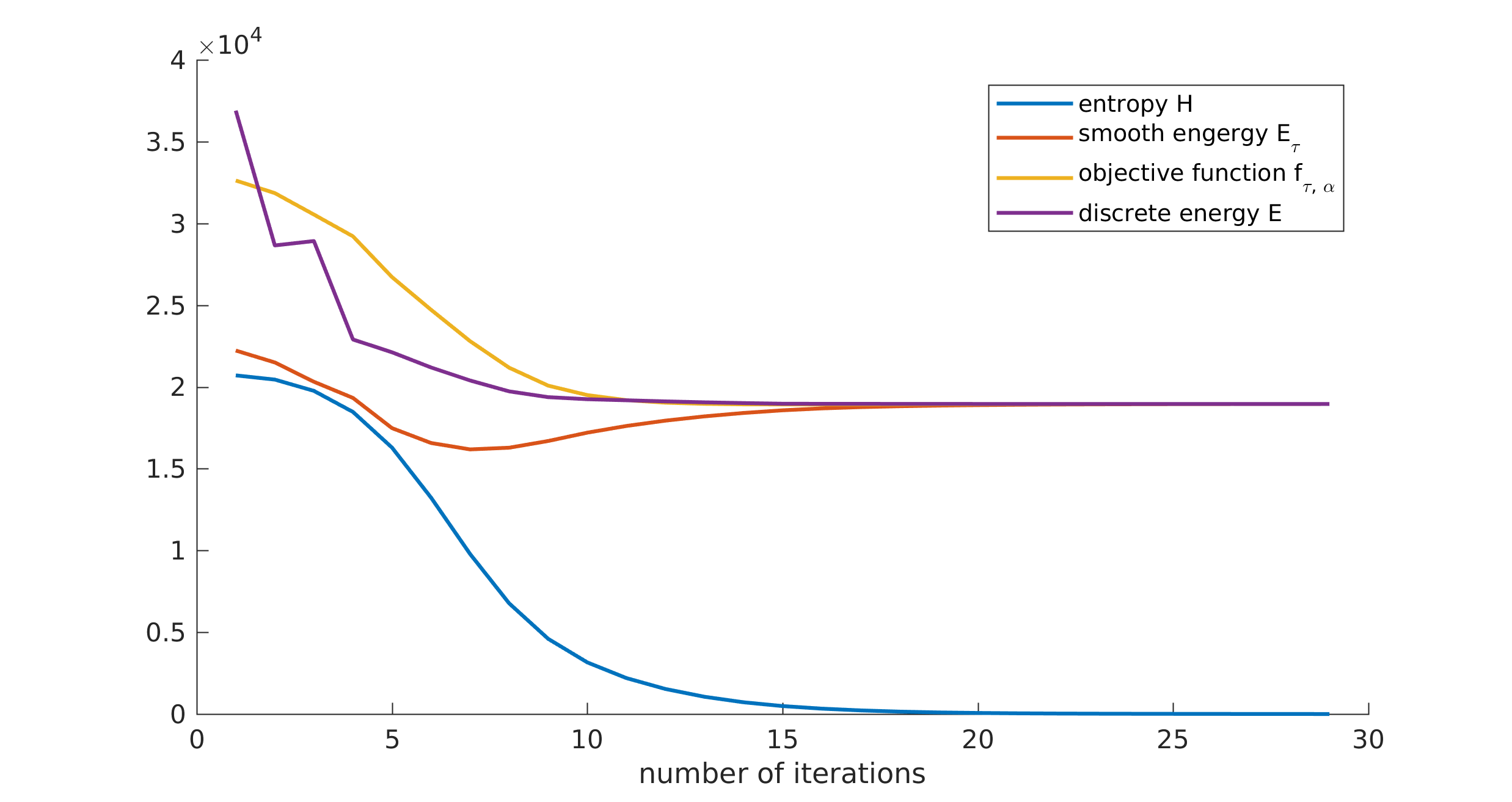}
}
\caption{ 
Connection between the objective function $f_{\tau, \alpha}$ \eqref{eq:def-f-tau-alpha} and the discrete energy $E$ \eqref{eq:def-E} of the underlying graphical model, 
for a fixed value $\alpha = 0.5$. Minimizing $f_{\tau, \alpha}$ (yellow) by our approach also minimizes $E$ (violet), which was calculated for this illustration by   rounding the assignment vectors at every iterative step. Additionally, as already discussed in more detail in connection with Fig.~\ref{fig:parameterInfluence_nAvgEntropy_t10}, the interplay between the two terms of 
$f_{\tau, \alpha} = E_\tau + \alpha H$ is shown, where $E_\tau$ (orange) denotes the smoothed energy \eqref{eq:def-J-smooth} and $H$ (blue) the 
entropy \eqref{eq-entropy-formula_W} causing rounding.
}
\label{fig:parameterInfluence_analysis_t10}
\end{figure}

On the other hand, choosing the smoothing parameter $\tau$ too large lead to poor approximations of the Wasserstein distance gradients and consequently to erroneous non-regularized labelings, as displayed in the left column of Fig.~\ref{fig:paramInfluence} corresponding to $\tau = 0.5$. 
Once $\tau$ is small enough, in our experiments: $\tau < 0.1$, the Wasserstein distance gradients are properly approximated, and the label assignment is regularized as expected and can be controlled by $\alpha$. 
In particular, this upper bound on $\tau$ is sufficiently large to ensure very rapid convergence of the fixed point iteration for computing the Wasserstein distance gradients.

Fig.~\ref{fig:parameterInfluence_analysis_t10} shows the connection between the objective function $f_{\tau, \alpha}$ \eqref{eq:def-f-tau-alpha} 
and the discrete energy $E$ \eqref{eq:def-E} of the underlying graphical model. Minimizing $f_{\tau, \alpha}$ (yellow curve) using our 
approach also minimizes the discrete energy $E$ (violet curve), which was calculated by rounding the assignment vectors after each iterative step. Fig.~\ref{fig:parameterInfluence_analysis_t10} also shows the interplay between the two terms in $f_{\tau, \alpha} = E_\tau + \alpha H$, with smoothed energy \eqref{eq:def-J-smooth} $E_\tau$ plotted as orange curve and with the entropy \eqref{eq-entropy-formula_W} plotted as blue curve. These curves illustrate (i) the smooth combination of optimization and rounding into a single process, and (ii) that the original discrete energy \eqref{eq:def-E} is effectively minimized by this smooth process.

\subsection{Exploring all Cyclic Graphical Models on $\mc{K}^{3}$}\label{sec:experiments-K3}
In this section, we report an exhaustive exploration of all possible binary models, $\mc{X}=\{0,1\}$, on the minimal cyclic graph $\mc{K}^{3}$ (Fig.~\ref{fig:K3}, left panel). Due to the single cycle, models exist where the LP relaxation \eqref{eq:LP-relaxation} returns a non-binary solution (red part of the right panel of Fig.~\ref{fig:K3}). As a consequence, evaluating such models with our geometric approach for minimizing \eqref{eq:def-J-smooth} enables to check two properties: 
\begin{enumerate}[(i)]
\item
Whenever solving the LP relaxation \eqref{eq:LP-relaxation} by convex programming returns the global binary minimum of \eqref{eq:def-E} as solution, we assess if our geometric approach based on the smooth approximation \eqref{eq:def-J-smooth} returns this solution as well. 
\item
Whenever the LP relaxation has a \textit{non-binary} vector as global solution, which therefore is \textit{not} optimal for the labeling problem \eqref{eq:def-E}, we assess the rounding property of our approach by comparing the result with the \textit{correct} binary labeling globally minimizing \eqref{eq:def-E}.
\end{enumerate}

The graph $\mc{K}^{3}$ enables us to specify the so-called \textit{marginal polytope} $\mc{P}_{\mc{K}^{3}}$ whose vertices (extreme points) are the feasible binary combinatorial solutions that correspond to valid labelings (cf.~Section \ref{sec:Overview}), and to examine the difference to the local polytope $\mc{L}_{\mc{K}^{3}}$ whose representation only involves a subset of the constraints corresponding to $\mc{P}_{\mc{K}^{3}}$. We refer to \cite{CorrelationPolytope-89} for background and details.

\begin{figure}[h]
\centerline{
\includegraphics[width=0.35\textwidth]{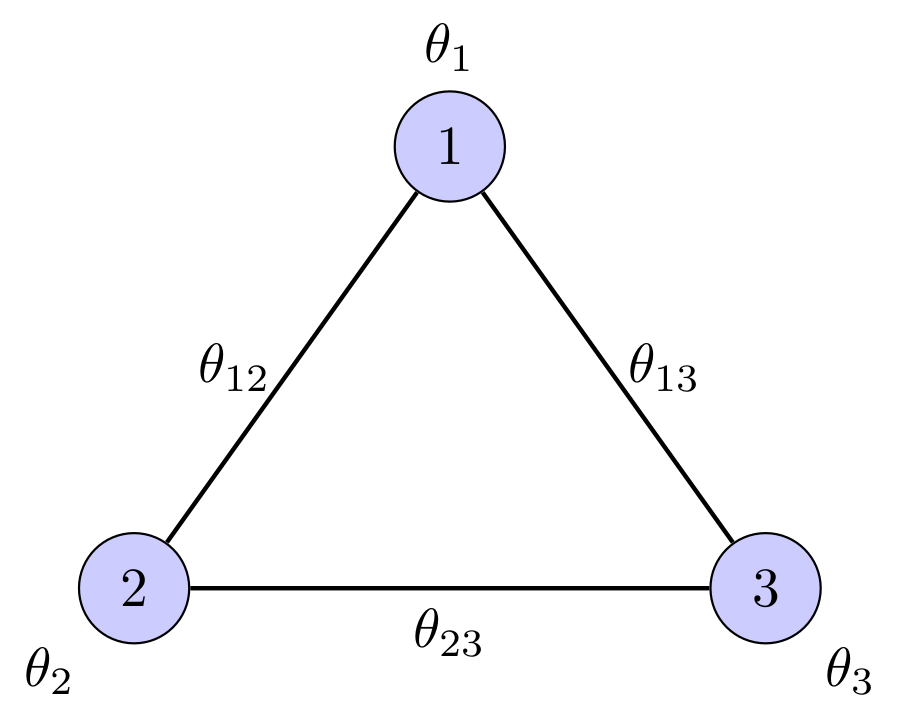} \hspace{0.1\textwidth}
\includegraphics[width=0.18\textwidth]{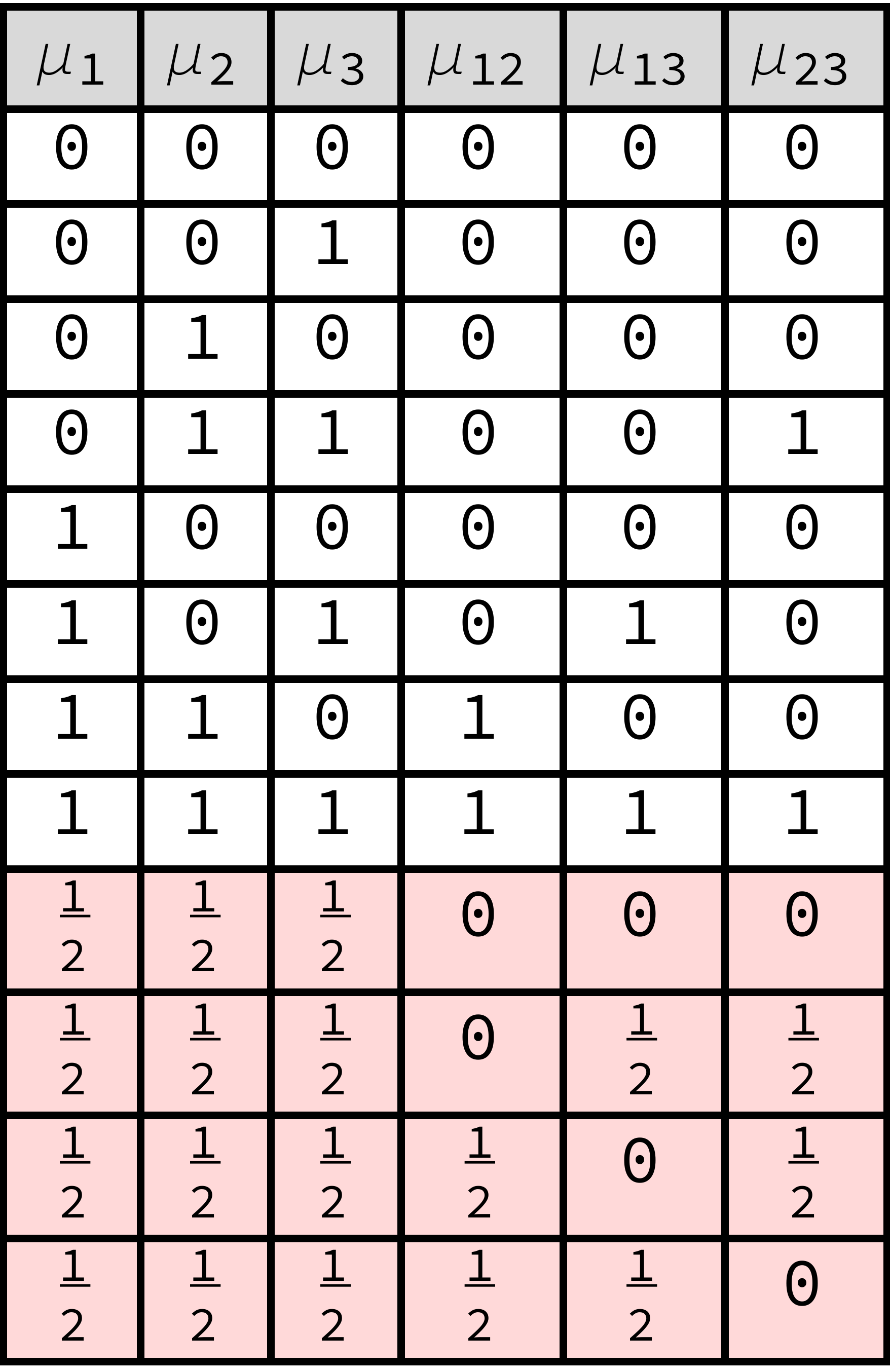}
}
\caption{
\textsc{left:} The minimal binary cyclic graphical model $\mc{K}^{3}=(\mc{V},\mc{E})=(\{1,2,3\},\{12,13,23\})$. \textsc{right:} The $8$ vertices (white background) of the minimally represented marginal polytope $\mc{P}_{\mc{K}^{3}} \subset \R_{+}^{6}$ and the $4$ additional non-integer vertices (red background) of the minimally represented local polytope $\mc{L}_{\mc{K}^{3}} \subset \R_{+}^{6}$.
}
\label{fig:K3}
\end{figure}

The constraints are more conveniently stated using the so-called \textit{minimal representation} of binary graphical models \cite[Sect.~3.2]{WainrightJordan08}, that involves the variables\footnote{We reuse the symbol $\mu$ for simplicity and only `overload' in this subsection the symbols $\mu_{i}, \mu_{ij}$ for local vectors \eqref{eq:def-mu-Pi} by the variables on the left-hand sides of \eqref{eq:def-minimal-mui-muij}}
\begin{equation}\label{eq:def-minimal-mui-muij}
\mu_{i}:=\mu_{i}(1),\; i \in \mc{V},\qquad \mu_{ij}:=\mu_{i}(1) \mu_{j}(1),\; ij \in \mc{E}
\end{equation}
and encodes the local vectors \eqref{eq:def-mu-Pi} by
\begin{equation}
\bpm  1-\mu_{i} \\ \mu_{i} \epm 
\;\leftarrow\;
\bpm \mu_{i}(0) \\ \mu_{i}(1) \epm,
\qquad\qquad
\bpm (1-\mu_{i})(1-\mu_{j}) \\
(1-\mu_{i}) \mu_{j} \\ \mu_{i} (1-\mu_{j}) \\
\mu_{ij} \epm
\;\leftarrow\;
\bpm \mu_{ij}(0,0) \\ \mu_{ij}(0,1) \\ \mu_{ij}(1,0) \\ \mu_{ij}(1,1) \epm.
\end{equation}
Thus, it suffices to use a single variable $\mu_{i}$ for every node $i \in \mc{V}$ instead of two variables $\mu_{i}(0), \mu_{i}(1)$, and also a single variable $\mu_{ij}$ for every edge $ij \in \mc{E}$ instead of four variables $\mu_{ij}(0,0),\mu_{ij}(0,1),\mu_{ij}(1,0),\mu_{ij}(1,1)$.
The \textit{local} polytope constraints \eqref{eq:def-mu-Pi} then take the form
\begin{equation}\label{eq:constraints-L3}
0 \leq \mu_{ij},\qquad
\mu_{ij} \leq \mu_{i},\qquad
\mu_{ij} \leq \mu_{j},\qquad
\mu_{i}+\mu_{j}-\mu_{ij} \leq 1,\qquad \forall ij \in \mc{E}.
\end{equation}
The \textit{marginal} polytope constraints additionally involve the so-called triangle inequalities \cite{Deza-Laurent-97}
\begin{subequations}\label{eq:constraints-P3}
\begin{gather}
\sum_{i \in \mc{V}} \mu_{i} - \sum_{jk \in \mc{E}} \mu_{jk} \leq 1, \\
\mu_{12}+\mu_{13}-\mu_{23} \leq \mu_{1},\qquad
\mu_{12}-\mu_{13}+\mu_{23} \leq \mu_{2},\qquad
-\mu_{12}+\mu_{13}+\mu_{23} \leq \mu_{3}.
\end{gather}
\end{subequations}
Figure \ref{fig:K3}, right panel, lists the $8$ vertices of $\mc{P}_{\mc{K}^{3}}$ and the $4$ additional vertices of $\mc{L}_{\mc{K}^{3}}$ that arise when dropping the subset of constraints \eqref{eq:constraints-P3}.

We evaluated $10^{5}$ models generated by randomly sampling the model parameters \eqref{eq:def-theta-i-theta-ij}: With $\mc{U}[a,b]$ denoting the uniform distribution on the interval $[a,b] \subset \R$, we set
\begin{equation}
\theta_{i}=\bpm 1-p \\ p \epm - \frac{1}{2} \bpm 1 \\ 1 \epm,\quad
p \sim \mc{U}[0,1],\qquad
\theta_{ij} = \bpm p_{1} & p_{2} \\ p_{3} & p_{4} \epm,\quad
p_{i} \sim \mc{U}[-2,2],\; i \in [4].
\end{equation}
Note the different scale, $\theta_{i} \in [-\frac{1}{2},+\frac{1}{2}]^{2}$, $\theta_{ij} \in [-2,+2]^{2 \times 2}$, which results in a larger influence of the pairwise terms and hence make inference more difficult. Suppose, for example, that the diagonal terms of $\theta_{ij}$ are large, which favours the assignment of \textit{different} labels to the nodes $1,2,3 \in \mc{V}$. Then assigning say labels $0$ and $1$ to the vertices $1$ and $2$, respectively, will inherently lead to a large energy contribution due to the assignment to node $3$, no matter if this third label is $0$ or $1$, because it  must agree with the assignment either to node $1$ or to $2$.

Every \textit{binary} vertex listed by Fig.~\ref{fig:K3}, right panel, is the global optimum of both the linear relaxation \eqref{eq:LP-relaxation} and the original objective function \eqref{eq:def-E} in approximately $\approx 11.94\%$ of the $10^{5}$ scenarios, whereas every \textit{non-binary} vertex is optimal in approximately $\approx 1.12\%$.

An example where a \textit{non-binary} vertex is optimal for the linear relaxation \eqref{eq:LP-relaxation} is given by the model parameter values
\begin{equation}\label{eq:model_parameter_rounding_failes}
\begin{aligned}
\theta_1 &= \bpm -0.2261\\ \phantom{-}0.2261 \epm, \quad &\theta_{12} &= \bpm -0.9184 & -1.6252\\ -1.8891 &-0.9807 \epm, \\ 
\theta_2 &= \bpm -0.4449\\ \phantom{-}0.4449 \epm, \quad &\theta_{13} &= \bpm 0.3590 & 0.0958 \\ {-1.8668} & 1.5193\epm, \\
\theta_3 &= \bpm -0.3202\\ \phantom{-}0.3202 \epm, \quad &\theta_{23} &= \bpm 1.2147 & -1.5215\\ -0.3302 & -0.0459\epm.
\end{aligned}
\end{equation}
The corresponding solutions of the marginal polytope $\mc{M}_{\mc{G}}$, the local polytope
 $\mc{L}_{\mc{G}}$ and our method are listed as  Table~\ref{table:triangle_post_processing_failes}. Due to the non-binary solution returned by the LP-relaxation, rounding in a post-processing step amounts to random guessing. In contrast, our method is able to determine the optimal solution because rounding is smoothly integrated into the overall optimization process. 

\bgroup
\newcolumntype{K}[1]{>{\centering\arraybackslash}p{#1}}
\def\arraystretch{1.3}
\begin{table}[h!]
\begin{center}
\begin{tabular}{ |K{2cm}|K{1.3cm}|*{3}{K{1.5cm}|}*{1}{|K{1.4cm}}|}
\hline
\multicolumn{2}{|c|}{}& $\mu_1$ & $\mu_2$ & $\mu_3$ & Iterations\\
\hline
 \multicolumn{2}{|c|}{Marginal Polytope $\mc{M}_{\mc{G}}$} & $1$ & $0$ & $0$ & -\\
\hline 
\multicolumn{2}{|c|}{Local Polytope $\mc{L}_{\mc{G}}$} & $0.5$ & $0.5$ & $0.5$ & -\\
\hline
\multirow{2}{*}{Our Method}   & $\alpha = 0.2$ & $0.999$ & $0.258 \mathrm{e}{-3}$ & $0.205 \mathrm{e}{-3}$ & $108$ \\
\cline{2-6}
\multirow{2}{*}{($\tau = \frac{1}{10}$)} & $\alpha = 0.5$ & $0.999$ & $0.161\mathrm{e}{-3}$ & $0.114\mathrm{e}{-4}$& $14$\\
\cline{2-6}
& $\alpha = 0.9$ & $0.999$ & $0.239\mathrm{e}{-4}$& $0.546\mathrm{e}{-6}$& $8$\\
\hline
\end{tabular}
\end{center}
\caption{Solutions $\mu=(\mu_{1},\mu_{2},\mu_{2})$ of the marginal polytope $\mc{M}_{\mc{G}}$, the local
 polytope $\mc{L}_{\mc{G}}$ and our method, for the triangle model with parameter values 
 \eqref{eq:model_parameter_rounding_failes}. Our method was applied with threshold $10^{-3}$ as  termination criterion \eqref{eq:normalized_avg_entropy}, stepsize $h=0.5$, smoothing parameter $\tau = 0.1$ and three values of the  
 rounding parameter $\alpha \in \{0.2, 0.5, 0.9\}$. By definition, minimizing over the marginal polytope returns the globally optimal discrete solution. 
 The local polytope relaxation has a fractional solution for this model, so that rounding in a post-processing step amounts to random guessing. Our approach returns the global optimum in each case, up to numerical precision.
 } 
\label{table:triangle_post_processing_failes}
\end{table}
\egroup

\begin{figure}[h]
  \centering
  \begin{tabulary}{1\linewidth}{C C}
    \includegraphics[width=0.5\textwidth]{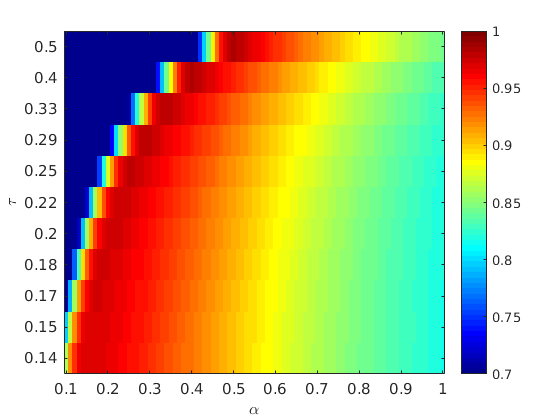} & \includegraphics[width=0.5\textwidth]{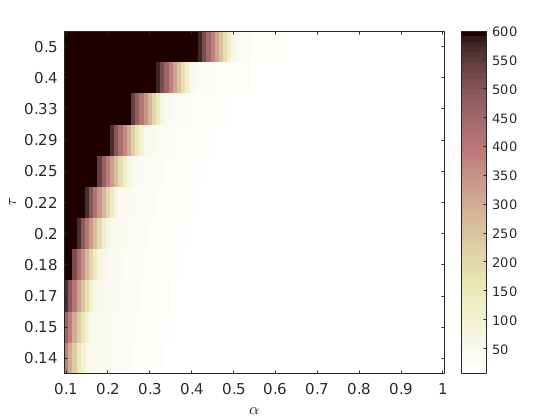} \end{tabulary}
  \caption{\textsc{Evaluation of the Minimal cyclic Graphical Model $\mathcal{K}^3$:} For every pair of parameter values $(\tau,\alpha)$, we evaluated $10^{5}$ models, 
  which were generated as explained in the text. In each experiment, we terminated the algorithm when the average entropy dropped below $10^{-3}$ or if the maximum number of 600 iterations was reached. In addition, we chose a constant step-size $h=0.5$. \textsc{Left:} The plot shows the percentage of experiments 
  where the energy returned by our algorithm had a relative error smaller then $1\%$ compared to the minimal energy of the globally optimal integral labeling. In agreement with Fig.~\ref{fig:parameterInfluence_nAvgEntropy_t10} (bottom), less aggressive rounding yielded labelings closer to the global optimum. \textsc{Right:} This plot shows the corresponding average number of iterations. The black region indicates experiments where the maximum number of 600 iterations was reached, because too strong smoothing of the Wasserstein distance (large $\tau$) overcompensated the effect of rounding (small $\alpha$), so that the convergence criterion
  \eqref{eq:normalized_avg_entropy} which measures the distance to integral solutions, cannot be satisfied. In the remaining large parameter regime, the choice of $\alpha$ enables to control the trade-off between high-quality (low-energy) solutions and computational costs.}
  \label{fig:Results-K3}
\end{figure}
\begin{figure}[h]
  \centering
  \begin{tabulary}{1\linewidth}{C C}
    \includegraphics[width=0.5\textwidth]{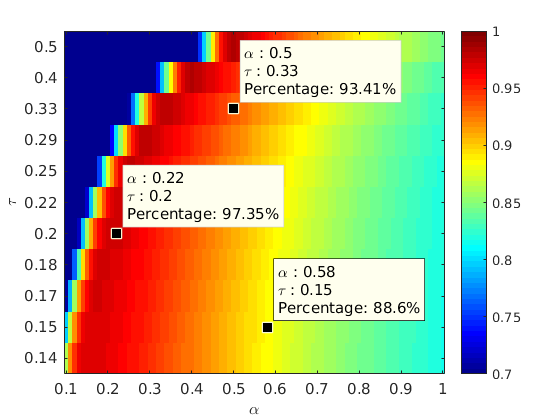} & \includegraphics[width=0.5\textwidth]{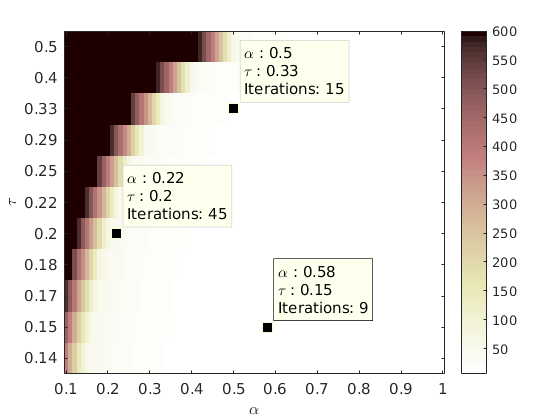} \end{tabulary}
  \caption{The plots display the same results as shown by Fig.~\ref{fig:Results-K3} together with additional data boxes and information for three 
different configurations of parameter values. The comparison of the success rate (left panel) and the number of iterations until convergence (right panel) clearly demonstrates the trade-off between accuracy of optimization and convergence rate,  depending on the \textit{rounding} variable $\alpha$ and the \textit{smoothing} parameter $\tau$. Overall, the number of iterations is significantly smaller than for first-order methods of convex programming for solving the LP relaxation, that additionally require rounding as a post-processing step to obtain an integral solution.
}
  \label{fig:Results-K3-Cursor}
\end{figure}

Fig. \ref{fig:Results-K3} presents the results of the experiments for the minimal cyclic graphical model $\mathcal{K}^3$. In order to assess clearly the influence of the \textit{rounding} parameter $\alpha$ and the 
\textit{smoothing} parameter $\tau$, we evaluated all $10^{5}$ models for \textit{each pair} of $(\tau,\alpha)$, where 
$\tau \in \{\tfrac{1}{2}, \tfrac{1}{2.5}, \ldots, \tfrac{1}{6.5}, \tfrac{1}{7}\}$ and $\alpha \in \{0.1,0.11, \dots, 0.99,1\}$. These statistics show that our algorithm converges to integral solutions, except for very unbalanced parameter values: strong smoothing with large $\tau$, weak rounding with small $\alpha$. Within the remaining broad parameter regime, parameter $\alpha$ enables to control the influence of rounding. In particular, in agreement with Fig.~\ref{fig:parameterInfluence_nAvgEntropy_t10} (bottom), less agressive rounding computed labelings closer to the global optimum. 

Fig.~\ref{fig:Results-K3-Cursor} display exactly the same results as Fig.~\ref{fig:Results-K3}, except for additional data boxes for three 
different configurations of parameter values. For instance, using $\alpha = 0.22$ and $\tau = 0.2$, our algorithm found in $97.35\%$ of the experiments 
an energy with relative error smaller then $1\%$ with respect to the optimal energy. 
In addition, the algorithm required on average 45 iterations to converge. Using instead $\alpha = 0.58$ and $\tau = 0.15$, that is more aggressive rounding in each iteration step \eqref{eq:W-update-alpha}, the average number of 
iterations reduced to 9, but the accuracy also dropped down to $88.6\%$.

Overall, these experiments clearly demonstrate
\begin{itemize}
\item the ability to control the trade-off between high-quality (low energy) labelings and computational costs in terms of $\alpha$, for all values of $\tau$ below a reasonably large upper bound; 
\item a small or very small number of iterations required to converge, depending on the choice of $\alpha$.
\end{itemize}

\subsection{Comparison to Other Methods}\label{sec:comparison-to-other-methods}
We compared our geometric approach to sequential tree-reweighted message passing (TRWS) \cite{Kolmogorov-TRMP06} and
 loopy belief propagation \cite{weiss2001comparing} (Loopy-BP) based on the OpenGM package \cite{opengm-library}. 

\renewcommand{\imgsize}{0.2\textwidth}
\begin{figure}[h]
  \centering
  \begin{tabulary}{1\linewidth}{C C}
    \includegraphics[width=\imgsize]{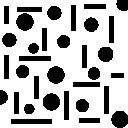} &
    \includegraphics[width=\imgsize]{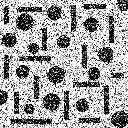} \\
    {Original data}   & {Noisy data} 
  \end{tabulary}
  \caption{Noisy image labeling problem: a binary ground truth image (left) to be recovered from noisy input data (right).}
  \label{fig:binaryImageAndNoisyVersion}
\end{figure}

For this comparison, we evaluated the performance of the methods
for a noisy binary labeling scenario depicted by Fig.~\ref{fig:binaryImageAndNoisyVersion}.
Let $f \colon \mc V \to [0, 1]$ denote the noisy image data given on the grid graph $\mc G = (\mc V, \mc E)$ with a $4$-neighborhood and $\mc X = \{0, 1\}$ as prototypes (labels). 
The following data term and Potts prior were used,
\begin{equation}\label{eq:binary_denoising_data_potts}
  \theta_i = \bpm f(i)\\ 1-f(i) \epm\quad\text{for}\quad i\in \mc V\quad \text{and}\quad \theta_{ij} = \bpm 0 &1\\ 1 &0\epm\quad\text{for}\quad ij\in \mc E\;.
\end{equation}
The threshold $10^{-4}$ was used for the normalized average entropy termination criterion \eqref{eq:normalized_avg_entropy}.
Figure \ref{fig:binaryComparison1} shows the visual reconstruction as well as the corresponding discrete energy values and percentage of correct labels for all three methods.
Our method has similar accuracy and returns a slightly better optimal discrete energy level than TRWS and Loopy-BP.

\begin{figure}
  \centering
  \begin{tabulary}{1\linewidth}{C C C}
    \includegraphics[width=\imgsize]{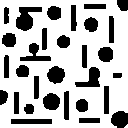} &
    \includegraphics[width=\imgsize]{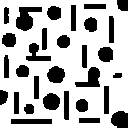} &
    \includegraphics[width=\imgsize]{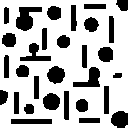} \\
    {Geometric}         & {TRWS}             & {Loopy-BP}  \\
    4977.24 / 98.31\%  & 4979.61 / 98.07\%  & 4977.75 / 98.38\%   \\
  \end{tabulary}
  \caption{Results for the noisy labeling problem from Fig.~\ref{fig:binaryImageAndNoisyVersion} using a standard data term with Potts prior, with discrete energy / accuracy values. 
  Parameter values for the geometric approach:  smoothing $\tau = 0.1$, step-size $h = 0.2$ and rounding strength $\alpha = 0.1$. The threshold for the
  termination criterion was $10^{-4}$. All methods show similar  performance.}
  \label{fig:binaryComparison1}
\end{figure}

We investigated again the influence of the \textit{rounding} mechanism by repeating the same experiment, but using different values of the rounding parameter $\alpha \in \{0.1, 1, 2, 5\}$. As shown by Fig.~\ref{fig:binaryComparison2}, the results confirm the finding of the experiments of the preceding section: More aggressive rounding
  scheme ($\alpha$ large) leads to faster convergence but yields less regularized results with higher energy values.

\begin{figure}
  \centering
  \begin{tabulary}{1\linewidth}{C C C C}
    {$\alpha = 0.1$} & {$\alpha = 1$} & {$\alpha = 2$} & {$\alpha = 5$}   \\
    \includegraphics[width=\imgsize]{Figures/binary_denoising/gmfilter_noisy_binary_4c_r1_a01_t10_h02_E4977_24_acc09831} &
    \includegraphics[width=\imgsize]{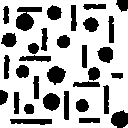} &
    \includegraphics[width=\imgsize]{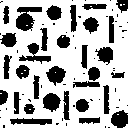} &
    \includegraphics[width=\imgsize]{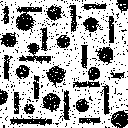} \\
    4977.24 / 98.31\%& 5071.25 / 98.46\% & 5472.71 / 96.97\%   & 7880.64 / 91.25\%       \\
  \end{tabulary}
  \caption{Results for the noisy labeling problem from Fig.~\ref{fig:binaryImageAndNoisyVersion} using different values of the rounding parameter 
  $\alpha \in \{0.1, 1, 2, 5\}$ with discrete energy / accuracy values: more aggressive rounding
  scheme ($\alpha$ large) leads to less regularized results with higher energy values. Parameter values of the geometric approach: 
  smoothing $\tau = 0.1$, step size $h = 0.2$, threshold $10^{-4}$ for termination.}
  \label{fig:binaryComparison2}
\end{figure}

\subsection{Non-Uniform (Non-Potts) Priors}\label{sec:non-potts-prior}
We examined the behavior of our approach for a non-Potts prior by applying it to a non-binary labeling problem with noisy input data, as depicted 
by Fig.~\ref{fig:nonPotts_origData_noisyData}. Our objective is to demonstrate that pre-specified pairwise model parameters (regularization) by a graphical model are properly taken into account.

The label indices corresponding to the five RGB-colors of the original image (Fig. \ref{fig:nonPotts_origData_noisyData} right) are
\begin{equation}
  \mc X = \{\ell_1 = \text{"dark blue"}, \ell_2 = \text{"light blue"}, \ell_3 = \text{"cyan"}, \ell_4 = \text{"orange"}, \ell_5 = \text{"yellow"}\}\subset [0, 1]^3\;.
\end{equation}
Let $f \colon \mc V \to [0, 1]^3$ denote the noisy input image (Fig.~\ref{fig:nonPotts_origData_noisyData}, center panel) given on the grid graph $\mc G = (\mc V, \mc E)$ with a $4$-neighborhood. This image was created by randomly selecting $40\%$ of the original image pixels and then uniformly sampling a label at 
each chosen position. The unary term was defined using the $\|\cdot\|_1$ distance and a scaling factor $\rho>0$ by
\begin{equation}
  \theta_i = \frac{1}{\rho} \big( \| f(i) - \ell_1 \|_1, \ldots, \| f(i) - \ell_5\|_1\big), \quad i\in \mc V.
\end{equation}

Now assume additional information about a labeling problem were available. For example, let the RGB-color  dark blue in the 
image represent the direction "top", light blue "bottom", yellow "right", orange "left" and cyan "center" (Fig. \ref{fig:nonPotts_origData_noisyData} left). Suppose it is known beforehand that "top" and "bottom" as well as "left" and "right" cannot be adjacent to each other but are separated 
by another label corresponding to the center. This prior knowledge about the labeling problem was taken into account by specifying non-uniform pairwise model parameters that penalize these unlikely label transitions by a factor of 10:

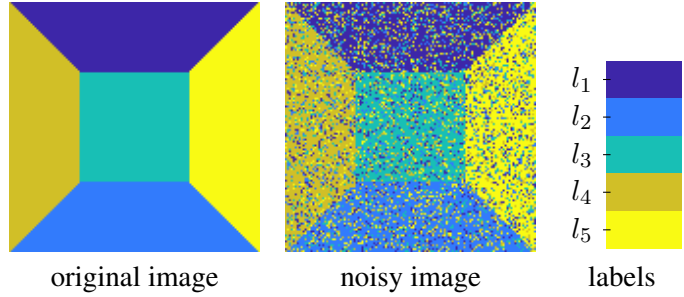
\begin{figure}[h]
  \centering
  \begin{tabulary}{1\linewidth}{C C C}
    \includegraphics[width=\imgsize]{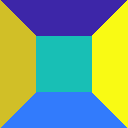} &
    \includegraphics[width=\imgsize]{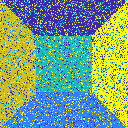} &
    \begin{tikzpicture}
      \definecolor{darkblue}{rgb}{0.2392,0.1490,0.6588}
      \definecolor{lightblue}{rgb}{0.1961, 0.4824, 0.9882}
      \definecolor{cyan}{rgb}{0.0941, 0.7490, 0.7098}
      \definecolor{orange}{rgb}{0.8196, 0.7490, 0.1529}
      \definecolor{yellow}{rgb}{0.9765, 0.9804, 0.0784}
      
      \fill[darkblue] (0, 2) rectangle ++(1, 0.5);
      \draw (0+0.05 , 2.25) -- (0-0.05 , 2.25) node[anchor=east] {$l_1$};
      
      \fill[lightblue] (0, 1.5) rectangle ++(1, 0.5);
      \draw (0+0.05 , 1.75) -- (0-0.05 , 1.75) node[anchor=east] {$l_2$};
      
      \fill[cyan] (0, 1) rectangle ++(1, 0.5);
      \draw (0+0.05 , 1.25) -- (0-0.05 , 1.25) node[anchor=east] {$l_3$};
      
      \fill[orange] (0, 0.5) rectangle ++(1, 0.5);
      \draw (0+0.05 ,0.75) -- (0-0.05 ,0.75) node[anchor=east] {$l_4$};
      
      \fill[yellow] (0,0) rectangle ++(1, 0.5);
      \draw (0+0.05 ,0.25) -- (0-0.05 ,0.25) node[anchor=east] {$l_5$};
        
    \end{tikzpicture} \\
    {original image}    & {noisy image}  & {labels}
  \end{tabulary}
  \caption{Original image (left), encoding the image directions "top", "bottom", "center", "left" and "right" by the RGB-color labels $\ell_1, \ell_2, \ell_3, \ell_4$ and $\ell_5$ (right).
    The noisy test image (middle) was created by randomly selecting $40\%$ of the original image pixels and then uniformly sampling a label at each position. Unlikely label transitions $\ell_{1} \leftrightarrow \ell_{2}$ and $\ell_{4} \leftrightarrow \ell_{5}$ are represented by color (feature) vectors that are close to each other and hence can be easily confused. 
}
  \label{fig:nonPotts_origData_noisyData}
\end{figure}

\begin{equation}\label{eq:non-Potts_definition}
  \theta_{ij} = \frac{1}{10} \bpm 0 &10 &1 &1 &1\\ 10 &0 &1 &1 &1\\ 1 &1 &0 &1 &1\\ 1 &1 &1 &0 &10\\ 1 &1 &1 &10 &0\epm,\qquad ij\in\mc E.
\end{equation}
In words, every entry of $\theta_{ij}$ corresponding to a label transition $\ell_1 = \text{"dark blue"}$ ("top") next to $\ell_2 = \text{"light blue"}$ ("bottom") or 
$\ell_4 = \text{"orange"}$ ("left") next to $\ell_5 = \text{"yellow"}$ ("right") has the large penalty value $1$, whereas all other "natural" configurations are treated 
as with the Potts prior and smaller penalty value of $0$ and $0.1$, respectively. We point out that no color vectors or any other embedding was used to facilitate this regularization task or to represent it in a more application-specific way. Rather, the non-uniform prior \eqref{eq:non-Potts_definition} was considered as \textit{given} in terms of some discrete graphical models and its energy function \eqref{eq:def-E}. On the other hand, the pairs of labels $(\ell_{1},\ell_{2})$ and $(\ell_{4},\ell_{5})$ forming unlikely label transitions can be easily confused by the data term, due to the small distance of the color (feature) vectors representing these labels.

To demonstrate how these non-uniform model parameters influence label assignments, 
we compared the evaluation of this model against a model with a uniform Potts prior
\begin{equation}
  \big( \theta_{ij}' \big)_{k,r} = \tfrac{1}{10} (1 - \delta_{k, r}),
    \quad \text{where}\quad \delta_{k, r} = \begin{cases} 1 & \text{if } k = r,\\ 0 & \text{else},\end{cases},
    \qquad\text{for}\quad ij \in \mc E.
\end{equation}
In our experiments, we used the scaling factor $\rho = 15$ for the unaries, step-size $h = 0.1$, rounding parameter $\alpha = 0.01$, smoothing parameter $\tau = 0.01$ and 
$10^{-4}$ as threshold for the normalized average entropy termination criterion \eqref{eq:normalized_avg_entropy}.

\begin{figure}[h]
  \centering
  \begin{tabulary}{1\linewidth}{C C }
    \includegraphics[width=\imgsize]{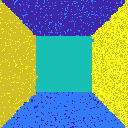} &
    \includegraphics[width=\imgsize]{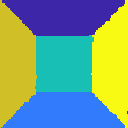} \\
    {Potts}         & {non-Potts} \\
    Acc : 87.12\%  & Acc : 99.34\%
  \end{tabulary}
  \caption{Results of the labeling problem using the Potts and non-Potts prior model together with the Accuracy (Acc) values. Parameters for this
  experiment are $\rho = 15$, smoothing $\tau = 0.01$, step-size $h = 0.1$ and rounding strength $\alpha = 0.01$. The threshold for the
  termination criterion \eqref{eq:normalized_avg_entropy} was $10^{-4}$.}
  \label{fig:nonPotts_results}
\end{figure}

The results depicted in Fig.~\ref{fig:nonPotts_results} clearly show the positive influence of the non-Potts prior (labeling accuracy $99.34\%$) whereas using the Potts prior lowers the accuracy to $87.12\%$. This is due to the fact that the color labels $\ell_4$ and $\ell_5$ as well as $\ell_1$ 
and $\ell_2$ have a relatively small $\| \cdot \|_1$ distance and are therefore not easy to distinguish using both the data term and a Potts prior. On the other hand, the additional prior information about valid label configurations encoded by \eqref{eq:non-Potts_definition} 
was sufficient to overcome this difficulty, despite using the same data term, and to separate the regions correctly.

%% file: TexInput/Conclusion.tex

We presented a novel approach to the evaluation of discrete graphical models in a smooth geometric setting. The novel inference algorithm propagates in parallel `Wasserstein messages' along edges. These messages are lifted to the assignment manifold and drive a Riemannian gradient flow, that terminates at an integral labeling. Local marginalization constraints are satisfied throughout the process. A single parameter enables to trade-off accuracy of optimization and speed of convergence.  

Our work motivates to address applications using graphical models with higher edge connectivity, where established inference algorithms based on convex programming noticeably slow down. 
Likewise, generalizing our approach to tighter relaxations based on hypergraphs and corresponding entropy approximations \cite{Yedidia-GenBP-05,Pakzad:2005aa} seems worth additional investigation. 
Our future work will leverage the inherent smoothness of our mathematical setting for designing more advanced numerical schemes based on higher-order geometric integration and using multiple spatial scales. 

%% file: TexInput/Appendix.tex

\subsection{Proof of Proposition \ref{prop:euclidean_gradient_smooth_energy_general}}
\label{sec:proof-euclidean_gradient_smooth_energy_general}
  Let $\gamma \colon (-\varepsilon, \varepsilon) \to \mc W$ be a smooth curve, with $\varepsilon > 0$, $\gamma(0) = W$ and $\dot{\gamma}(0) = V$. We then have
  \begin{equation}\label{eq:first_step_calc_euclidean_gradient_smooth_energy}
    \la \nabla E_\tau(W), V \ra = \frac{d}{dt} E_\tau\big(\gamma(t)\big) \Big|_{t=0} \overset{\eqref{eq:J-smooth_rewritten}}{=} 
    \sum_{i\in V} \Big( \la P_{T}(\theta_i), V_i\ra + \sum_{j \colon (i, j) \in \mc E} \frac{d}{dt} d_{\theta_{ij}, \tau}(\gamma_i(t), \gamma_j(t))\Big|_{t=0}\Big)\;,
  \end{equation}
  where $\gamma_k(t)$ denotes the $k$-th row of the matrix $\gamma(t) \in \mc W \subset \mb{R}^{m\times n}$. Since
  \begin{equation}
    \frac{d}{dt} d_{\theta_{ij}, \tau}(\gamma_i(t), \gamma_j(t))\Big|_{t=0} = \la \nabla_1 d_{\theta_{ij}, \tau}(W_i, W_j), V_i\ra + \la \nabla_2 d_{\theta_{ij}, \tau}(W_i, W_j), V_j\ra\;,
  \end{equation}
the r.h.s. of \eqref{eq:first_step_calc_euclidean_gradient_smooth_energy} becomes
  \begin{equation}\label{eq:second_step_calc_euclidean_gradient_smooth_energy}
    \la \nabla E_\tau(W), V \ra = \sum_{i\in V} \Big( \la P_{T}(\theta_i), V_i\ra + \sum_{j \colon (i, j) \in \mc E} \la \nabla_1 d_{\theta_{ij}, \tau}(W_i, W_j), V_i\ra\Big)
    + \sum_{i\in V} \sum_{j \colon (i, j) \in \mc E} \la \nabla_2 d_{\theta_{ij}, \tau}(W_i, W_j), V_j\ra\;,
  \end{equation}
  where we deliberately separated the outer sum into two parts. Let $\delta_{(k,l)\in\mc E}$ be the function with value $1$ if $(k, l) \in \mc E$ and $0$ if $(k, l) \notin \mc E$.
  Then the second sum of the expression above reads
  \begin{subequations}
    \begin{align}
      \sum_{i\in V} \sum_{j \colon (i, j) \in \mc E} \big\la \nabla_2 d_{\theta_{ij}, \tau}(W_i, W_j), V_j\big\ra 
      &= \sum_{i\in V} \sum_{j \in V} \delta_{(i,j)\in\mc E} \big\la \nabla_2 d_{\theta_{ij}, \tau}(W_i, W_j), V_j\big\ra\\
      &= \sum_{j\in V} \sum_{i \in V} \delta_{(i,j)\in\mc E} \big\la \nabla_2 d_{\theta_{ij}, \tau}(W_i, W_j), V_j\big\ra\\
      &= \sum_{j\in V} \sum_{i\colon (i,j)\in\mc E} \big\la \nabla_2 d_{\theta_{ij}, \tau}(W_i, W_j), V_j\big\ra\\
      &= \sum_{i\in V} \sum_{j\colon (j,i)\in\mc E} \big\la \nabla_2 d_{\theta_{ji}, \tau}(W_j, W_i), V_i\big\ra\;,
    \end{align}
  \end{subequations}
  where the last equation follows by renaming the indices of summation. Substitution into \eqref{eq:second_step_calc_euclidean_gradient_smooth_energy} gives
\begin{subequations}
\begin{align}
\la \nabla E_\tau(W), V \ra 
&= \sum_{i\in V} \Big\la P_{T}(\theta_i) + \sum_{j \colon (i, j) \in \mc E} \nabla_1 d_{\theta_{ij}, \tau}(W_i, W_j)
    + \sum_{j \colon (j, i) \in \mc E} \nabla_2 d_{\theta_{ji}, \tau}(W_j, W_i), V_i\Big\ra \\
&= \sum_{i\in V} \la \nabla_{i} E_\tau(W), V_{i} \ra
\end{align}
\end{subequations}
which proves \eqref{eq:euclidean_gradient_smooth_energy_general}.

\subsection{Proof of Lemma \ref{lem:rewritten_argmax_of_g}}
\label{sec:proof-lem-rewritten_argmax_of_g}
We first show that, if $\ol{\nu}$ is an optimal dual solution, then
\begin{equation}\label{eq:dual-optima-inclusion}
\underset{\nu \in \R^{2 n}}{\argmax}\;g(p,\nu)
\subseteq \ol{\nu} + \mc{N}(\mc{A}^{\T}).
\end{equation}
Let $\ol{\nu}' \neq \ol{\nu}$ be another optimal dual solution, that is $g(p,\ol{\nu})=g(p,\ol{\nu}')$. By \eqref{eq:def-g-dual}, this equation reads
\begin{equation}\label{eq:rewritten_G_difference}
G_{\tau}^{\ast}(\mathcal{A}^\top \ol{\nu} - \Theta) - G_{\tau}^{\ast}(\mathcal{A}^\top \ol{\nu}' - \Theta) = \la p, \ol{\nu} - \ol{\nu}' \ra\;.
\end{equation}
Moreover, due to the optimality conditions \eqref{eq:dtau-opt-conditions}, $\ol{\nu}'$ satisfies
\begin{equation}\label{eq:rewritten_optimality_condition_primal_dual}
\ol{M}' = \nabla G_{\tau}^{\ast}(\mathcal{A}^\top \ol{\nu}' - \Theta),\qquad \mc{A} \ol{M}' = p,
\end{equation}
with a corresponding primal optimal solution $\ol{M}'$. Hence
\begin{equation}\label{eq:rewritten_G_difference2}
\la p, \ol{\nu} - \ol{\nu}' \ra 
= \la \mc{A} \ol{M}', \ol{\nu} - \ol{\nu}' \ra 
= \la \ol{M}', \mc{A}^\top (\ol{\nu} - \ol{\nu}') \ra
\overset{\eqref{eq:rewritten_optimality_condition_primal_dual}}{=} \la \nabla G_{\tau}^{\ast}(\mathcal{A}^\top \ol{\nu}' - \Theta), \mc{A}^\top (\ol{\nu} - \ol{\nu}') \ra\;.
\end{equation}
Using the shorthands 
\begin{equation}
\ol{w} = \mc{A}^\top \ol{\nu} - \Theta,\qquad 
\ol{w}' = \mc{A}^\top \ol{\nu}' - \Theta,
\end{equation} 
we have 
\begin{equation}\label{eq:diff-w1-w2}
\ol{w}' - \ol{w} = \mc{A}^\top (\ol{\nu}' - \ol{\nu})
\end{equation}
and therefore
\begin{equation}
G_{\tau}^{\ast}(\ol{w}') - G_{\tau}^{\ast}(\ol{w}) \overset{\eqref{eq:rewritten_G_difference}}{=} 
\la p, \ol{\nu}' - \ol{\nu} \ra 
\overset{\eqref{eq:rewritten_G_difference2}}{=} 
\la \nabla G_{\tau}^{\ast}(\ol{w}'), \ol{w}' - \ol{w} \ra.
\end{equation}
Since $G_{\tau}^{\ast}$ is strictly convex, this equality can only hold if 
\begin{equation}
0 = \ol{w}' - \ol{w} 
\overset{\eqref{eq:diff-w1-w2}}{=} 
\mc{A}^\top(\ol{\nu}' - \ol{\nu}).
\end{equation}
This shows that $\ol{\nu}$ and $\ol{\nu}'$ can only differ by a nullspace vector, i.e.~we have shown relation \eqref{eq:dual-optima-inclusion}. It remains to show the reverse inclusion, that is vectors characterized by the right-hand side of \eqref{eq:argmax-g} maximize the dual objective function $g(p,\nu)$.

Let again $\ol{\nu}$ be an optimal dual solution, and let $\ol{\nu}' \in \ol{\nu} + \mc{N}(\mc{A}^{\T})$ be an arbitrary vector. Lemma~\ref{lem:kernel_A_transposed} implies that $\ol{\nu}'$ takes the form
\begin{equation}\label{eq:nu-strich-explicit}
\ol{\nu}' = \ol{\nu} + \alpha \bsm \eins_{n} \\ -\eins_{n} \esm,\qquad \alpha \in \R.
\end{equation}
Now suppose $\left\la p, \bsm \eins_{n} \\ -\eins_{n} \esm \right\ra = 0$. Then, since $\mc{A}^\top \ol{\nu}' = \mc{A}^\top \ol{\nu}$, we have
\begin{subequations}
\begin{align}
g(a, \ol{\nu}') 
&= \la p, \ol{\nu} + \alpha \bsm \eins_{n} \\ -\eins_{n} \esm \ra - G_{\tau}^{\ast}\Big(\mathcal{A}^\top \big(\ol{\nu} + \alpha \bsm \eins_{n} \\ -\eins_{n} \esm\big) - \Theta\Big) \\
&= \la p, \nu \ra - G_{\tau}^{\ast}(\mc{A}^{\T}\ol{\nu}-\Theta)
= g(a, \ol{\nu}),
\end{align}
\end{subequations}
that is $\ol{\nu}' \in \argmax_{\nu \in \R^{2 n}} g(p,\nu)$.

Finally, suppose $\left\la p, \bsm \eins_{n} \\ -\eins_{n} \esm \right\ra \neq 0$, $\ol{\nu}$ is an optimal dual solution and $\ol{\nu}'$ is another optimal dual vector, which has the form \eqref{eq:nu-strich-explicit} as just shown.
Inserting \eqref{eq:nu-strich-explicit} into \eqref{eq:rewritten_G_difference} yields
\begin{equation}\label{eq:rewritten_condition_r}
0 = \la p, \ol{\nu}' - \ol{\nu} \ra = \alpha \la p, \bsm \eins_{n} \\ -\eins_{n} \esm \ra.
\end{equation}
Since $\left\la p, \bsm \eins_{n} \\ -\eins_{n} \esm \right\ra \neq 0$, this can only hold if $\alpha = 0$. Thus, $\ol{\nu}' = \ol{\nu}$ by \eqref{eq:nu-strich-explicit}, which shows uniqueness of $\ol{\nu}$ as claimed by  \eqref{eq:argmax-g}.